\documentclass[sn-mathphys,Numbered]{sn-jnl}
\usepackage{graphicx}
\usepackage{multirow}
\usepackage{amsmath,amssymb,amsfonts}
\usepackage{amsthm}
\usepackage{mathrsfs}
\usepackage[title]{appendix}
\usepackage{xcolor}
\usepackage{textcomp}
\usepackage{manyfoot}
\usepackage{booktabs}
\usepackage{algorithm}
\usepackage{algorithmicx}
\usepackage{algpseudocode}
\usepackage{listings}
\usepackage{framed}
\usepackage{soul}
\usepackage[caption=false]{subfig}
\usepackage{xcolor}
\RequirePackage[capitalize]{cleveref}[0.19]

\newtheorem{theorem}{Theorem}

\newtheorem{proposition}[theorem]{Proposition}
\newtheorem{lemma}[theorem]{Lemma}
\newtheorem{corollary}[theorem]{Corollary}
\newenvironment{objdef}{\vspace{4pt}\noindent\textbf{Problem Assumption.}}{\vspace{4pt}\\}

\newtheorem{example}{Example}

\newtheorem{assumption}{Assumption}
\newcommand{\refobjdef}[1]{Problem Assumption}

\newtheorem{definition}{Definition}
\crefname{proposition}{Proposition}{Propositions}
\crefname{lemma}{Lemma}{Lemmas}
\crefname{assumption}{Assumption}{Assumptions}
\crefname{section}{section}{sections}
\crefname{subsection}{subsection}{subsections}
\Crefname{section}{Section}{Sections}
\Crefname{subsection}{Subsection}{Subsections}
\Crefname{figure}{Figure}{Figures}
\crefformat{equation}{\textup{#2(#1)#3}}
\crefrangeformat{equation}{\textup{#3(#1)#4--#5(#2)#6}}
\crefmultiformat{equation}{\textup{#2(#1)#3}}{ and \textup{#2(#1)#3}}
{, \textup{#2(#1)#3}}{, and \textup{#2(#1)#3}}
\crefrangemultiformat{equation}{\textup{#3(#1)#4--#5(#2)#6}}
{ and \textup{#3(#1)#4--#5(#2)#6}}{, \textup{#3(#1)#4--#5(#2)#6}}{, and \textup{#3(#1)#4--#5(#2)#6}}
\Crefformat{equation}{#2Equation~\textup{(#1)}#3}
\Crefrangeformat{equation}{Equations~\textup{#3(#1)#4--#5(#2)#6}}
\Crefmultiformat{equation}{Equations~\textup{#2(#1)#3}}{ and \textup{#2(#1)#3}}
{, \textup{#2(#1)#3}}{, and \textup{#2(#1)#3}}
\Crefrangemultiformat{equation}{Equations~\textup{#3(#1)#4--#5(#2)#6}}
{ and \textup{#3(#1)#4--#5(#2)#6}}{, \textup{#3(#1)#4--#5(#2)#6}}{, and \textup{#3(#1)#4--#5(#2)#6}}
\crefdefaultlabelformat{#2\textup{#1}#3}

\newcommand{\bsigma}{\boldsymbol{\sigma}}
\newcommand{\bd}{\mathbf{d}}
\newcommand{\br}{\mathbf{r}}
\newcommand{\bx}{\mathbf{x}}
\newcommand{\by}{\mathbf{y}}
\newcommand{\zero}{\mathbf{0}}
\newcommand{\eps}{\varepsilon}
\newcommand{\wt}{\widetilde}
\newcommand{\wh}{\widehat}

\newcommand{\R}{\mathbb{R}}

\newcommand{\srh}{\mathrm{SR}_\eps}
\newcommand{\csr}{\mathrm{SR}}

\newcommand{\rn}{\mathrm{RN}}
\newcommand{\fix}{\mathrm{fi}}
\newcommand{\sign}{\mathrm{sign}}
\newcommand{\expt}{\mathrm{E}}
\newcommand{\calc}{{\mathcal{C}}}
\newcommand{\calo}{{\mathcal{O}}}
\newcommand{\cals}{{\mathcal{S}}}
\newcommand{\calx}{{\mathcal{X}}}
\newcommand{\calw}{{\mathcal{W}}}
\newcommand{\ph}{\phantom}
\DeclareMathOperator*{\argmin}{arg\,min}

\usepackage{lineno}
\begin{document}

\title[Convergence of GD with stochastic fixed-point rounding errors]{On the Convergence of the Gradient Descent Method with Stochastic Fixed-Point Rounding Errors under the Polyak--{\L}ojasiewicz Inequality}

\author*[1]{\fnm{Lu} \sur{Xia}}\email{l.xia1@tue.nl}

\author[2]{\fnm{Stefano} \sur{Massei}}\email{stefano.massei@unipi.it}

\author[1]{\fnm{Michiel E.} \sur{Hochstenbach}}\email{m.e.hochstenbach@tue.nl}

\affil*[1]{\orgdiv{Department of Mathematics and Computer Science}, \orgname{Eindhoven University of Technology}, \orgaddress{\city{Eindhoven}, \postcode{5600 MB}, \country{The Netherlands}}}

\affil[2]{\orgdiv{Department of Mathematics}, \orgname{Università di Pisa}, \orgaddress{ \city{Pisa}, \postcode{56127}, \country{Italy}}}

\abstract{In the training of neural networks with low-precision computation and fixed-point arithmetic, rounding errors often cause stagnation or are detrimental to the convergence of the optimizers. This study provides insights into the choice of appropriate stochastic rounding strategies to mitigate the adverse impact of roundoff errors on the convergence of the gradient descent method, for problems satisfying the Polyak--{\L}ojasiewicz inequality. Within this context, we show that a biased stochastic rounding strategy may be even beneficial in so far as it eliminates the vanishing gradient problem and forces the expected roundoff error in a descent direction. Furthermore, we obtain a bound on the convergence rate that is stricter than the one achieved by unbiased stochastic rounding. The theoretical analysis is validated by comparing the performances of various rounding strategies when optimizing several examples using low-precision fixed-point arithmetic.}



\maketitle

\section{Introduction}\label{sec1}

\label{sec:introduction}
To reduce computing time and hardware complexity, low-precision numerical formats are becoming increasingly popular, especially in the area of machine learning. However, when training neural networks (NNs) using low-precision computation, rounding errors frequently result in stagnation or negatively affect the convergence of the optimizers. 

Both floating-point and fixed-point arithmetics are commonly used in low-precision training~\cite{wang2018training,gupta2015deep,chen2017fxpnet}. Although the floating-point format has a wider number representation, owing to the non-uniform distributed number representation, considerable research efforts focus on improving both offline training and online inference with fixed-point representation. The latter is employed in power-efficient devices such as field-programmable gate arrays (FPGAs) and dedicated application-specific integrated circuits (ASICs). These practical low-cost embedded microprocessors and microcontrollers often rely on fixed-point arithmetic for finite-precision signal processing due to the cost and complexity of floating-point hardware. Additionally, low-precision fixed-point arithmetic is gaining increasing attention in AI-powered drone development; see, e.g., \cite{palossi2021fully,muller2021funfiiber,niculescu2021improving}.

This work is concerned with the gradient descent method (GD), which is the method of choice for a large group of machine learning problems, in view of its computational efficiency and stable convergence. GD has a sublinear convergence for convex problems and it has been shown to have linear convergence for both strongly convex problems \citep{nesterov2003introductory} and problems satisfying the \emph{Polyak--{\L}ojasiewicz inequality} (PL) \citep{polyak1963gradient,lojasiewicz1963topological,karimi2016linear}. Note that strongly convex functions also satisfy the PL condition; see \cite[Appendix B]{karimi2016linear}. The PL condition has been recently analyzed by \cite{karimi2016linear} for optimization problems in machine learning, e.g., least squares and logistic regression. Also, a wide range of NNs satisfies the PL condition \citep{charles2018stability,nguyen2020global,frei2021proxy} and variants thereof \citep{liu2022loss}. 

In the context of low-precision training of NNs, the utilization of the classical \emph{round to nearest method} (RN) \cite[Sec.~2.3]{higham2002accuracy} may cause stagnation, while the employment of the unbiased stochastic rounding method, that we call \emph{stochastic rounding} ($\csr$) (see, e.g., \cite[Sec.~3]{connolly2021stochastic}), has been experimentally shown to provide training accuracy similar to single-precision computation; see, e.g., \citep{gupta2015deep,na2017chip,ortiz2018low,wang2018training}. Therefore, the choice of appropriate rounding methods plays a critical role in the training of NNs in low-precision computation. Given this, there is a growing demand to conduct a thorough error analysis that encompasses both deterministic and stochastic rounding errors throughout the entire gradient descent updating procedure. By gaining a deeper understanding of the factors contributing to the stagnation of GD in low-precision computation, we can make informed decisions in selecting an appropriate rounding method for training NNs.
\begin{figure}[htb!]
\centering
\includegraphics[width=0.38\textwidth]{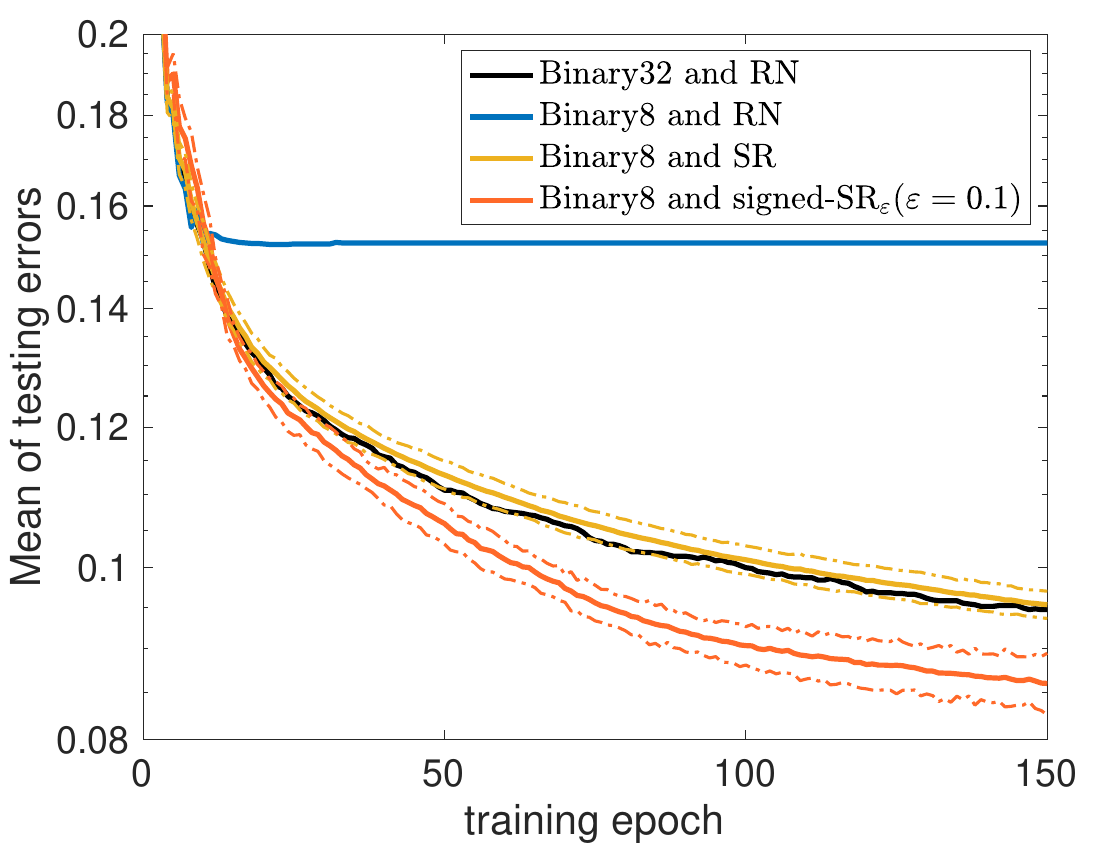}
\caption{Mean of testing errors of training a multinomial logistic regression model on the MNIST dataset over 20 simulations. The number representation systems and rounding strategies compared are: \textsf{Binary32} with RN,  \textsf{Binary8} with RN, $\csr$, and signed-$\srh$. The dashed lines indicate the estimated 95\% confidence intervals of the methods involving a stochastic rounding strategy.}\label{fig:example}
\end{figure}

{\bf{Previous work.}}  In our previous work \citep{xia2022float}, we have analyzed the influence of stochastic \emph{floating-point} roundoff error on the convergence of GD in low-precision computation for \emph{convex} problems. We have theoretically explained why $\csr$ can help prevent the stagnation of GD in the floating-point environment. Moreover, we have introduced two new stochastic rounding methods, namely \emph{$\eps$-biased stochastic rounding} ($\srh$) and \emph{signed-$\srh$} (see \cref{sec:rounding} for a summary of these rounding methods); the latter are more effective than $\csr$ in speeding up the convergence of GD when employed for rounding the outcome of the floating-point operations encountered during the execution of the algorithm.
As an example, \cref{fig:example} shows the mean value of testing errors of training a multinomial logistic regression model (MLR) using different number representation systems, and rounding methods. The application of signed-$\srh$ with an 8-bit number format (\textsf{Binary8}) \cite[Sec.~2.1]{xia2022float} dramatically accelerates the convergence of GD compared to that obtained by $\csr$ with \textsf{Binary8} and RN with single-precision computation (\textsf{Binary32}). All operations in the MLR training are performed using the specified number formats, except for the softmax evaluations, which are first evaluated in double precision and then rounded to the specified number formats.

{\bf{Contributions of the paper.}} We study the impact of the \emph{fixed-point} rounding errors on the convergence of GD for problems satisfying the PL condition. Here the stagnation of the algorithm is caused by the absolute rounding errors, hence the analysis is significantly different from the one concerning the stagnation in the floating-point case. We demonstrate that the customized sign of the rounding bias, which has been the main motivation for introducing signed-$\srh$ in \citep{xia2022float}, is not necessary in the case of fixed-point arithmetic. In particular, $\srh$ and signed-$\srh$ result in the same rounding bias when choosing the same $\eps$. We prove that, for both $\csr$ and $\srh$, the linear convergence rate of GD under the PL condition is preserved for low-precision fixed-point computations and that GD with $\srh$ satisfies a stricter convergence bound than that of GD with $\csr$. Finally, we compare the influence of fixed-point and floating-point rounding errors on the convergence of GD. The comparison illustrates that the implementation of GD using $\srh$ with floating-point arithmetic performs similarly to a gradient descent method with automatically adapted stepsizes in each coordinate of the current iterate. On the other hand, the implementation using $\srh$ with fixed-point arithmetic behaves like a combination of vanilla gradient descent and stochastic sign gradient descent methods.

{\bf{Notations.}} We denote the Euclidean norm and infinity norm by $\|\cdot\|$ and $\|\cdot\|_\infty$, respectively. We indicate by $\mathbb F \subseteq \R$ the subset of the real numbers that are exactly representable within the chosen number representation system. Specifically, given a real number $x\in \R$ and a rounding method, $\wt{x}\in\mathbb F$ is used to denote a corresponding fixed-point representation; this may be a result of a single rounding operation, or an accumulation of rounding errors (as in \eqref{eq:gd_fp}).
This quantity depends on the specific rounding procedure, and is a random variable in the case of stochastic rounding.
We denote by $L$ the Lipschitz constant for the gradient of the objective function and by $\chi$ the largest distance between the iterates generated by GD and the minimizer, i.e., $\chi:=\max_{k} \| \bx^{(k)}-\bx^*\|<\infty$, where $\bx^{(k)}\in \R^n$; in particular, we assume that the iterates of GD remain in a compact set. Furthermore, the rounding precision, i.e., the distance between two subsequent fixed-point numbers,  is indicated by $u$. For instance, in binary representation, the rounding precision in fixed-point arithmetic with one fractional point can be expressed as $u=2^{-1}$.
In the context of floating-point representations, $u$ indicates the relative approximation error due to rounding, i.e., $2^{1-m}$, where $m$ is the number of bits employed for the mantissa.
We employ $\simeq$ and $\lesssim$ to denote approximate equality and inequality with error bounded up to second and higher-order terms in $u$ and a hidden constant with at most polynomial dependence on $n$, $L$, and $\chi$.

\subsection{Problem setting}\label{sec:Problem statements}
Let us briefly recall the convergence behavior of GD for the problems with Lipschitz continuous gradients and satisfying the PL inequality. We consider an unconstrained optimization problem for a differentiable function $f$ with a non-empty solution set 
\begin{equation*}\label{eq:objec}
\calx^*:= \argmin_{\bx\in \R^n} f(\bx),
\end{equation*}
and we denote the optimal value by $f^*:=f(\mathbf x^*)$, for any $\mathbf x^*\in\mathcal X^*$, where $f$ satisfies the following assumption.

\begin{objdef}\label{objdef:plineq_nonconvex}
Let $f: \R^n \to \R$ be a differentiable function whose gradient $\nabla f: \R^n\! \to \!\R^n$ is Lipschitz continuous with constant $L>0$, that is: 
\begin{align}
\|\nabla f(\bx) -\nabla f(\by)\| \le \,& L\, \| \bx-\by\|, \quad \text{for all}~ \mathbf x, \mathbf y \in \R^n.\label{eq:lpineq}
\end{align}
In addition, $f$ satisfies the PL inequality:
\begin{equation}\label{eq:PLineq}
2 \mu \,(f(\bx)-f^*)\le \, \|\nabla f(\bx)\|^2, \quad \text{for all}~ \bx\in \R^n.
\end{equation}
\end{objdef} 

For quadratic functions $f(\bx) = \tfrac12 \,\bx^T\!A\bx$, where $A$ is symmetric positive definite, $\mu$ and $L$ can be taken as the smallest and largest eigenvalue of $A$, respectively. 
For general twice continuously differentiable functions it holds that $\mu \le L$, which we can see as follows.
For this argument, we may assume without loss of generality that $f^* = 0$ and $\zero\in \mathcal X^*$. Using the Taylor series at this point, we get
\[
\mu \ \bx^T \, \nabla^2 f(\zero) \, \bx + \calo(\|\bx\|^3) = 2 \, \mu \, f(\bx) \le \|\nabla f(\bx)\|^2 \le L^2 \, \| \bx \|^2.
\]
By taking $\bx$ in the direction of the largest eigenvalue of $\nabla^2 f(\zero)$ in modulus, we conclude $\mu \le L$. 
Usually, $\mu$ will be considerably smaller than $L$.
For the readability of our results, it turns out convenient to assume that $\mu \le \frac12 L$; we have therefore added this condition to our assumption (cf.~\cref{assum:universal assumption} in \cref{sec:universalassum}).
We point out that this is not a restriction; slightly modified results hold without this extra clause. For instance, for some results instead of requiring the updating stepsize $t \le \frac{1}{L}$, we need $t\le  \min\{ \frac{1}{L}, \frac{1}{2\,\mu}\}$. Additionally, as explained in \cite[]{karimi2016linear}, \cref{eq:PLineq} implies that every stationary point $\bx^* $ is a global minimum. Therefore, a unique globally optimal solution is not implied by our assumption. We remark that our statements hold with weaker hypotheses, for instance by requiring that \cref{eq:lpineq} and \cref{eq:PLineq} are only satisfied in a certain domain that contains all the iterates of GD.

In exact arithmetic, GD updates iteratively in the opposite direction of the gradient, given by
\begin{equation}
\bx^{(k+1)}=\bx^{(k)}-t\ \nabla f(\bx^{(k)}),
\label{eq:gd}
\end{equation}
where $t\in \mathbb R$ is the \emph{stepsize}. When optimizing problems satisfying \cref{eq:lpineq} and \cref{eq:PLineq} using GD in exact arithmetic, a linear convergence rate is proven by \cite{polyak1963gradient}. A simple version of this proof is given by \cite{karimi2016linear}. 
\begin{theorem}\cite[Thm.~1]{karimi2016linear}\label{them:plineq_nonconvex}
Let the objective function $f$ satisfy \cref{eq:lpineq} and \cref{eq:PLineq}.
In exact arithmetic, the $k$th iteration step of the gradient descent method with a fixed stepsize $t\le \frac{1}{L}$ satisfies the following inequality:
\begin{align}\label{eq:plineq_nonconvex}
f(\bx^{(k)})-f^*\le(1-t\,\mu)^k \, (f(\bx^{(0)})-f^*).
\end{align} 
\end{theorem}
Throughout the paper, we assume that there is no overflow and, for convenience, we assume that both the stepsize $t$ and the starting point $\bx^{(0)}$ are fixed-point numbers, i.e., $t=\wt t$ and $\bx^{(0)}=\widetilde{\bx}^{(0)}$. Under these assumptions, there are two sources of roundoff errors at each iteration of GD, i.e., the absolute rounding errors caused by evaluating $\nabla f(\wt\bx^{(k)})$ at $\wt\bx^{(k)}$ and multiplying the evaluated gradient $\nabla f(\wt\bx^{(k)})$ by $t$ in a limited precision; the latter are denoted by $\bsigma^{(k)}_{1}$ and $\bsigma^{(k)}_2$, respectively. Therefore, we can write the updating rule of GD with rounding errors as
\begin{align}
\wt\bx^{(k+1)} & = \wt\bx^{(k)}-t\ \nabla f(\wt\bx^{(k)})-t\,\bsigma^{(k)}_{1}-\bsigma^{(k)}_2.\label{eq:gd_fp}
\end{align}
Comparing the implementation of GD in fixed-point arithmetic \cref{eq:gd_fp} to the exact arithmetic \cref{eq:gd}, the rounding errors $\bsigma^{(k)}_{1}$ and $\bsigma^{(k)}_2$ may affect the convergence speed and accuracy depending on the magnitude and sign of the absolute rounding errors.

We analyze the influence of rounding errors on the convergence of GD concerning two aspects of the method: monotonicity and convergence rate with respect to the objective function value.
During the execution of GD, it is common to see differences in magnitudes among the various components of the gradient vector. To conduct a thorough error analysis throughout the entire updating procedure, we identify three cases depending on the relation between the magnitudes of the gradients and the rounding precision.
In Case I, the updating vectors of GD consist entirely of scaled gradient components, i.e., $t\ |\nabla f(\wt\bx^{(k)})_i+\,\sigma^{(k)}_{1,i}| \ge u$ for all $i$.
In the second case, the updating vectors of GD are on the level of the rounding errors, i.e., $t\ |\nabla f(\wt\bx^{(k)})_i+\,\sigma^{(k)}_{1,i}|<u$ for all $i$ (Case II). In the last case (Case III), the updating vectors of GD depend on both the scaled gradient components and the rounding errors, i.e., some components of $t\ |\nabla f(\wt\bx^{(k)})+\,\bsigma^{(k)}_{1}|$ satisfy Case I and others fulfill Case II.

Among the three cases, Case I is the one that most likely describes the updating procedure of GD in the initial iteration steps when the starting point is far from the optimal point. Cases II and III apply to situations when the gradient has some or all components close to the rounding precision $u$. For all three cases, we assume that each coordinate of the gradient function is assessed with a \emph{non-opposite sign} (cf.~\cref{def:monosign}) to their exact values. In particular, the non-opposite sign evaluation indicates that the component-wise multiplication of the rounded gradient and its exact value is always a non-negative vector. Therefore, the convergence of GD is proven to be guaranteed at least until each coordinate of the gradient function is close to the level of $u$, i.e., $\|\nabla f\| =\calo(\sqrt{n}u)$. We prove that in Case I, the employment of $\csr$ with low-precision computation leads to the same linear convergence bound, in expectation, as the exact computation for the problems satisfying the PL condition. For Cases II and III, the convergence bound of GD may be slightly larger than that for the exact computation. When $\srh$ is employed, we prove that a stricter convergence bound is achieved compared to $\csr$ and may be even stricter than the bound applicable to the exact computation.

\subsection{Main theoretical contributions} 
The main contribution of this study is to provide sufficient conditions to ensure linear convergence of GD with the presence of rounding errors.
For ease of presentation, we display the bound of the convergence rate for the following case. We assume that the objective function $f$ satisfying \cref{eq:lpineq} and \cref{eq:PLineq} is optimized using GD with fixed-point arithmetic; from the start to the $k_1$th iteration the updating rule satisfies Case I, from the $(k_1+1)$st to the $k_2$th iteration it satisfies Case II, and from the $(k_2+1)$st to the $k_3$th iteration it satisfies Case III. 
Note that this assumed order of cases aims at simplifying the notation; cases may have different orders and the theoretical result applies to any order. 

If both $\bsigma_1$ and $\bsigma_2$ in \cref{eq:gd_fp} are results of $\csr$ with fixed stepsize $t\le \frac{1}{4\,L}$ and under the assumption that $\theta_{j_1}>0$, where $\theta_{j_1}$ is defined in \cref{eq:case2_theta}, as shown in \cref{coro:convergencerate_c1_csr,prop:convergencerate_c2_csr,prop:convergencerate_c3_csr}, then with $\alpha_{j_2}$ as defined in \cref{eq:alpha_k} the convergence rate satisfies 
\[
 \expt\,[\,f(\wt\bx^{(k)})-f^*\,]\lesssim(1-\mu\,t)^{k_1} \prod_{j_1=k_1+1}^{k_2}(1-\mu\,t\,\theta_{j_1})  \prod_{j_2=k_2+1}^{k_3}(1-\mu\,(t+\alpha_{j_2}))(f(\bx^{(0)})-f^*),
\]
where all the factors, apart from  $f(\bx^{(0)})-f^*$, are in the interval $(0,1)$.

If $\bsigma_1$ is the result of $\csr$ and $\bsigma_2$ is obtained using $\srh$ in \cref{eq:gd_fp} with fixed stepsize $t\le \frac{1}{4\,L}$ and under the assumption that $\theta_{j_1},\theta_{j_2}>0$, see \cref{them:convergencerate_c1_srh,prop:convergencerate_c2_srh,prop:convergencerate_c3_srh}, then with $\tau_1,\tau_2$ as defined in \cref{eq:tau_1} and \cref{eq:tau_2}, respectively, the convergence rate satisfies
\begin{small}
\begin{align*}
  &\expt\,[\,f(\wt\bx^{(k)})-f^*\,] \\
 &\lesssim
(1-\mu\,(t+\tau_1))^{k_1}\!\!\prod_{j_1=k_1+1}^{k_2}\!\!\!(1-\mu\,\theta_{j_1}\,(t+\tau_2))\!\!\!\prod_{j_2=k_2+1}^{k_3}\!\!(1-\mu\,(t+\alpha_{j_2}+\theta_{j_2}\,\tau_2))(f(\bx^{(0)})-f^*),
\end{align*}
\end{small}\noindent
where all the factors, apart from $f(\bx^{(0)})-f^*$, are in the interval $(0,1)$. In particular, the presence of the (positive) quantities $\tau_1$ and $\tau_2$ suggests faster convergence of GD when $\srh$ is used instead of $\csr$ for rounding the multiplication between the gradient and the step size. Detailed discussions on $\tau_1$ and $\tau_2$ are provided in \cref{prop:convergencerate_c2_srh,prop:convergencerate_c3_srh}, respectively. We remark that instead of using a norm-wise assumption to guarantee the monotonicity of the algorithm, we employ the component-wise assumptions $\theta_{j_1}, \theta_{j_2}>0$ that are sufficient but not necessary to achieve a linear convergence rate in different cases. On the other hand, the use of a norm-wise assumption on the gradient would result in a cumbersome analysis when targeting a linear convergence rate. Additionally, in the numerical study of \cref{sec:simulation}, we show examples where $\theta_{j_1}, \theta_{j_2}$ are negative, and GD suffers from stagnation using $\csr$. 

A summary of all the theoretical results with respect to different cases is given in \cref{tab:sumlayout}. Note that we only consider the use of $\csr$ when evaluating gradients. The reason is that $\csr$ ensures an estimation of the gradient with a non-opposite sign compared to the exact arithmetic and this property is crucial for the theoretical analysis. In contrast, ensuring the latter property with $\srh$ appears difficult in our setting.

\begin{table}[htb!]
\caption{{\footnotesize Summary of the theoretical results.}}\label{tab:sumlayout}
\begin{tabular}{@{}clll@{}}
\toprule
Case &Result& Rounding strategy & Reference \\ \midrule
\multirow{3}{*}{I} & Convergence rate & General & \cref{them:boundforfixed_point_PL} \\ \rule{0pt}{2.3ex}
& Convergence rate & $\boldsymbol{\sigma_1} (\csr)$ and $\boldsymbol{\sigma_2} (\srh)$ & \cref{them:convergencerate_c1_srh} \\
 & Convergence rate & $\boldsymbol{\sigma_1} (\csr)$ and $\boldsymbol{\sigma_2} (\csr)$ & \cref{coro:convergencerate_c1_csr} \\\rule{0pt}{3.3ex}
\multirow{4}{*}{II} & Monotonicity & $\boldsymbol{\sigma_1} (\csr)$ and $\boldsymbol{\sigma_2} (\csr)$ & \cref{prop:monoto_csr} \\\rule{0pt}{2.3ex}%
& Convergence rate & $\boldsymbol{\sigma_1} (\csr)$ and $\boldsymbol{\sigma_2} (\csr)$ & \cref{prop:convergencerate_c2_csr} \\ \rule{0pt}{2.3ex}%
 & Monotonicity & $\boldsymbol{\sigma_1} (\csr)$ and $\boldsymbol{\sigma_2} (\srh)$ & \cref{prop:monoto_srh} \\ \rule{0pt}{2.3ex}%
 & Convergence rate & $\boldsymbol{\sigma_1} (\csr)$ and $\boldsymbol{\sigma_2} (\srh)$ & \cref{prop:convergencerate_c2_srh} \\ \rule{0pt}{3.3ex}
\multirow{2}{*}{III} & Convergence rate & $\boldsymbol{\sigma_1} (\csr)$ and $\boldsymbol{\sigma_2} (\csr)$ &  \cref{prop:convergencerate_c3_csr} \\ \rule{0pt}{2.3ex}%
 & Convergence rate & $\boldsymbol{\sigma_1} (\csr)$ and $\boldsymbol{\sigma_2} (\srh)$ & \cref{prop:convergencerate_c3_srh} \\\botrule
\end{tabular}
\end{table}
\subsection{Outline of the paper} \Cref{sec:fixed and rounding} reviews the basic properties of fixed-point arithmetic and the rounding rules that will be used in this paper. There, we identify the non-opposite sign property to guarantee the descent direction of rounding errors when implementing GD in low-precision computation. In \cref{sec:gdm}, the rounding errors are analyzed with respect to its sign and magnitude. Based on the magnitude of rounding errors, we categorize the updating vector of GD into the three cases that we already mentioned in \cref{sec:Problem statements}. In \cref{sec:convergence analy}, the impact of the rounding bias on the convergence of GD is studied for problems satisfying \cref{eq:lpineq} and \cref{eq:PLineq} in each of the three cases. \Cref{sec:compar} compares the influence of rounding bias on the convergence of GD in floating-point and fixed-point arithmetic. We validate our theoretical findings with several numerical simulations in \cref{sec:simulation}. Finally, conclusions are presented in \cref{sec:conclu}.

\section{Number representation and rounding schemes}\label{sec:fixed and rounding}
Let us introduce fixed-point arithmetic and review the definitions of different rounding methods: $\rn$, $\csr$, $\srh$, and signed-$\srh$. Then, we define the non-opposite sign property for the rounded values, which will be used to guarantee the descent direction in GD. 
\subsection{Fixed-point representation}\label{sec:fixedp}
Fixed-point arithmetic is an alternative way to represent numbers in low-cost embedded microprocessors and microcontrollers, where a floating-point unit is unavailable. Fixed-point numbers are mostly in base 2 (binary representation). We use the format $\mathrm{Q}[\mathrm{QI}].[\mathrm{QF}]$ to represent binary fixed-point numbers and use the two's complement\footnote{Adding a 1 to the least significant (rightmost) bit after inverting the bits of the absolute value.} for the sign, where $\mathrm{QI}$ denotes the number of integer bits and $\mathrm{QF}$ denotes the number of fractional bits \citep{oberstar2007fixed}. For instance, a Q4.6 number presents a 10-bit value with 4 integer bits and 6 fractional bits. Moreover, the rounding precision is given by $u=2^{-\mathrm{QF}}$. The fixed-point arithmetic used in the numerical experiments is implemented using Matlab's \texttt{fi} toolbox.

\subsection{Review of rounding methods}
\label{sec:rounding}
When converting a real number to a fixed-point number format, we denote by $\fix(\cdot)$ a general rounding operator that maps $x \in \R$ to $\fix(x)\in\mathbb{F}$. Throughout this study, we consider rounding schemes of the form
\begin{align}
\fix(x)= \begin{cases}
\lfloor x \rfloor, \quad \quad &\text{with probability}~p(x),\\
\lfloor x \rfloor+u, \quad  &\text{with probability}~1-p(x),
\end{cases}
\label{eq:strounding}
\end{align}
with $p(x)\in[0,1]$, and $\lfloor x \rfloor=\max\{y\in\mathbb{F}: y \le x\}$ indicating the largest representable fixed-point number less than or equal to $x$. When a specific rounding scheme is applied, $\fix(\cdot)$ and $p$ are replaced by the corresponding rounding operator and probability, respectively. Specifically, we denote $p_r$, $p_0$, $p_{\varepsilon}$, and $p_{\varepsilon s}$, the rounding probabilities of rounding to the nearest, unbiased stochastic rounding, $\eps$-biased stochastic rounding, and signed $\eps$-biased stochastic rounding, respectively.

{\bf{Round to the nearest method with half to even}} (RN) is the default rounding mode applied in IEEE 754 floating-point operations. It forces the least significant bit (LSB) to 0 in the case of a tie \cite[]{santoro1989rounding}. When we use $\rn$ as the rounding operator ($\fix=\rn$) in \cref{eq:strounding}, we have 
$p_r(x)= 1$ when $x-\lfloor x \rfloor<\tfrac{1}{2}u$ or when $x-\lfloor x \rfloor=\tfrac{1}{2}u$ and $\mathrm{LSB}$ of $\lfloor x \rfloor$ is even.

{\bf{Unbiased stochastic rounding ($\csr$)}}
provides an unbiased rounded result by setting a probability that is proportional to the distance from $x$ to the nearest representable fixed-point number. Some basic properties and the backward error analysis of $\csr$ have recently been studied by \cite{connolly2021stochastic}. When $\csr$ is employed as the rounding operator ($\fix=\csr$) in \cref{eq:strounding}, we have
\begin{align}
p_0(x)=1-(x-\lfloor x \rfloor)\,/\,u.
\label{eq:csr}
\end{align}

{\bf $\eps$-biased stochastic rounding ($\srh$)} has been introduced by \cite{xia2022float}. It provides a rounding bias with the same sign as its input by increasing or decreasing by $\eps\in(0,1)$ the probability of rounding
down in $\csr$, depending on the sign of the input number (cf.~\cref{eq:epsilon}). When we choose $\srh$ as the rounding operator ($\fix=\srh$) in \cref{eq:strounding}, we have
\begin{subequations}\label{eq:srh}
\begin{align}\label{eq:srh_p}
p_\eps(x):=\varphi(\eta(x, \eps)),
\end{align} 
with \begin{align}
&\eta(x,\eps) := 1-(x-\lfloor x \rfloor)/u-\sign(x)\,\eps,~ \quad
\quad \varphi(y)=\begin{cases}
0,&y\le 0,\\
y,&0\le y\le 1,\\
1,&y\ge 1.
\end{cases}\label{eq:epsilon}
\end{align}
\end{subequations}

{\bf Signed-$\boldsymbol{\srh}$} is a variant of $\srh$ where we can customize the sign of the rounding bias by replacing $\sign(x)$ in \cref{eq:epsilon} with the sign of the additional variable $v\in \R$, i.e.,
\begin{subequations}\label{eq:ssrh}
\begin{align}\label{eq:ssrh_p}
p_{\eps s}(x):=\varphi(\eta_{s}(x, \eps,v)),
\end{align} 
with \begin{align}
\eta_{s}(x,\eps,v) := 1-(x-\lfloor x \rfloor)/u-\sign(v)\,\eps.
\end{align} 
\end{subequations}
A detailed description of $\srh$ and signed-$\srh$, and their implementations, can be found in \cite[Sec.~2.2]{xia2022float}. We will use all four rounding methods for the convergence analysis in \cref{sec:convergence analy,sec:compar}.

\subsection{Preserving non-opposite sign of fixed-point arithmetic operations}
Under the assumption that there is no overflow, the addition and subtraction of fixed-point numbers within the same number format are performed without introducing any rounding errors. However, representing the product of two fixed-point numbers may require more digits than those used for representing the factors \cite[]{yates2009fixed}; e.g., multiplying two 8-bit values produces a 16-bit result, in general. 
In the case of multiplication, we adopt the following model to present the rounded value 
\begin{align}
\fix(\wt{x}\,\wt{y})=\wt{x}\,\wt{y}+\sigma,\qquad \text{where}~\begin{cases}| \sigma |< u &\text{for $\csr$, $\srh$ and signed-$\srh$}, \\
| \sigma | \le \frac{1}{2}u&\text{for $\rn$},
\end{cases}
\label{eq:eqmodel}
\end{align}
where $\fix\in\{\rn, \csr, \srh, \text{signed-}\srh\}$ and $\fix(\wt x\,\wt y)\in \{\lfloor \wt x\,\wt y \rfloor, \lfloor \wt x\,\wt y \rfloor+u\}$. The main difference between \cref{eq:eqmodel} and the model of floating-point computation \cite[Sec.~2.1]{higham2002accuracy} is that the rounding error $\sigma$ is ensured to be small only in an absolute sense.  
Due to the rounding properties, we always have $\sign(\fix(\wt{x}\,\wt{y}))\,\sign(\wt x\,\wt y) \ge 0$. Specifically, for a single operation, all the rounding schemes mentioned in \cref{sec:rounding} will not result in a rounded value with an opposite sign as its input.

However, this may be different when a series of operations is implemented. For instance, using $\csr$ to compute $\fix(0.24)\!-\fix(0.26)$, by first rounding the two numerical values in the binary number format Q$1.1$\! ($u\!=\!\frac1{2}$), and then evaluating their difference. We have 
 \begin{align*}
 \csr(0.24)-\csr(0.26)=\begin{cases}
 \ph{-}0.5,\quad& \text{with probability $0.2304$},\\
 \ph{-0.}0,\quad &\text{with probability $0.4992$},\\
 -0.5,\quad &\text{with probability $0.2704$}.
 \end{cases}
 \end{align*}
 Similar issues are encountered in deterministic rounding. Despite the monotonicity of $\rn$, i.e., $x\le y\in \R$ implies $\rn(x)\le \rn(y)$, when more than two rounding errors are accumulated, an opposite sign may be obtained; for instance, within the number format Q$1.1$, we have 
\[
 \sign(\rn(0.26)+\rn(-0.24)+\rn(-0.24))=-\sign(0.26-0.24-0.24).
\]
In \cref{sec:signd}, we will see that it is crucial to ensure that the rounded gradient preserves the sign of each component of the exact gradient, as this causes GD to update in a non-ascent direction. For this reason, we introduce the property of \emph{non-opposite sign} for the rounded result obtained from a series of operations within a given number format.

\begin{definition}[non-opposite sign] \label{def:monosign} 
When evaluating $g: \R^n \to \R$ in fixed-point arithmetic at $\wt \bx\in\mathbb F^n$, we say that the resulting value $\wt{g(\wt \bx)}$ has the non-opposite sign property if
\[
\sign\,(\wt{g(\wt \bx)})\ \sign\,(g(\wt \bx)) \ge 0.
\]
 \end{definition}
To facilitate further analysis, we propose a condition to guarantee the non-opposite sign property for $\rn$ and stochastic rounding methods. Denote by $\sigma_g$ the accumulated rounding error in evaluating $g(\wt \bx)$, such that $\wt{g(\wt \bx)}=g(\wt \bx)+\sigma_g$. If 
\begin{align}\label{eq:mono_allrounding}
 | g(\wt\bx)| \ge| \sigma_g|, 
\end{align} then it easily follows that $\wt{g(\wt \bx)}$ has the non-opposite sign property.

Now, we analyze what happens in expectation to the sign of stochastically rounded quantities. In the case of $\csr$, we demonstrate that the zero mean property implies the non-opposite sign in expectation when $g(\wt \bx)$ can be evaluated using a finite sequence of elementary operations.
\begin{proposition}\label{prop:exptf_csr}
Let $g: \Gamma \to \R$, with $\Gamma\subseteq \R^n$ an open set, be a function that can be evaluated with a finite sequence of the elementary operations ($+,-,*,/$), and let $\wt{g(\wt\bx)}$ be the corresponding random variable obtained by evaluating the operations using fixed-point arithmetic and $\csr$ for rounding, at the point $\wt\bx\in\Gamma\cap\mathbb F^n$. If overflow and division by zero are not encountered, then
\begin{align}
\big|\expt\,[\,\wt{g(\wt\bx)}\,] - g(\wt \bx)\big| \lesssim c\,  u^2,
\end{align}
where $c$ depends the number of multiplications and divisions performed for evaluating $g(\wt\bx)$, and may depend on the point $\wt\bx$ itself.
\end{proposition}
\begin{proof}
Let $\sigma_g$ denote  the accumulated error caused by evaluating $g$ so that $\wt{g(\wt{\bx})}=g(\wt{\bx})+\sigma_g$. If the evaluation of $g(\wt{\bx})$ requires $m$ elementary operations, then considering the Taylor expansion of the accumulated error, in accordance with \cite[(8)]{linnainmaa1976taylor}, we have 
\[
\sigma_g=\sum_{i=1}^{m} c_{m,i}\,\sigma_i+\sum_{i=1}^m\sum_{j=1}^i c_{m,ij}\,\sigma_i\,\sigma_j+\calo(u^3),
\]
where $\sigma_i$ and $\sigma_j$ indicate the error introduced by single operations and  $c_{m,i},c_{m,ij}$ are certain scalar coefficients. The zero mean independence property for $\csr$ \cite[Lemma 5.2]{connolly2021stochastic} implies that $\expt\,[\,\sigma_{i}\,]=0$ and $\expt\,[\,\sigma_{i}\,\sigma_{j}\,]=0$ for $i\ne j$.
Therefore, we achieve 
\[
\left| \expt\,[\,\sigma_g\,]\right| \le \sum_{i=1}^m | c_{m,ii}| \, \sigma_{i}^2 +\calo(u^3) \le c\,u^2+\calo(u^3),
\]
where $c$ is problem-dependent and can be obtained based on Table 1 in \citep{linnainmaa1976taylor}. Note that Table 1 in \citep{linnainmaa1976taylor} is based on floating-point arithmetic; when fixed-point arithmetic is applied, only the coefficients in multiplication and division operations are relevant. 
\end{proof}
\begin{example} \rm
For instance, when $g(\wt{\bx})=\wt{x}_1^2+\wt{x}_1\wt{x}_2$, based on the Taylor expansion and Table 1 in \citep{linnainmaa1976taylor}, the accumulated error can be expressed as $\sigma_g=2\wt{x}_1 \sigma_{\wt{x}_1}+\wt{x}_2 \sigma_{\wt{x}_1}+\wt{x}_1 \sigma_{\wt{x}_2}+\sigma_{\wt{x}_1}^2 +\calo (u^3)$, where the constant in front of $\sigma_{\wt{x}_1}^2$ is obtained from the fifth row and column in Table 1 in \citep{linnainmaa1976taylor}. Therefore, using $\csr$, the expectation of the accumulated error is bounded by $\expt \,[\,\sigma_g\,]\le u^2$.
\end{example}
Therefore, when using $\csr$, we have that $\expt\,[\,\wt{g(\wt{\bx})}\,]$ has a non-opposite sign with respect to $g(\wt{\bx})$. However, this property cannot be guaranteed by $\srh$. In the next section, we analyze the influence of rounding errors on the convergence of GD under the PL condition for different rounding strategies on the basis of \cref{eq:mono_allrounding,prop:exptf_csr}.
\section{Gradient descent method with fixed-point arithmetic}\label{sec:gdm} 
Near an optimum, the GD algorithm in finite precision may move back and forth around the minimum repeatedly; see, e.g., \cite[Fig.~9.2]{boyd2004convex}. This zigzag behavior is generally caused by the rapid change in negative gradient direction and can be mitigated by reducing the stepsize (setting a smaller $t$) or adding a momentum term; see, e.g., the heavy ball method \cite{polyak1964some} and Nesterov's accelerated method \cite[(2.2.11)]{nesterov2003introductory}. Although reducing the stepsize mitigates the zigzag behavior, when implementing GD in low precision and with $\rn$, a small stepsize may prevent GD from converging because of the loss of gradient information. This issue is known as the vanishing gradient problem and can be overcome using stochastic rounding strategies \cite[]{gupta2015deep, xia2022float}. We start our convergence analysis by studying the role of rounding errors during the updating procedure of GD for fixed-point arithmetic.

A key tool for our analysis is the \emph{update vector incorporating rounding errors} of GD (cf.~\cref{eq:gd_fp})
\[
\wt{\bd}^{(k)}:=t\ \nabla f(\wt\bx^{(k)})+t\,\bsigma^{(k)}_{1}+\bsigma^{(k)}_2,
\]
where $\bsigma^{(k)}_{1}, \bsigma^{(k)}_{2}$ depend on the two rounding schemes used for the evaluation of the gradient and the multiplication by $t$. In particular, we have: 
\begin{align} \label{eq:dk} \wt\bx^{(k+1)}&=\wt\bx^{(k)}-\wt{\bd}^{(k)}.
\end{align}
From \cref{eq:eqmodel}, we have that $\|\bsigma^{(k)}_{1}\|_\infty\le mu$, where $m$ is a positive constant depending on the gradient function and $\wt\bx$; 
$\|\bsigma^{(k)}_2 \|_\infty\le \frac{1}{2}u$ and $\|\bsigma^{(k)}_2 \|_\infty<u$ for $\rn$ and stochastic rounding, respectively. In particular, both signs of the elements and magnitude of $\wt{\bd}^{(k)}$ are crucial to understanding the convergence of GD. In the next subsection, we distinguish three cases that describe the number of components of the updating vector that have magnitudes close to the rounding errors. We explore the roles of $\bsigma_1^{(k)}$ and $\bsigma_2^{(k)}$ in determining properties of $\wt{\bd}^{(k)}$. We start by studying a condition that ensures that $\wt{\bd}^{(k)}$ has a non-descent direction.
\subsection{A condition for the non-opposite sign property} 
\label{sec:signd}
Intuitively, it is desirable to have that the entries of $\wt{\bd}^{(k)}$ have non-opposite signs to the corresponding entries of  $\nabla f(\wt\bx^{(k)} )$, as this property implies that $\bx^{(k+1)}$ updates in a non-ascent direction. Taking a closer look at \eqref{eq:gd_fp} we see that $\bsigma_2^{(k)}$, resulting from a multiplication, can not result in an opposite sign. On the other hand, in view of \cref{eq:eqmodel}, $\bsigma^{(k)}_{1}$ plays a significant role when some of the entries of the exact gradient verify $|\nabla f(\wt\bx^{(k)})_i|<| \sigma_{1,i}^{(k)}|$. 
In this case, GD may still update in a descent direction, i.e., $\nabla f (\bx^{(k)})^T\wt{\bd}^{(k)}>0$; however, we cannot guarantee the monotonicity of the objective function values in general and, often, zigzag behavior is observed in the neighborhood of the optimal points. The influence of errors in evaluating the gradient function on the convergence of stochastic or inexact gradient descent has been extensively analyzed in many studies \citep{bertsekas2000gradient,schmidt2011convergence,nguyen2018does}. It has been shown that even using exact arithmetic for the other computations, the linear or sublinear convergence rate can be proven only up to a neighborhood of a stationary point, where the majority of the gradient components are small \citep{nguyen2018does}.

In our study, we focus on the utilization of rounding methods to implement GD in finite-precision arithmetic, with an additional assumption ensuring that the evaluation of the gradient does not introduce a change in sign with respect to the exact quantity.
That is, for convenience our analysis assumes that for all iterations we have
\begin{align}\label{eq:bound_gradient} 
 |\nabla f(\wt\bx^{(k)})_i | \ge| \sigma^{(k)}_{1,i} |, \qquad i=1,\dots, n. 
\end{align} 
In this situation, GD is ensured to update in a non-ascent direction as \cref{eq:bound_gradient} implies that
\[
\sign(t\ \nabla f(\wt\bx^{(k)})_i)\,\,\sign(t\ \nabla f(\wt\bx^{(k)})_i+t\,\sigma^{(k)}_{1,i}+\sigma^{(k)}_{2,i}) \ge 0\qquad i=1,\dots, n.
\]
Note that condition \cref{eq:bound_gradient} is a very reasonable assumption, as it indicates that the rounding error in evaluating the components of the gradient is not larger than the true quantities.
We remark that assumption \cref{eq:bound_gradient} may be relaxed, by only requiring that in every iteration there is at least one component $i$ satisfying $|\nabla f(\wt\bx^{(k)})_i | > | \sigma^{(k)}_{1,i}|$, if we slightly update the GD procedure.
Namely, when we set all components of $\wt{\bd}^{(k)}$ that do not satisfy \cref{eq:bound_gradient} to zero, then the modified direction vector is still a descent direction.
We do not analyze this modified GD procedure because, in practice, it might be too conservative to stop updating coordinates as soon as they violate \cref{eq:bound_gradient}; in addition, the theoretical analysis appears more technical, although it follows similar arguments.

Together with the Lipschitz continuity property \cref{eq:lpineq}, \cref{eq:bound_gradient} implies the lower bound $\|\wt\bx^{(k)} -\bx^{*}\|\ge L^{-1} \, \|\, \bsigma_1^{(k)}\|$, which means that in the context of low precision, GD may not converge arbitrarily close to the optimal point.
In \cref{sec:Himmelblau's function}, we show an example (cf.~\cref{fig:himmelblau_func}) where GD converges to the exact optimal point when it is representable in finite precision and otherwise converges to a neighborhood of the optimal point whose size depends on $u$. 

\subsection{Recap of the general assumptions}\label{sec:universalassum}
So far, we have discussed several reasonable assumptions concerning the objective function, the number format, and the rounding operation.  To get concise statements in our convergence analysis and to make it easier for the reader to retrieve these properties, in the following, we recap the conditions assumed for all the theoretical statements in this paper.

\begin{framed}
\begin{assumption}\label{assum:universal assumption}  
\quad
\begin{enumerate}
\vspace{-2mm}
\item The objective function $f$ satisfies \cref{eq:lpineq} and \cref{eq:PLineq} with constant $0 < \mu \le \frac12 L$.
\item The evaluation of the components of $\nabla f$ satisfies the assumption of \cref{prop:exptf_csr}. 
Moreover, the parameter $c$ is bounded above by a polynomial function of $L, \chi$, and $n$.
\item For every iteration step, the absolute error $\bsigma_1^{(k)}$ caused by evaluating the gradient satisfies \cref{eq:bound_gradient}.
\item The quantities $\bx^{(0)}$ and $t$ are exactly represented in the chosen number format.
\item All the iterates of GD are in a compact space such that $\|\wt{\bx}^{(k)}-{\bx}^*\| \le \chi$. for all $k$.
\item Overflow is not encountered during computations.
\end{enumerate}
\end{assumption}
\vspace{-3mm} 
\end{framed}
Note that the conditions 1 and 5 in \cref{assum:universal assumption} imply that  $\|\nabla f(\wt{\bx}^{(k)})\|\le L\,\chi$. Overflow is not considered in this study, because the primary focus of this paper is on addressing the stagnation problem of GD, which is often caused by limited precision of fractional bits. In the recent study of NVIDIA \cite{nvidia2023}, it addresses the importance of preventing gradients from being rounded to
zero in the 16-bit representation and it also indicates that even if an overflow occurs, infrequent skipping of weight updates results in the same training accuracy as that of 32-bit training.

\subsection{The magnitude of the updating vector} 
\label{sec:magnit}
Together with the non-opposite sign property discussed in the previous subsection, the magnitude of the entries in $\wt{\mathbf{d}}^{(k)}$ plays a crucial role in analyzing the convergence properties of GD. However, it is reasonable to observe varying magnitudes among the different components of the gradient during the execution of GD. For this reason, we identify and subdivide our theoretical analysis for the following three cases.

\vspace{2mm}
\noindent\textbf{Condition of Case I.}
\vspace{-8.1mm}
\begin{equation}\label{eq:condu_nostagnation}
\qquad\qquad\ |\nabla f(\wt\bx^{(k)})_i+\sigma_{1,i}^{(k)}| \ge \frac{u}{t}, \qquad i=1,\dots, n.
\end{equation}
\ \\[-2mm]
\textbf{Condition of Case II.} 
\vspace{-8.1mm}
\begin{equation}\label{eq:condu_stagnation}
\qquad\qquad\ |\nabla f(\wt\bx^{(k)})_i+\sigma_{1,i}^{(k)}|< \frac{u}{t}, \qquad i=1,\dots, n.
\end{equation} 
\ \\[-2mm]
\textbf{Condition of Case III.} There exist disjoint non-empty subsets of indices $\calc_1, \calc_2$ such that $\calc_1\cup\calc_2=\{1,\dots,n\}$ and
\begin{align}\label{eq:cond_case3}
|\nabla f(\wt\bx^{(k)})_i+\sigma_{1,i}^{(k)}| \ge \frac{u}{t}, \quad \text{for }i\in\calc_1, \qquad  |\nabla f(\wt\bx^{(k)})_i+\sigma_{1,i}^{(k)}|< \frac{u}{t},\quad \text{for }i\in\calc_2. 
\end{align}
As discussed, Case I mostly characterizes the early stages of the updating procedure of GD. Conversely, Cases II and III come into play when, at a later stage of the process, some or all the rounded gradient components approach the rounding precision threshold $u$. For a practical demonstration, in \cref{sec:Himmelblau's function} we demonstrate the occurrence of these three cases in the GD updating process in the context of Himmelblau's function; see, e.g.,  \cref{fig:h4}.
Considering the general expression for GD's iteration \cref{eq:dk}, we can reformulate $\wt{\bd}^{(k)}$ for these three conditions.

\vspace{2mm}
\noindent\textbf{Case I.} In this case the magnitude of the $i$th component of $\wt{\bd}^{(k)}$ (cf.~\cref{eq:condu_nostagnation}) mainly depends on $\nabla f(\wt\bx^{(k)})_i$ and we informally say that \emph{the updating procedure of GD is dominated by the gradient}. Let us rewrite the updating vector as
\begin{align}\label{eq:d_c1}
\wt{\bd}^{(k)}=t\ \nabla f(\wt\bx^{(k)})\,\circ(\mathbf 1+\br^{(k)}),
\end{align}
where $\circ$ is the Hadamard (component-wise) product and $\br^{(k)}$ is a vector of the relative errors with respect to the exact quantity $t\ \nabla f(\wt\bx^{(k)})$, whose entries are given by
\begin{align}
r_i^{(k)}=\frac{t\,\sigma^{(k)}_{1,i}+\sigma^{(k)}_{2,i}}{t\ \nabla f(\wt\bx^{(k)})_i}, \qquad \text{for $i=1,\dots,n$.}\label{eq:h}
\end{align}
Equations~\cref{eq:bound_gradient} and \cref{eq:condu_nostagnation} imply bounds on the entries of $\br^{(k)}$; indeed, we have the following result and the proof is available in Appendix~\ref{ap:appendixA}. 
\begin{lemma}\label{lem:bound_h}
Under the \cref{assum:universal assumption} (cf.~\cref{eq:bound_gradient}) and the condition of Case I \cref{eq:condu_nostagnation}, for $i=1,\dots,n$, we have that $-1< r_i^{(k)}< 3$. 
\end{lemma}

\vspace{2mm}
\noindent\textbf{Case II.}
Condition \cref{eq:condu_stagnation} shows that $\bsigma_2^{(k)}$ has a crucial impact on determining the magnitude of $\wt{\bd}^{(k)}$. All the entries of the updating vector belong to the set $\{0, u, -u\}$ and, if $\rn$ is employed, it is likely that GD stagnates.
In this case, \emph{the updating procedure of GD is dominated by the rounding errors.} The updating vector is generated according to the rule
 \begin{align}\label{eq:d_c2}
 \wt{d}_i^{(k)}=\begin{cases}
 0,&\text{with probability}~p(t\ \nabla f(\bx^{(k)})_i+t\,\sigma_{1,i}^{(k)}),\\
 \mathrm{sign}(\nabla f(\bx^{(k)})_i+\sigma_{1,i}^{(k)})\,u,&\text{with probability}~1-p(t\ \nabla f(\bx^{(k)})_i+t\,\sigma_{1,i}^{(k)}).
 \end{cases}
\end{align}
Here, $p\in\{p_r,p_0,p_{\varepsilon},p_{\varepsilon s}\}$ depends on the rounding method employed, which is introduced in \cref{sec:rounding}. This updating rule is similar to the approach of the \emph{sign gradient descent method} by \cite{moulay2019properties}, except for the stochastic dependence of our scheme.

\vspace{2mm}
\noindent\textbf{Case III.}
The last case refers to situations where the components of the gradient vector have different scales. In this case, we may encounter entries on the level of rounding errors while others are close to the value of the exact partial derivatives.
We informally say that \emph{the updating procedure of GD is partially dominated by the gradient and partially dominated by the rounding errors}. More explicitly, in Case III the updating vector is generated as follows:
\begin{small}
 \begin{align}\label{eq:d_c3}
&i\in\calc_{1}\,\,\Rightarrow\,\, 
 \wt{d}_i^{(k)}=\,t\ \nabla f(\wt\bx^{(k)})_i\,(1+r_i^{(k)});\\
 &i\in\calc_2 \,\, \Rightarrow\,\, \wt{d}_i^{(k)}=\,\begin{cases}
 0,&\text{with probability}~p(t\, \nabla f(\bx^{(k)})_i+t\,\sigma_{1,i}^{(k)}),\\
 \mathrm{sign}(\nabla f(\bx^{(k)})_i+\sigma_{1,i}^{(k)})\,u,&\text{with probability}~1-p(t\, \nabla f(\bx^{(k)})_i+t\,\sigma_{1,i}^{(k)}),\nonumber
 \end{cases}
\end{align}
\end{small}\noindent
which may be viewed as a combination of \cref{eq:d_c1} and \cref{eq:d_c2}.

In the next section, we provide an analysis of the convergence of GD for all three cases,  regarding two aspects: the monotonicity and the convergence rate.

\section{Convergence analysis of GD with fixed-point arithmetic}\label{sec:convergence analy} 
Under \cref{assum:universal assumption}, when using stochastic rounding, the quantities $\wt{\bd}^{(k)}$, $\wt{\bx}^{(k)}$, and $\nabla f(\wt\bx^{(k)})$ can be viewed as random vectors obtained by GD. 
We prove that for the objective functions satisfying \cref{assum:universal assumption}, the linear convergence of GD in exact arithmetic (cf.~\cref{them:plineq_nonconvex}) can be extended to our framework \cref{eq:gd_fp}. We begin by analyzing the updating direction of rounding errors for different stochastic rounding methods. 

\subsection{The direction of stochastic rounding errors}
 Taking the expectation of $\nabla f(\wt\bx^{(k)})^T\wt{\bd}^{(k)}$ in \cref{eq:dk}, we have
\begin{align*}
 \expt\,[\,\nabla f(\wt\bx^{(k)})^T\wt{\bd}^{(k)}\,]=\,&\expt\,[\,\nabla f(\wt\bx^{(k)})^T\,(t\ \nabla f(\wt\bx^{(k)})+t\,\bsigma^{(k)}_1+\bsigma^{(k)}_2)\,]\nonumber\\
 =\,&t\,\expt\,[\, \|\,\nabla f(\wt\bx^{(k)})\, \|^2 \,]+t\,\expt\,[\,\nabla f(\wt\bx^{(k)})^T\bsigma^{(k)}_1\,]+\expt\,[\,\nabla f(\wt\bx^{(k)})^T\bsigma^{(k)}_2 \,].
\end{align*} 
We see that both $\bsigma^{(k)}_1$ and $\bsigma^{(k)}_2$ may result in a different updating direction and distance from the exact updating procedure. One may choose the stepsize $t$ to influence the term $\expt\,[\,\nabla f(\wt\bx^{(k)})^T\bsigma^{(k)}_1\,]$; however, the choice of $t$ only affects the probability of rounding up or down for $\nabla f(\wt\bx^{(k)})^T\bsigma^{(k)}_2$, so that the influence on $\expt\,[\,\nabla f(\wt\bx^{(k)})^T\bsigma^{(k)}_2 \,]$ may be mild.

Now, we study the quantity $\expt\,[\,\nabla f(\wt\bx^{(k)})^T\bsigma_1^{(k)}\,]$ for different stochastic rounding methods. We have the following result; for the proof see Appendix~\ref{ap:appendixA}.
\begin{lemma} \label{lem:csr_sigma1}
Under \cref{assum:universal assumption}, if $\csr$ is applied for evaluating $\bsigma_1$, then it holds 
\begin{align*}
\expt\,[\,\nabla f(\wt\bx^{(k)})^T \, (\nabla f(\wt\bx^{(k)})+\bsigma_{1}^{(k)})\,] \simeq \expt\,[\, \|\nabla f(\wt\bx^{(k)})\|^2 \,], 
\end{align*}
where $\simeq$ is defined in \cref{sec:introduction}.
\end{lemma}
The proof of \cref{lem:csr_sigma1} shows that when $L\chi \sqrt{n}\,u^2$ is negligible, $\csr$ causes the expectation of the inner product $\nabla f(\wt\bx^{(k)})^T \, \csr(\nabla f(\wt\bx^{(k)}))$ to coincide with the corresponding quantity in exact arithmetic, i.e.,
 \[\expt\,[\,\nabla f(\wt\bx^{(k)})^T\,\csr(\nabla f(\wt\bx^{(k)}))\,]\simeq\expt\,[\, \|\,\nabla f(\wt\bx^{(k)})\, \|^2 \,].\]
Unfortunately, this is not the case for $\srh$ because it is hard to determine $\expt\,[\,\bsigma_1\,]$ when the value of $\expt\,[\,\nabla f(\wt\bx^{(k)})\,]$ is unknown. To circumvent this problem-dependent behavior of $\srh$, we only consider the use of $\csr$ to evaluate all the gradients in our analysis. 

Now, let us study the quantity $\expt\,[\,\nabla f(\wt\bx^{(k)})^T\bsigma_2^{(k)}\,]$ under the use of $\csr$ and its zero bias property. 
\begin{lemma}\label{lem:csr_sigma2}
If $\bsigma_2$ is obtained using $\csr$, then it holds $\expt\,[\,\nabla f(\wt\bx^{(k)})^T \bsigma_2^{(k)}\,]=0.$
\end{lemma}
We omit the proof of \cref{lem:csr_sigma2}, since it can be simply derived from $\expt\,[\,\sigma_{2,i}^{(k)}\,\big|\, \nabla f(\wt\bx^{(k)})_i,\sigma_{1,i}^{(k)}\,]=0$ for all $i$ and utilizing the law of total expectation (e.g., \cite[(1.14)]{biagini2016elements}). On the other hand, if we apply $\srh$ we obtain, on average, an additional contribution in an ascent direction. 
\begin{lemma}\label{lem:srh_dk}
If $\bsigma_2$ is the result of $\srh$, $\nabla f(\wt\bx^{(k)})$ is not identically zero, and \cref{eq:bound_gradient} is satisfied, then
$\expt\,[\,\nabla f(\wt\bx^{(k)})^T\bsigma_2^{(k)}\,]>0$.
\end{lemma}
The proof is available in Appendix~\ref{ap:appendixA}. In the next section, we will use the results stated here as building blocks for studying the impact of rounding errors on the convergence of GD.

\subsection{Convergence analysis of GD for the three cases} \label{sec:Convergence analysis of GD for three cases}
Let us proceed to analyze the convergence of GD in the three cases introduced in \cref{sec:magnit}. For each case, we provide conditions to ensure that the sequence generated by GD has decreasing objective function values. Moreover, we quantify the rate of convergence based on the number of steps that GD spends in each case. 
\subsubsection{Case I}\label{sec:convergence analysisStageII} 
We start by observing that the condition of Case I \cref{eq:condu_nostagnation} implies $2\,t\, |\nabla f(\wt\bx^{(k)})_i| \ge u$, $i=1,\dots,n$, which leads to 
\begin{align}\label{eq:uc1}
 \|\nabla f(\wt\bx^{(k)}) \|\ge\tfrac12 \sqrt{n}\,\frac{u}{t}.
\end{align}
In view of \cref{eq:lpineq}, \cref{eq:uc1} implies $\|\bx^{(k)} -\bx^*\|\ge\tfrac12 \,L^{-1}\sqrt{n}\,\frac{u}{t},$ where $\bx^*\in \calx^*$. Together with the inequality \cref{eq:uc1}, it tells us that GD can satisfy the condition of Case I only outside a neighborhood of $\calx^*$.  

We commence the study of the objective function values associated with the sequence generated by GD in finite precision.
Let us denote the minimum ratio between the exact entries and the rounded entries of $t\ \nabla f(\wt\bx^{(j)})$ at the $j$th iterate by 
\begin{align*}
 \gamma_j:=\min_{i=1,\dots, n} \,\, 1+r_i^{(j)},
\end{align*}
so that $\gamma_j\in (0,4)$ (see \cref{lem:bound_h}).
 Adapting the proof of \cite[Theorem 1]{karimi2016linear} to our framework, we obtain the following result for general rounding errors.
\begin{theorem}\label{them:boundforfixed_point_PL}
Under \cref{assum:universal assumption}, consider $k$ iteration steps of GD in fixed-point arithmetic with a fixed stepsize $t \le \frac{1}{4\,L}$. Suppose that the condition of Case I \cref{eq:condu_nostagnation} is satisfied throughout these $k$ iteration steps. Then, the $k$th iterate satisfies
\begin{align}
 f(\wt\bx^{(k)})-f^*\le 
\prod_{j=0}^{k-1}(1-t\,\mu\,\gamma_j)\,(f(\bx^{(0)})-f^*), \label{eq:plbound}
\end{align}
where $0< t\,\mu\,\gamma_j<\frac 12$, $j=0,\dots, k-1$.
\end{theorem}
\begin{proof}
The Lipschitz gradient property (cf.~\cref{eq:lpineq}) and \cref{eq:d_c1} allow us to write 
\begin{align}
f(\wt\bx^{(k+1)})\le \,&f(\wt\bx^{(k)})-\nabla f(\wt\bx^{(k)})^T\wt{\bd}^{(k)}+\tfrac{1}{2}L\, \| \wt{\bd}^{(k)}\|^2 \nonumber\\
=\,&f(\wt\bx^{(k)})-t\,\sum^n_{i=1}(1+r^{(k)}_i)\,(1-\tfrac12 \,L\,t\,(1+ r^{(k)}_i))\,(\nabla f(\wt\bx^{(k)})_i)^2.\nonumber
\end{align}
Since $t\le \frac{1}{4\,L}$ and $0< 1+r_i^{(k)}< 4$ (\cref{lem:bound_h}), we have that
\begin{align}
f(\wt\bx^{(k+1)})\le \,& f(\wt\bx^{(k)})-\tfrac12 \,t\,\sum_{i=1}^{n} (1+r^{(k)}_i)\,(\nabla f(\wt\bx^{(k)})_i)^2,\label{eq:2ndtaylor}
\end{align}
which in turn implies
\begin{align}
f(\wt\bx^{(k+1)})-f^*&\le \,
f(\wt\bx^{(k)})-f^*-\tfrac12 \,t\,\gamma_{k}\, \|\nabla f(\wt\bx^{(k)})\|^2 \nonumber\\
&\underset{\cref{eq:PLineq}}{\le}\,f(\wt\bx^{(k)})-f^*-t\,\mu\,\gamma_{k} \,(f(\wt\bx^{(k)})-f^*)\nonumber\\
&=\,(1-t\,\mu\,\gamma_{k} )\,(f(\wt\bx^{(k)})-f^*).\label{eq:fxkPL}
\end{align}
The assumptions $\mu \le \frac 12 L$, $t\le \frac{1}{4L}$, and $\gamma_k<4$, yield $1-t\,\mu\,\gamma_k\in (\frac{1}{2}, 1)$. By expanding the recursive definition of \cref{eq:fxkPL}, we obtain the required result.
\end{proof}
 Comparing \cref{them:plineq_nonconvex,them:boundforfixed_point_PL}, in the presence of rounding errors, a smaller $t$ is required to ensure the monotonicity of GD. In particular, the bound on $t$ depends on the upper bound on the relative rounding errors of $t\ \nabla f(\wt\bx^{(k)})$. In other words, a larger $r_i^{(k)}$ in \cref{eq:h} leads to a smaller bound on $t$. Moreover, when rounding errors are zeros, i.e., $\gamma_j=1$ in \cref{eq:plbound} for all $j$, \cref{them:boundforfixed_point_PL}yields the same bound on the convergence rate achieved by exact arithmetic (cf.~\cref{eq:plineq_nonconvex}). When $\gamma_j<1$, the rounding errors may result in slower convergence of GD compared to that in \cref{them:plineq_nonconvex}. Furthermore, on the basis of \cref{eq:2ndtaylor}, we see that stagnation may occur only if $r_i^{(j)}=-1$ for all $i$. Since $r_i^{(j)}>-1$ under \cref{assum:universal assumption} and \cref{eq:condu_nostagnation} (cf.~\cref{lem:bound_h}), we have that in Case I, GD does not suffer from stagnation.

Interestingly, \cref{them:boundforfixed_point_PL} also states that a faster convergence rate may be achieved when many of the $\gamma_j$s are larger than $1$ (again, comparing with \cref{eq:plineq_nonconvex}). Next, we show that we may achieve this, on average, when $\bsigma_2$ is the result of $\srh$. Before presenting the result, we introduce and comment on a quantity that is important for our analysis: 
\begin{align}\label{eq:lambda}
 \rho_k\displaystyle:= \min_{i=1,\dots,n}\frac{n\,\expt\,[\,\sigma_{2,i}^{(k)} \nabla f(\wt\bx^{(k)})_i\,]}{\expt\,[\, \|\nabla f(\wt\bx^{(k)})\|^2 \,]}.
\end{align} 
 The value of $\rho_k$ measures the minimum ratio between the expected rounding errors and the expected squared norm of the gradient at the $k$th iteration. According to Lemmas \ref{lem:csr_sigma1} and \ref{lem:csr_sigma2}, it is easy to check that when $\bsigma_2$ is obtained by the use of $\csr$, we have $\rho_j=0$ for $j=1,\dots,k$. When $\srh$ is applied, we can instead rely on the following upper bound. 
\begin{lemma}\label{lem:lambda_bound}
If $\bsigma_2$ is obtained using $\srh$, under the condition of Case I \cref{eq:condu_nostagnation}, we have the inequality $0<\rho_k\le 2\,t\,\eps$.
\end{lemma}
The proof of \cref{lem:lambda_bound} is available in the Appendix~\ref{ap:appendixA}. We are now ready to analyze the convergence rate for $\srh$. To facilitate the analysis, let us denote  the minimum value of $\rho_j$ over the iteration steps as
\begin{align}\label{eq:tau_1}
 \tau_1:=\min_{j=0,\dots,k-1} \rho_j.
\end{align}
Based on \cref{lem:lambda_bound}, we have $\tau_1\in(0, 2\,t\,\eps\,]$. 
\begin{theorem}\label{them:convergencerate_c1_srh}
Under \cref{assum:universal assumption}, after $k$ iteration steps of GD in fixed-precision arithmetic with a fixed stepsize $t$ such that $t\le \frac{1}{4\,L}$ and suppose that the condition of Case I \cref{eq:condu_nostagnation} is satisfied throughout the $k$ iteration steps. If $\bsigma_{1}$ and $\bsigma_2$ in \cref{eq:gd_fp} are obtained using $\csr$ and $\srh$, respectively, then with $\tau_1$ as in \cref{eq:tau_1}, we have 
\begin{equation}\label{eq:plexpbound}
\expt\,[\,f(\wt\bx^{(k)})-f^*\,]
\lesssim(1-\mu\,(t+\tau_1))^k\,(f(\bx^{(0)})-f^*),
\end{equation}
$\tau_1\in(0, 2\,t\,\eps\,]$ and $\mu\,(t+\tau_1)<1$, where $\lesssim$ is the notation described in \cref{sec:introduction}.
\end{theorem}
\begin{proof}
Substituting \cref{eq:h} into \cref{eq:2ndtaylor}, and taking the expectation of both sides, we obtain
\begin{align}
\expt\,[\,f(\wt\bx^{(k+1)})-f^*\,]
\le \,&\expt\,[\,f(\wt\bx^{(k)})-f^*\,]-\tfrac12 \,t\,\expt\,[\, \|\nabla f(\wt\bx^{(k)})\|^2 \,]\nonumber\\&-\tfrac12 \sum_{i=1}^{n}(t\,\expt\,[\,\sigma_{1,i}^{(k)}\nabla f(\wt\bx^{(k)})_i\,]+\expt\,[\,\sigma_{2,i}^{(k)}\nabla f(\wt\bx^{(k)})_i\,]).\label{eq:expectedfxk1}
\end{align}
In view of Lemmas \ref{lem:csr_sigma1} and \ref{lem:srh_dk}, we have $\expt\,[\, \|\nabla f(\wt\bx^{(k)})\|^2 \,]+\expt\,[\,\sigma_{1,i}^{(k)}\,\nabla f(\wt\bx^{(k)})_i\,]\simeq\expt\,[\, \|\nabla f(\wt\bx^{(k)})\|^2 \,]$ and $\expt\,[\,\sigma_{2,i}^{(k)}\,\nabla f(\wt\bx^{(k)})_i\,]>0$. Furthermore, \cref{eq:lambda} and \cref{eq:expectedfxk1} imply that
\begin{align}
\expt\,[\,f(\wt\bx^{(k+1)})-f^*\,]
\lesssim \expt\,[\,f(\wt\bx^{(k)})-f^*\,]-\tfrac12 \,t\,\expt\,[\, \|\nabla f(\wt\bx^{(k)})\|^2 \,]-\tfrac12 \rho_{k}\,\expt\,[\, \|\nabla f(\wt\bx^{(k)})\|^2 \,].\label{eq:expectfxk_srh}
\end{align}
Taking the expectation of \cref{eq:PLineq}, we obtain
\begin{equation}
 \expt\,[\, \|\nabla f(\wt\bx^{(k)})\|^2\,] \ge 2 \,\mu\, \expt\,[\,f(\wt\bx^{(k)})-f^*\,].\label{eq:PLineq_expec}
\end{equation}
Hence, substituting \cref{eq:PLineq_expec} into \cref{eq:expectfxk_srh}, we have
\[
\expt\,[\,f(\wt\bx^{(k+1)})-f^*\,]
\lesssim(1-(t+\rho_k)\,\mu)\ \expt\,[\,f(\wt\bx^{(k)})-f^*\,].
\]
Expanding the recursion, we obtain 
\[
\expt\,[\,f(\wt\bx^{(k)})-f^*\,]
\lesssim\prod_{j=0}^{k-1}(1-(t+\rho_j)\,\mu)\ \expt\,[\,f(\bx^{(0)})-f^*\,].
\]
According to \cref{lem:lambda_bound}, we have the claim for $\tau_1\in(0, 2\,t\,\eps\,]$. Furthermore, the properties $\varepsilon<1$, $\rho_j\le 2\,t\,\varepsilon$ and $\mu\le L/2$ yield $(t+\rho_j)\,\mu\le \frac{1+2\varepsilon}{8}<\frac{3}{8}$. 
\end{proof}
Looking at \cref{eq:lambda}, a larger value of $\eps$ in $\srh$ might allow a larger bound for $\tau_1$. Comparing \cref{them:plineq_nonconvex} (cf.~\cref{eq:plineq_nonconvex}) and \cref{them:convergencerate_c1_srh} (cf.~\cref{eq:plexpbound}), when choosing the same $t$, we have that a tighter bound on convergence rate, in expectation, is obtained by using $\srh$. 

When both $\bsigma_{1}$ and $\bsigma_2$ are obtained by $\csr$, \cref{lem:csr_sigma2} implies $\expt\,[\,\sigma_{2,i}^{(k)}\,\nabla f(\wt\bx^{(k)})_i\,]=0$, which yields $\rho_k=0$ in \cref{eq:expectedfxk1}. Together with \cref{lem:csr_sigma1}, this gives, in expectation, the same convergence rate that holds for infinite-precision computations.
\begin{corollary}\label{coro:convergencerate_c1_csr}
Under the same assumptions of \cref{them:convergencerate_c1_srh}, if $\bsigma_2$ is obtained using $\csr$, then we have 
\[
\expt\,[\,f(\wt\bx^{(k)})-f^*\,]
\lesssim (1-t\,\mu)^k \, (f(\bx^{(0)})-f^*).
\]
\end{corollary} 
We omit the proof of \cref{coro:convergencerate_c1_csr} since it can be easily obtained on the basis of \cref{lem:csr_sigma2} and by proceeding analogously to the proof of \cref{them:convergencerate_c1_srh}.
 Both \cref{them:convergencerate_c1_srh} and \cref{coro:convergencerate_c1_csr}
 indicate that smaller values of $t$ are required to ensure the convergence of GD compared to the case of the exact computation. This is because their analysis relies on the worst-case bound of $\bsigma_1$. In the simulation studies reported in \cref{sec:simulation}, we will demonstrate that these restrictions are usually pessimistic and, in practice, $t$ can be chosen as large as the upper bound in \cref{them:plineq_nonconvex}, that is $t\le \frac{1}{L}$.

\subsubsection{Case II}\label{sec:case2} 
In Case II, each of the entries of the updating vector $\wt{\bd}^{(k)}$ takes one of the values in $\{0, u, -u\}$; when it is $0$, property \cref{eq:bound_gradient} implies that $\nabla f(\wt\bx^{(k)})_i=0$. Therefore, we have $\sigma_{2,i}=0$, which does not influence the convergence of GD. For our analysis, it is convenient to single out the cases where the rounded gradient has some components that are exactly zero. More formally, we denote the finite set of nonzero values that the $i$th component of $\nabla f(\wt\bx^{(k)})$ may assume by $\calw_i^{(k)}$. Note that $\calw^{(k)}_{i}$ can be empty for some entries of $\nabla f(\wt\bx^{(k)})$, but when it is empty for all $i$, in accordance with \cref{eq:bound_gradient}, it implies that $\nabla f(\wt\bx^{(k)})=0$ and GD has reached a stationary point. 
Additionally, \cref{eq:d_c2} shows that GD is highly likely to reach a state of stagnation when the quantity $t\ \nabla f(\bx^{(k)})_i+t\,\sigma_{1,i}^{(k)}$ is small. In particular, when $t\ \nabla f(\bx^{(k)})_i+t\,\sigma_{1,i}^{(k)}<\frac{u}{2}$, the use of RN will cause $\wt{d}_i^{(k)}=0$, resulting in the stagnation of GD. In the following propositions, we show that the use of stochastic rounding can ensure the monotonicity of the objective function and achieve linear convergence of GD.
Before we start our analysis, we recall a basic property of conditional expectation to facilitate further proof.

For random variables $X, Y, Z$ we have (see, e.g., \cite[(10.40)]{steyer2017probability})
\begin{align}\label{eq:towerprop}
\expt\,[\,\expt\,[\,X\ \big| \ Y, Z\,] \ \big| \ Y\,]=\expt\,[\,X\ \big| \ Y\,].
\end{align}
When stochastic rounding is employed, our aim is to prove that GD achieves a linear convergence rate until it reaches a neighborhood around $\bx^*$. To shed light on the size of this neighborhood, we introduce the quantity 
\begin{align}\label{eq:case2_theta}
 \theta_k:= \min_{i= 1, \dots, n}\min_{\nabla f(\wt\bx^{(k)})_i\in \calw^{(k)}_i} \frac{2\ |\nabla f(\wt\bx^{(k)})_i|-L\,u}{|\nabla f(\wt\bx^{(k)})_i|},
\end{align}
which implies that $ \theta_k<2$.
The value of $\theta_k$ is approximately equal to $2$ when the gradient's components are large compared to $L\,u$, which is typically the case at the beginning of the process. Only when at least one entry of the gradient gets close to $0$,  $\theta_k$ may be no longer positive. In this case, where some of the gradient components are near zero, rounding errors can play a positive or negative role in determining whether GD converges, stagnates, or oscillates. When these rounding errors are intentionally designed to align in a descent direction, such as when using $\srh$, then it may help the optimization process; see for instance \cref{fig:fu4}.

In the subsequent convergence analysis, we show that the condition $\theta_k> 0$ provides, in expectation, strict monotonicity and linear convergence rate. Note that $\theta_k> 0$ implies that $|\nabla f(\wt\bx^{(k)})_i|>\frac12 \,L\,u$. This observation does not contradict condition \cref{eq:condu_stagnation}. This is due to the fact that \cref{eq:condu_stagnation} provides an upper bound on the rounded gradient, which means that even if $|\nabla f(\wt\bx^{(k)})_i|$ is larger than $\frac12 \, L\,u$, the rounded gradient can still be small, possibly even reaching 0. We begin with the monotonicity when using $\csr$. 
\begin{proposition}\label{prop:monoto_csr}
Under \cref{assum:universal assumption}, after $k$ iteration steps of GD in fixed-precision arithmetic with a fixed stepsize $t\le \frac{1}{L}$, suppose that the condition of Case II \cref{eq:condu_stagnation} has been satisfied throughout the $k$ iteration steps. If $\bsigma_{1}$ and $\bsigma_2$ in \cref{eq:gd_fp} are obtained by $\csr$ and $\theta_k> 0$ in \cref{eq:case2_theta} then
\[\expt\,[\,f(\wt\bx^{(k)})\,]\ < \ \expt\,[\,f(\wt\bx^{(k-1)})\,].
\]
\end{proposition}
\begin{proof}
The Lipschitz gradient property \cref{eq:lpineq}, rephrased in terms of expectations, yields
\begin{align}
\expt\,[\,f(\wt\bx^{(k+1)})\,]\le \,&\expt\,[\,f(\wt\bx^{(k)})\,]-\sum_{i=1}^{n}\expt\,[\,\nabla f(\wt\bx^{(k)})_i\,\wt{d}_i^{(k)}-\tfrac12 \,L\,(\wt{d}_i^{(k)})^2 \,].
\label{eq:csrexptd_caseII}
\end{align}
According to \cref{eq:towerprop}, we have \[\expt\,[\,\nabla f(\wt\bx^{(k)})_i\,\sigma_{2,i}^{(k)}\ \big| \ \nabla f(\wt\bx^{(k)})_i]\,=\expt\,[\,\expt\,[\,\nabla f(\wt\bx^{(k)})_i\,\sigma_{2,i}^{(k)}\ \big| \ \nabla f(\wt\bx^{(k)})_i,\, \sigma_{1,i}^{(k)}\,]\ \big| \ \nabla f(\wt\bx^{(k)})_i]=0.
 \]
Therefore, letting $q$ vary over $\calw^{(k)}_i$, the set of all the possible values of $\nabla f(\wt\bx^{(k)})_i$, and utilizing the law of total expectation, we have
\begin{align}\label{eq:c2csr_fisrtpart}
 &\expt\,[\,\nabla f(\wt\bx^{(k)})_i\,\wt{d}_i^{(k)}\,]\nonumber\\
 &=\sum_{q\,\in\calw^{(k)}_i }\expt\,[\,\nabla f(\wt\bx^{(k)})_i\,\wt{d}_i^{(k)}\ \big| \ \nabla f(\wt\bx^{(k)})_i=q\,]\,P(\nabla f(\wt\bx^{(k)})_i=q)\nonumber\\
 &=\sum_{q\,\in\calw^{(k)}_i }\expt\,[\,\nabla f(\wt\bx^{(k)})_i\,(t\ \nabla f(\wt\bx^{(k)})_i+t\,\sigma^{(k)}_{1,i})\ \big| \ \nabla f(\wt\bx^{(k)})_i=q\,]\,P(\,\nabla f(\wt\bx^{(k)})_i=q\,).
\end{align}
Similarly, we have
\[
 \expt\,[\,(\wt{d}_i^{(k)})^2 \,] =\sum_{q\,\in\calw^{(k)}_i }\expt\,[\,(\wt{d}_i^{(k)})^2 \ \big| \ \nabla f(\wt\bx^{(k)})_i=q\,]\,P(\,\nabla f(\wt\bx^{(k)})_i=q\,).
\]
Based on \cref{eq:d_c2}, when $\nabla f(\wt\bx^{(k)})_i+\sigma^{(k)}_{1,i}>0$, we have 
\begin{small}
 \begin{align*}
 (\wt{d}_i^{(k)})^2=\csr(t\ \nabla f(\wt\bx^{(k)})_i+t\,\sigma^{(k)}_{1,i})^2=\begin{cases} 0,&\text{with probability}\, p_0(t\ \nabla f(\wt\bx^{(k)})_i+t\,\sigma^{(k)}_{1,i}),\\
 u^2,&\text{with probability}\, 1-p_0(t\ \nabla f(\wt\bx^{(k)})_i+t\,\sigma^{(k)}_{1,i}).
 \end{cases}
\end{align*}
\end{small}\noindent
Replacing $x$ by $t\ \nabla f(\wt\bx^{(k)})_i+t\,\sigma^{(k)}_{1,i}$ in \cref{eq:csr}, we have
\[
 \expt\,[\,\wt{d}_i^{(k)}\ \big| \ \nabla f(\wt\bx^{(k)})_i,\sigma^{(k)}_{1,i}\,]=t\ \nabla f(\wt\bx^{(k)})_i+t\,\sigma^{(k)}_{1,i}= u\,(1-p_0(t\ \nabla f(\wt\bx^{(k)})_i+t\,\sigma^{(k)}_{1,i})),
\]
so that\begin{small}
\begin{align}\label{eq:csrd^2}
 \expt\,[\,(\wt{d}_i^{(k)})^2 \ \big| \ \nabla f(\wt\bx^{(k)})_i,\sigma^{(k)}_{1,i}\,]=u^2 \,(1-p_0(t\ \nabla f(\wt\bx^{(k)})_i+t\,\sigma^{(k)}_{1,i})\,)=u\, | t\ \nabla f(\wt\bx^{(k)})_i+t\,\sigma^{(k)}_{1,i}|.
\end{align}\end{small}\noindent
It is easy to check that \cref{eq:csrd^2} also holds for the condition when $\nabla f(\wt\bx^{(k)})_i+\sigma^{(k)}_{1,i}<0$.
In view of \cref{eq:towerprop}, we have
\[
 \expt\,[\,(\wt{d}_i^{(k)})^2 \ \big| \ \nabla f(\wt\bx^{(k)})_i\,]=\expt\,[\,\expt\,[\,(\wt{d}_i^{(k)})^2 \ \big| \ \nabla f(\wt\bx^{(k)})_i,\,\sigma^{(k)}_{1,i}\,]\ \big| \ \nabla f(\wt\bx^{(k)})_i\,],
\]
 which leads to
\[
 \expt\,[\,(\wt{d}_i^{(k)})^2 \,]=\sum_{q\,\in\calw^{(k)}_i}\expt\,[\, | t\ \nabla f(\wt\bx^{(k)})_i+t\,\sigma^{(k)}_{1,i}|\,u\ \big| \ \nabla f(\wt\bx^{(k)})_i=q\,]\,P(\nabla f(\wt\bx^{(k)})_i=q).\]
Equation \cref{eq:bound_gradient} implies that
\[
\nabla f(\wt\bx^{(k)})_i\,(t\ \nabla f(\wt\bx^{(k)})_i+t\,\sigma^{(k)}_{1,i})=|\nabla f(\wt\bx^{(k)})_i|\,\, |\,t\ \nabla f(\wt\bx^{(k)})_i+t\,\sigma^{(k)}_{1,i}|,
\]
and together with \cref{eq:c2csr_fisrtpart}, it yields 
\begin{small}
 \begin{align}\label{eq:c2_lastpart}
 &\expt\,[\,\nabla f(\wt\bx^{(k)})_i\,\wt{d}_i^{(k)}-\tfrac12 \,L\,(\wt{d}_i^{(k)})^2 \,]\\
 &=\sum_{q\,\in\calw^{(k)}_i}\!\!\!\expt\,[\, |\,t\ \nabla f(\wt\bx^{(k)})_i+t\,\sigma^{(k)}_{1,i}|\,(\, |\nabla f(\wt\bx^{(k)})_i|-\tfrac12 \,L\,u)\ \big| \ \nabla f(\wt\bx^{(k)})_i=q\,]\,P(\nabla f(\wt\bx^{(k)})_i=q).\nonumber
\end{align}
\end{small}\noindent
Finally, we observe that $\nabla f(\wt\bx^{(k)})_i\in\calw^{(k)}_{i}$ implies $\nabla f(\wt\bx^{(k)})_i\ne 0$, and also $\theta_k> 0$ implies $\expt\,[\,\nabla f(\wt\bx^{(k)})_i\,\wt{d}_i^{(k)}-\tfrac12 \,L\,(\wt{d}_i^{(k)})^2 \,]>0$, $i=1,\dots,n$, which together with \eqref{eq:csrexptd_caseII} implies that $ \expt\,[\,f(\wt\bx^{(k)})\,]> \expt\,[\,f(\wt\bx^{(k+1)})\,]$.
\end{proof}
Equation \cref{eq:c2_lastpart} tells us that we might observe an oscillatory or stagnating behavior when $\theta_k\le0$; for a detailed numerical demonstration, see \cref{fig:u4,fig:u6}. Concerning the convergence rate obtained by $\csr$, we prove the following bound. 
\begin{proposition}\label{prop:convergencerate_c2_csr}
 Under the same assumptions of \cref{prop:monoto_csr}, if $\theta_j> 0$, $j=0,\dots, k-1$, then 
 \begin{align}
\expt\,[\,f(\wt\bx^{(k)})-f^*\,]\lesssim\,&\prod_{j=0}^{k-1}(1-t\,\mu\,\theta_j)\,(f(\bx^{(0)})-f^*),\label{eq:conv-expcsrplstageII}
\end{align}
where $0<1-t\,\mu\,\theta_j<1$.
\end{proposition}
\begin{proof}
On the basis of \cref{eq:case2_theta} and \cref{eq:c2_lastpart}, we have
\begin{align}\label{eq:c2csr_ith}
 \expt\,[\,\nabla &f(\wt\bx^{(k)})_i\,\wt{d}_i^{(k)}-\tfrac12 \,L\,(\wt{d}_i^{(k)})^2 \,]\nonumber\\
\ge\,&\tfrac12 \,\theta_k\sum_{q\,\in\calw^{(k)}_i}\expt\,[\,(\,t\ \nabla f(\wt\bx^{(k)})_i+t\,\sigma^{(k)}_{1,i})\nabla f(\wt\bx^{(k)})_i\ \big| \ \nabla f(\wt\bx^{(k)})_i=q\,]\,P(\nabla f(\wt\bx^{(k)})_i=q)\nonumber\\
 =\,&\tfrac12 \,\theta_k\,t\,\expt\,[\,\nabla f(\wt\bx^{(k)})_i^2+\nabla f(\wt\bx^{(k)})_i\,\sigma^{(k)}_{1,i}\,]
 \underset{\quad\text{\cref{lem:csr_sigma1}}}{\,\,\simeq}\,\tfrac12 \,\theta_k\,t\,\expt\,[\,\nabla f(\wt\bx^{(k)})_i^2 \,].
\end{align}
Substituting it into \cref{eq:csrexptd_caseII}, we obtain
\begin{align}
\expt\,[\,f(\wt\bx^{(k+1)})\,]\lesssim\,&\expt\,[\,f(\wt\bx^{(k)})\,]-\sum_{i=1}^{n}\tfrac12 \,\theta_k\,t\,\expt\,[\,\nabla f(\wt\bx^{(k)})_i^2 \,]\nonumber\\
=\,&\expt\,[\,f(\wt\bx^{(k)})\,]-\tfrac12 \,\theta_k\,t\,\expt\,[\, \|\nabla f(\wt\bx^{(k)})\|^2 \,].\nonumber
\end{align}
Substituting \cref{eq:PLineq_expec} into it, we get
\[
 \expt\,[\,f(\wt\bx^{(k+1)})-f^*\,]\lesssim(1-t\,\mu\,\theta_k)\,\expt\,[\,f(\wt\bx^{(k)})-f^*\,].
\]
Expanding the recursion, we obtain \cref{eq:conv-expcsrplstageII}.
Since $\theta_j,t,\mu>0$ we have $1-t\,\mu\,\theta_j<1$. Finally, $\theta_j<2$ and $t\le \frac{1}{L}\le \frac{1}{2 \,\mu}$ imply $1-t\,\mu\,\theta_j>0$.
\end{proof}
\cref{prop:convergencerate_c2_csr} demonstrates that, under the condition $\theta_j>0$, GD converges linearly to a neighborhood of the optimal point.
Looking at \cref{eq:conv-expcsrplstageII}, we see that larger values of $\theta_j$ lead to tighter bounds for the expected convergence rate of GD. Comparing \cref{eq:plineq_nonconvex} and \cref{eq:conv-expcsrplstageII}, when $\theta_j>1$ for many iteration steps $j$, it is likely to achieve a stricter bound than the one available for exact arithmetic; on the contrary, when $\theta_j<1$ for many indices $j$ we probably get a larger bound. Next, we prove that under the same condition as \cref{prop:monoto_csr}, the monotonicity of GD is also guaranteed when employing $\srh$ for $\bsigma_2$. 
\begin{proposition}\label{prop:monoto_srh}
 Under the same assumptions of \cref{prop:monoto_csr}, but instead of using $\csr$ now $\bsigma_2$ is obtained by $\srh$, if $\theta_k > 0$ in \cref{eq:case2_theta} then
 \[
 \expt\,[\,f(\wt\bx^{(k)})\,]< \expt\,[\,f(\wt\bx^{(k-1)})\,].
 \]
\end{proposition}
\begin{proof}
 Denote by $\cals$ the finite set of values that may be taken by $t\,(\nabla f(\wt\bx^{(k)})_i+\sigma_{1,i}^{(k)})$. Let $\cals_{1}$ indicate the subset of values of $\cals$ such that $0<p_\eps(t\,(\nabla f(\wt\bx^{(k)})_i+\sigma_{1,i}^{(k)}))<1$. Analogously, we define the subsets $\cals_{2}$ and $\cals_{3}$ corresponding to the conditions $p_\eps=0$ and $p_\eps=1$, respectively.
 Finally, we introduce the quantity
 \[
 \omega_i^{(k)}:=u^{-1}(u-|\,t\,(\nabla f(\wt\bx^{(k)})_i+\sigma_{1,i}^{(k)})\, |).
 \]
Note that the numerator of $\omega_i^{(k)}$ indicates the distance of $t\,(\nabla f(\wt\bx^{(k)})_i+\sigma_{1,i}^{(k)})$ to $u$ and $-u$ for $t\,(\nabla f(\wt\bx^{(k)})_i+\sigma_{i,1}^{(k)})\in \cals_{2}$ and $t\,(\nabla f(\wt\bx^{(k)})_i+\sigma_{i,1}^{(k)})\in \cals_{3}$, respectively. According to the definition of $\srh$, see \cref{eq:epsilon}, when $p_{\eps}(t\ \nabla f(\wt\bx^{(k)})_i+t\,\sigma^{(k)}_{1,i})$ equals $0$ or $1$ it follows that $\omega_i^{(k)}<\eps$. Further, since $t\,(\nabla f(\wt\bx^{(k)})_i+\sigma_{i,1}^{(k)})\in\cals_{2}$ and $t\,(\nabla f(\wt\bx^{(k)})_i+\sigma_{i,1}^{(k)})\in\cals_{3}$ are mutually exclusive events, we have $P(t\,(\nabla f(\wt\bx^{(k)})_i+\sigma_{1,i}^{(k)})\in\{S_{2}\cup S_{3}\})=P(t\,(\nabla f(\wt\bx^{(k)})_i+\sigma_{1,i}^{(k)})\in S_{2})+P(t\,(\nabla f(\wt\bx^{(k)})_i+\sigma_{1,i}^{(k)})\in S_{3})$. Therefore, we define the real-valued function $h$ as the map that associates with the random variable $\nabla f(\wt\bx^{(k)})_i+\sigma^{(k)}_{1,i}$ the quantity: \begin{small}
 \begin{equation}\label{eq:g}
 h(\nabla f(\wt\bx^{(k)})_i+\sigma^{(k)}_{1,i}):=\eps\,P(t\,(\nabla f(\wt\bx^{(k)})_i+\sigma_{1,i}^{(k)})\in\cals_1)+\omega_i^{(k)}\,P(t\,(\nabla f(\wt\bx^{(k)})_i+\sigma_{1,i}^{(k)})\in\{S_2 \cup S_3\}).
 \end{equation}
\end{small}\noindent
 When $\nabla f(\wt\bx^{(k)})_i+\sigma^{(k)}_{1,i}$ is not an identically zero random variable, we have
 \begin{align}
 \expt\,[\,\sigma^{(k)}_{2,i}\ \big|\ \nabla f(\wt\bx^{(k)})_i,\,\sigma^{(k)}_{1,i}\,]&=u\ \sign(\nabla f(\wt\bx^{(k)})_i)\,(\,
 \eps\,P(t\,(\nabla f(\wt\bx^{(k)})_i+\sigma_{1,i}^{(k)})\in\cals_1)\nonumber\\&\qquad\qquad\qquad+\omega_i^{(k)}\,u\ P(t\,(\nabla f(\wt\bx^{(k)})_i+\sigma_{1,i}^{(k)})\in\{S_2 \cup S_3\})\,)\nonumber\\
 &\,=\,u\ \sign(\nabla f(\wt\bx^{(k)})_i)\,h(\nabla f(\wt\bx^{(k)})_i+\sigma^{(k)}_{1,i})\nonumber.
 \end{align}
 Proceeding analogously as for \cref{eq:c2csr_fisrtpart}, we obtain
 \begin{align}\label{eq:c2srh_fisrtpart}
 \expt&\,[\,\nabla f(\wt\bx^{(k)})_i\,\wt{d}_i^{(k)}\,]\\
 &=\sum_{q\,\in\calw^{(k)}_i }\expt\,[\,\nabla f(\wt\bx^{(k)})_i\,(t\ \nabla f(\wt\bx^{(k)})_i+t\,\sigma^{(k)}_{1,i})\ 
 \big| \ \nabla f(\wt\bx^{(k)})_i=q\,]\,P(\,\nabla f(\wt\bx^{(k)})_i=q\,)\nonumber\\&\quad+\sum_{q\,\in\calw^{(k)}_i }\expt\,[\,u\, |\nabla f(\wt\bx^{(k)})_i |\,h(\nabla f(\wt\bx^{(k)})_i+\sigma^{(k)}_{1,i}) \ \big| \ \nabla f(\wt\bx^{(k)})_i=q\,]\,P(\,\nabla f(\wt\bx^{(k)})_i=q\,).\nonumber
\end{align}
Following a similar procedure to the one obtaining \cref{eq:csrd^2}, with $\srh$, we have
 \begin{align}
 \expt\,[\,(\wt{d}_i^{(k)})^2&\ \big| \ \nabla f(\wt\bx^{(k)})_i,\,\sigma^{(k)}_{1,i}\,]=\,u^2 \, (1-p_{\eps}(t\ \nabla f(\wt\bx^{(k)})_i+t\,\sigma^{(k)}_{1,i})\,)\nonumber\\
 &=(u\, |\,t\ \nabla f(\wt\bx^{(k)})_i+t\,\sigma^{(k)}_{1,i}\, |+\eps\,u^2)\,P(\nabla f(\wt\bx^{(k)})_i+\sigma^{(k)}_{1,i}\in\cals_1)\nonumber\\&\quad\quad+(u\, |\,t\ \nabla f(\wt\bx^{(k)})_i+\sigma^{(k)}_{1,i}\, |+\omega_i^{(k)}\,u^2)\,P(\nabla f(\wt\bx^{(k)})_i+t\,\sigma^{(k)}_{1,i}\in\{\cals_2 \cup\cals_3\}).\nonumber\\
 &=u\, |\,t\ \nabla f(\wt\bx^{(k)})_i+t\,\sigma^{(k)}_{1,i}\, |+u^2 \,h(\nabla f(\wt\bx^{(k)})_i+\sigma^{(k)}_{1,i}).\nonumber
\end{align}
Therefore:
\begin{align}
 \expt\,[\,(\wt{d}_i^{(k)})^2 \,]=\,&\sum_{q\,\in\calw^{(k)}_i }\expt\,[\,u\, |\,t\ \nabla f(\wt\bx^{(k)})_i+t\,\sigma^{(k)}_{1,i}| \ \big|\ \nabla f(\wt\bx^{(k)})_i=q\,]\,P(\,\nabla f(\wt\bx^{(k)})_i=q\,)\nonumber\\&+\sum_{q\,\in\calw^{(k)}_i }\expt\,[\,u^2 \,h(\nabla f(\wt\bx^{(k)})_i+\sigma^{(k)}_{1,i})\ \big|\ \nabla f(\wt\bx^{(k)})_i=q\,]\,P(\,\nabla f(\wt\bx^{(k)})_i=q\,).\nonumber
\end{align}
Together with \cref{eq:c2srh_fisrtpart} and in accordance with \cref{eq:towerprop}, we obtain
\begin{small}
 \begin{align}\label{eq:c2srh_lastpart}
 &\expt\,[\,\nabla f(\wt\bx^{(k)})_i\,\wt{d}_i^{(k)}-\tfrac12 \,L\,(\wt{d}_i^{(k)})^2 \,]\\
 &=\sum_{q\,\in\calw^{(k)}_i }\expt\,[\,\nabla f(\wt\bx^{(k)})_i\,\wt{d}_i^{(k)}-\tfrac12 \,L\,(\wt{d}_i^{(k)})^2 \ \big|\ \nabla f(\wt\bx^{(k)})_i=q\,]\,P(\,\nabla f(\wt\bx^{(k)})_i=q\,)\nonumber\\
 &=\!\!\!\sum_{q\,\in\calw^{(k)}_i }\expt\,[\, |\,t\ \nabla f(\wt\bx^{(k)})_i+t\,\sigma^{(k)}_{1,i}|\,(\, |\nabla f(\wt\bx^{(k)})_i|-\tfrac12 \,L\,u)\ \big|\ \nabla f(\wt\bx^{(k)})_i=q\,]\,P(\,\nabla f(\wt\bx^{(k)})_i=q\,)\nonumber\\
&+\!\!\!\sum_{q\,\in\calw^{(k)}_i }\!\!\expt\,[\,(u\, |\nabla f(\wt\bx^{(k)})_i|-\tfrac12 \,L\,u^2)\,h(\nabla f(\wt\bx^{(k)})_i+\sigma^{(k)}_{1,i})\ \big|\ \nabla f(\wt\bx^{(k)})_i=q\,]\,P(\,\nabla f(\wt\bx^{(k)})_i=q\,).\nonumber
\end{align}
\end{small}\noindent
The property $\theta_k> 0$ implies $\expt\,[\,\nabla f(\wt\bx^{(k)})_i\,\wt{d}_i^{(k)}-\tfrac12 \,L\,(\wt{d}_i^{(k)})^2 \,]> 0$, validating the claim.
\end{proof}
\cref{prop:monoto_srh} shows that under the same conditions as in \cref{prop:monoto_csr}, the strict monotonicity of GD can be also obtained by applying $\srh$. Further, we show that a stricter bound on convergence can be realized by utilizing $\srh$ compared to that of $\csr$. First let us define (cf.~\cref{eq:g})
\begin{align*}
 \beta_k:=\min_{i=1,\dots,n}\, h(\nabla f(\wt\bx^{(k)})_i+\sigma^{(k)}_{1,i}),
\end{align*}
where $\beta_k\le \eps$ for all $k$ and define
\begin{align}\label{eq:tau_2}
 \tau_2:=\min_{j=0,\dots,k}\frac{\beta_j\,u\ \expt\,[\, \|\nabla f(\wt\bx^{(j)})\|\,]}{\expt\,[\, \|\nabla f(\wt\bx^{(j)})\|^2 \,]}.
\end{align}
When $\theta_k> 0$ in \cref{eq:case2_theta}, it implies that $\expt\,[\, \|\nabla f(\wt\bx^{(k)})\|\,]> \frac{1}{2}\,L\,u$. On the basis of Jensen's inequality, we get 
\[
 \tau_2 < \frac{2 \,\beta_j}{L}\underset{L\,t\le1}{\le}2\,t\,\beta_j\le 2\,t\,\eps, \quad \text{for all $j=1,\dots, k$}.
\]
With \cref{eq:tau_2}, we may achieve a stricter bound on the convergence of GD.
\begin{proposition}\label{prop:convergencerate_c2_srh}
 Under the same assumption of \cref{prop:monoto_srh} but instead of $t\le \frac{1}{L}$ we have $t\le \frac{1}{(1+2\varepsilon)\,L}$, if $\theta_j>0$ in \cref{eq:case2_theta} for all $j$, then with $\tau_2$ as in \cref{eq:tau_2}, we have 
 \begin{align}\label{eq:fsrh_caseII}
 \expt\,[\,f(\wt\bx^{(k)})-f^*\,]\lesssim\,&\prod_{j=0}^{k-1}(1-(t+\tau_2)\,\mu\,\theta_j)\,(f(\bx^{(0)})-f^*),
\end{align}
$\tau_2 \in(0,2\,t\,\eps)$ and $0<(t+\tau_2)\,\mu\,\theta_j<1$.
\end{proposition}
\begin{proof}
Substituting $\theta_k$ (cf.~\cref{eq:case2_theta}) and $\beta_k$ into \cref{eq:c2srh_lastpart}, we get 
\begin{align}\label{eq:c2srh_ith}
 \expt\,&[\,\nabla  f(\wt\bx^{(k)})_i\,\wt{d}_i^{(k)}-\tfrac12 \,L\,(\wt{d}_i^{(k)})^2 \,]\nonumber\\
 &\ge\tfrac12 \,\theta_k\sum_{q\, \in\calw^{(k)}_i }\expt\,[\,(\,t\ \nabla f(\wt\bx^{(k)})_i+t\,\sigma^{(k)}_{1,i})\nabla f(\wt\bx^{(k)})_i\ \big| \ \nabla f(\wt\bx^{(k)})_i=q\,]\,P(\,\nabla f(\wt\bx^{(k)})_i=q\,)\nonumber\\
 &\qquad\qquad+\tfrac12 \,\theta_k\sum_{q\, \in\calw^{(k)}_i }\expt\,[\,\beta_k\, |\nabla f(\wt\bx^{(k)})_i |\,u\ \big| \ \nabla f(\wt\bx^{(k)})_i=q\,]\,P(\,\nabla f(\wt\bx^{(k)})_i=q\,)\nonumber\\
 &=\tfrac12 \,\theta_k\,\expt\,[\,t\ \nabla f(\wt\bx^{(k)})_i^2+t\ \nabla f(\wt\bx^{(k)})_i\,\sigma^{(k)}_{1,i}+\beta_k\,u\, |\nabla f(\wt\bx^{(k)})_i|\,]\nonumber\\
 &\simeq\tfrac12 \,\theta_k\,(t\,\expt\,[\,\nabla f(\wt\bx^{(k)})_i^2 \,]+\beta_k\,u\ \expt\,[\, |\nabla f(\wt\bx^{(k)})_i|\,]), \quad\text{from \cref{lem:csr_sigma1}}.
\end{align}
Substituting it into \cref{eq:csrexptd_caseII}, we obtain
\begin{align}
\expt\,[\,f(\wt\bx^{(k+1)})\,]\lesssim\,&\expt\,[\,f(\wt\bx^{(k)})\,]-\sum_{i=1}^{n}\tfrac12 \,\theta_k\,t\,\expt\,[\,\nabla f(\wt\bx^{(k)})_i^2 \,]-\sum_{i=1}^{n}\tfrac12 \,\theta_k\,\beta_k\,u\ \expt\,[\, |\nabla f(\wt\bx^{(k)})_i|\,]\nonumber\\
\le \,&\expt\,[\,f(\wt\bx^{(k)})\,]-\tfrac12 \,\theta_k\,(t\,\expt\,[\, \|\nabla f(\wt\bx^{(k)})\|^2 \,]+\beta_k\,u\ \expt\,[\, \|\nabla f(\wt\bx^{(k)})\|_1\,])\nonumber\\
\le \,&\expt\,[\,f(\wt\bx^{(k)})\,]-\tfrac12 \,\theta_k\,(t\,\expt\,[\, \|\nabla f(\wt\bx^{(k)})\|^2 \,]+\beta_k\,u\ \expt\,[\, \|\nabla f(\wt\bx^{(k)})\|\,]).\label{eq:c2_srh_1}
\end{align}
By Jensen's inequality, we have 
\[
\tau_2\le\frac{\beta_k\,u\ \expt\,[\, \|\nabla f(\wt\bx^{(k)})\|\,]}{\expt\,[\, \|\nabla f(\wt\bx^{(k)})\|^2 \,]}\le \frac{\beta_k\,u\ \expt\,[\, \|\nabla f(\wt\bx^{(k)})\|\,]}{\expt\,[\, \|\nabla f(\wt\bx^{(k)})\|\,]^2}=\frac{\beta_k\,u}{\expt\,[\, \|\nabla f(\wt\bx^{(k)})\|\,]}.
\]
Substituting \cref{eq:tau_2} into \cref{eq:c2_srh_1}, we obtain 
\[
\expt\,[\,f(\wt\bx^{(k+1)})\,]
\lesssim\expt\,[\,f(\wt\bx^{(k)})\,]-\tfrac12 \,\theta_k\,(t+\tau_2)\,\expt\,[\, \|\nabla f(\wt\bx^{(k)})\|^2 \,].
\]
Substituting \cref{eq:PLineq_expec} into the right-hand side and expanding the recursion $k$ times, we obtain the desired result with $0<\tau_2 < 2\,t\,\eps$. Furthermore, on the basis of definition of $\srh$ that $0<\varepsilon<1$ and the property $0<\theta_k<2$, it is easy to check that $0<(t+\tau_2)\,\mu\,\theta_j<1$, which concludes the proof.
\end{proof}
Similarly to $\csr$ (cf.~\cref{prop:convergencerate_c2_csr}), under the condition $\theta_j>0$, the employment of $\srh$ ensures linear convergence only up to a neighborhood of the optimal point.
In general, \cref{prop:convergencerate_c2_srh} shows that a larger $\eps$ in $\srh$, i.e., a larger rounding bias (cf.~\cref{sec:rounding}), allows for larger values of $\beta_j$ and of $\tau_2$. Comparing \cref{eq:conv-expcsrplstageII} and \cref{eq:fsrh_caseII}, a larger value of $\tau_2$ leads to a stricter bound on the convergence rate. Under the same condition as that for $\csr$, the convergence achieved by $\srh$ may be faster than that achieved by $\csr$. These claims can be validated using numerical experiments; see, e.g., Figures \ref{fig:differentscalea}, \ref{fig:f2fix_a} and \ref{fig:logisre}. In the next section, we analyze the convergence behavior of GD for Case III. 

\subsubsection{Case III}\label{sec:convergence analysisStageIII} 
In this case, we consider the updating rule that satisfies \cref{eq:d_c3}. The updating vector consists of entries $\wt{d}_i^{(k)}$ that satisfy either \cref{eq:d_c1} or \cref{eq:d_c2}. Our analysis demonstrates that the convergence rate is bounded between the rates observed in Case I and Case II. For the next result, we need to introduce the following quantity: 
\begin{align}\label{eq:alpha_k}
 \alpha_k:= & \frac{\sum_{\calc_2} t\,(\theta_k-1)\,\expt\,[\,\nabla f(\wt\bx^{(k)})_i^2 \,]}{\expt\,[\, \|\nabla f(\wt\bx^{(k)})\|^2 \,]},
\end{align}
where $\calc_2$ (cf.~\cref{eq:cond_case3}) is the set defined in the condition of Case III.
It is easy to see that $|\alpha_k| \le t\, |\theta_k-1|$. When $0<\theta_k< 2$ (cf.~\cref{eq:case2_theta}), then $-t<\alpha_k<t$.
Informally, $\alpha_k$ tends to be positive when the gradient is large (at the beginning of the process), and negative when the gradient is close to 0. 

\begin{proposition}\label{prop:convergencerate_c3_csr}
 Under \cref{assum:universal assumption} and the condition of Case III \cref{eq:cond_case3}, suppose that both $\bsigma_1$ and $\bsigma_2$ in \cref{eq:gd_fp} are obtained using $\csr$ with fixed stepsize $t$ such that $t\le \frac{1}{4\,L}$. If $\theta_j>0$ in \cref{eq:case2_theta} for all $j$, then with $\alpha_j$ as in \cref{eq:alpha_k}, we have
\begin{equation}
\expt\,[\,f(\wt\bx^{(k)})-f^*\,]\lesssim\,\prod_{j=0}^{k-1}(1-\mu\,(t+\alpha_j))\,(f(\bx^{(0)})-f^*),
\label{eq:conv-expcsrplstageIII}
\end{equation}
 and $0<\mu(t+\alpha_j)<1$, for $j=0,\dots, k-1$.
\end{proposition}
\begin{proof}
Because of \cref{eq:d_c3}, we have
\begin{small}
 \begin{align}\label{eq:f_c3}
\expt\,[\,f(\wt\bx&^{(k+1)})\,]\le \expt\,[\,f(\wt\bx^{(k)})\,]-\expt\,[\,\nabla f(\wt\bx^{(k)})^T\wt{\bd}^{(k)}\,]+\tfrac12 \,L\,\expt\,[\, \| \wt{\bd}^{(k)}\|^2 \,]\\
=\,&\expt\,[\,f(\wt\bx^{(k)})\,]-\!\!\sum_{i\in\calc_1\cup\,\calc_2} \expt\,[\,t\ \nabla f(\wt\bx^{(k)})_i^2+t\ \nabla f(\wt\bx^{(k)})_i\,\sigma^{(k)}_{1,i}+\nabla f(\wt\bx^{(k)})_i\,\sigma^{(k)}_{2,i}-\tfrac12 \,L\,(\wt{d}_i^{(k)})^2 \,].\nonumber
\end{align}
\end{small}\noindent
When $\csr$ is used to evaluate $\sigma^{(k)}_{1,i}$ and $\sigma^{(k)}_{2,i}$, on the basis of \cref{eq:expectedfxk1} and \cref{eq:c2csr_ith}, we have 
\begin{align}\label{eq:f_c3_csr}
 \expt\,[\,f(\wt\bx^{(k+1)})\,]\lesssim\,&\expt\,[\,f(\wt\bx^{(k)})\,]-\sum_{i\in\calc_1}\tfrac12 \,t\,\expt\,[\,\nabla f(\wt\bx^{(k)})_i^2 \,]-\sum_{i\in\calc_2}\tfrac12 \,\theta_k\,t\,\expt\,[\,\nabla f(\wt\bx^{(k)})_i^2 \,]\nonumber\\
 =\,&\expt\,[\,f(\wt\bx^{(k)})\,]-\sum_{i=1}^n\tfrac12 \,t\,\expt\,[\,\nabla f(\wt\bx^{(k)})_i^2 \,]-\sum_{i\in\calc_2}(\tfrac12 \,\theta_k-\tfrac{1}{2})\,t\,\expt\,[\,\nabla f(\wt\bx^{(k)})_i^2 \,]\nonumber\\
 =\,&\expt\,[\,f(\wt\bx^{(k)})\,]-\tfrac12 \,t\,\expt\,[\, \|\nabla f(\wt\bx^{(k)})\|^2 \,]-\sum_{i\in\calc_2}\tfrac{1}{2}t(\theta_k-1)\,\expt\,[\,\nabla f(\wt\bx^{(k)})_i^2 \,].
\end{align}
Substituting \cref{eq:alpha_k} into \cref{eq:f_c3_csr} and expanding the recursion $k$ times, we obtain \cref{eq:conv-expcsrplstageIII}. The properties $t\le \frac{1}{4\,L}\le \tfrac{1}{8 \,\mu}$ and $0<\theta_j<2$ imply that $1-2 \,\mu\,t>0$ and $-t<\alpha_j<t$, which indicates that $1-\mu\,(t+\alpha_j)>1-2 \,\mu\,t>0$.
\end{proof}
Equation \cref{eq:f_c3_csr} can be seen as the combination of \cref{eq:expectfxk_srh} with $\rho_k=0$ and \cref{eq:c2csr_ith}. From \cref{eq:f_c3_csr}, it can be seen that when $0<\theta_k<2$, we obtain monotonicity on the objective function. When $\theta_k=1$, we may achieve similar convergence to Case I, i.e., 
\[
 \expt\,[\,f(\wt\bx^{(k+1)})\,]\lesssim \expt\,[\,f(\wt\bx^{(k)})\,]-\tfrac12 \,t\,\expt\,[\, \|\nabla f(\wt\bx^{(k)})\|^2 \,],
\]
which leads to the same convergence bound as in \cref{coro:convergencerate_c1_csr}, i.e., \[
\expt\,[\,f(\wt\bx^{(k)})-f^*\,]\lesssim(1-t\,\mu)^k\,(f(\bx^{(0)})-f^*).
\]
However, when $0<\theta_k<1$, we may obtain a negative $\alpha_k$ in \cref{eq:conv-expcsrplstageIII}, which implies a larger bound on the convergence rate than Case I. In particular, the convergence rate of \cref{eq:conv-expcsrplstageIII} depends on the number of $\wt{d}_i^{(k)}\in\calc_2$. When $0<\theta_k<1$ a larger number of $\wt{d}_i^{(k)}\in\calc_2$ causes a smaller value of $\alpha_k$ (cf.~\cref{eq:alpha_k}). On the contrary, when $\theta_k\ge1$ a larger number of $\wt{d}_i^{(k)}\in\calc_2$ results in a larger value of $\alpha_k$. When many $\wt{d}_i^{(k)}$ belong to $\calc_2$ and the components of the gradient vectors are close to zero ($\alpha_k$ is negative), we may achieve a slower convergence.

Next, we show that with $\srh$ we may obtain a stricter bound on the convergence rate than that obtained by $\csr$ (cf.~\cref{eq:conv-expcsrplstageIII}). The proof is available in Appendix~\ref{appendixB}.
\begin{proposition}\label{prop:convergencerate_c3_srh}
 Under \cref{assum:universal assumption} and the condition of Case III \cref{eq:cond_case3}, suppose that $\bsigma_1$ and $\bsigma_2$ in \cref{eq:gd_fp} are obtained using $\csr$ and $\srh$, respectively, with fixed stepsize $t$ such that $t\le \frac{1}{4\,L}$. If $\theta_j>0$ in \cref{eq:case2_theta} for all $j$, then with $\tau_2$ as in \cref{eq:tau_2}, we have
\begin{equation}
\expt\,[\,f(\wt\bx^{(k)})-f^*\,]\lesssim\,\prod_{j=0}^{k-1}(1-\mu\,(t+\alpha_j+\theta_j\,\tau_2))\,(f(\bx^{(0)})-f^*),
\label{eq:conv-expsrhplstageIII}
\end{equation}
and $\tau_2 \in (0,2\,t\,\eps)$ with the values $\alpha_j$ defined in \cref{eq:alpha_k}.
\end{proposition}
Again, with $\srh$, we obtain a stricter bound on the convergence rate than that obtained by $\csr$ under Case III (comparing \cref{eq:conv-expcsrplstageIII} and \cref{eq:conv-expsrhplstageIII}). This comparison shows that a faster convergence may be obtained when using $\srh$ in place of $\csr$ for $\bsigma_2$. Equation \cref{eq:conv-expsrhplstageIII} shows that the convergence bound obtained by $\srh$ depends on the values of $\eps$ and $t$. Therefore, a larger $\eps$ in $\srh$ or a larger $t$ can result in faster convergence.

In the next section, we compare our convergence analysis of GD using fixed-point arithmetic to that using floating-point arithmetic. 

\section{Comparison with floating-point arithmetic}\label{sec:compar}
In this section, we demonstrate the different behaviors of our rounding method $\srh$ in fixed-point and floating-point arithmetic. When using $\srh$ to implement GD in fixed-point arithmetic, the rounding bias on each entry of iterate is in the same scale ($u$). However, in floating-point arithmetic, owing to the fact that numbers are not uniformly distributed, the rounding bias on each entry of the iterate also exhibits the same behavior. In particular, when implementing GD with limited precision in floating-point arithmetic, using $\srh$, the updating stepsize is adaptive to the entries of each iterate. 

Now let us dive into the difference between fixed-point and floating-point computation for different rounding methods. We start by comparing their expressions of the updating vector $\wt{\bd}^{(k)}$ of GD in fixed-point and floating-point number formats. In accordance with the argument in \cite[Sec.~4.2]{xia2022float} that the small numbers can be represented by the subnormal numbers in floating-point number formats, therefore stagnation of GD happens when evaluating the subtraction in \cref{eq:gd}. However, the stagnation of GD happens when evaluating the multiplication of $t$ and gradients in fixed-point number formats. To have a clear observation, we assume that stagnation starts from the $k$th iteration step and all the computations before stagnation are exact (errors are negligible) for both fixed-point and floating-point arithmetic, i.e., $\mathrm{fi}(\nabla f(\bx^{(k)})\,)=\nabla f(\bx^{(k)})$ and $\mathrm{fl}(t\,\mathrm{fl}(\nabla f(\bx^{(k)})\,)\,)=t\ \nabla f(\bx^{(k)})$ for fixed-point and floating-point arithmetic, where $\mathrm{fl}(\cdot)$ denotes a general rounding operator that converts a real number $x\in \R$ to the floating-point number representation $\wh{x}$. 

When GD stagnates with $\rn$ for the $i$th coordinate of $\bx^{(k)}$ and $\wt{d}_i^{(k)}\ne 0$, we have a fixed magnitude of $\wt{d}_i^{(k)}$ using stochastic rounding for fixed-point number representation, given by
\[
 \wt{d}_i^{(k)}=u\ \sign(\nabla f(\bx^{(k)})_i).
\]
 Under the same assumption, we have an adaptive magnitude of $\wt{d}_i^{(k)}$ using stochastic rounding methods for floating-point number representation. On the grounds of the model of floating-point operation \cite[(2.4)]{higham2002accuracy}, we have the following value for $\wt{d}_i^{(k)}$:
 \begin{small}
 \begin{align}\label{eq:dfl}
 \wt{d}_i^{(k)}=\begin{cases}\sign(\nabla f(\bx^{(k)})_i) \, | x_i^{(k)}|\,\delta_i^{(k)}+t\ \nabla f(\bx^{(k)})_i\,(1-\delta_i^{(k)}), &\text{if $\sign(x_i^{(k)}\,\nabla f(\bx^{(k)})_i)=\ph{-}1$},\\
\sign(\nabla f(\bx^{(k)})_i) \, | x_i^{(k)}|\,\delta_i^{(k)}+t\ \nabla f(\bx^{(k)})_i\,(1+\delta_i^{(k)}),&\text{if $\sign(x_i^{(k)}\,\nabla f(\bx^{(k)})_i)=-1$},
\end{cases}
\end{align}
 \end{small}\noindent
where $0<\delta_i^{(k)}<2 \,u$ denotes the relative error caused by evaluating $x_i^{(k)}-t\ \nabla f(\bx^{(k)})_i$ using stochastic rounding methods. For different stochastic rounding methods, $\wt{d}_i^{(k)}$ may be either zero or one of the expressions in \cref{eq:dfl} with different probability distributions. Therefore, GD behaves differently with different rounding methods and number representations. 

Now let us study the difference between different rounding methods in each number representation. For fixed-point arithmetic, we obtain the following updating vectors of GD when using $\csr$ and $\srh$.
\begin{proposition}\label{prop:dfixed} Assume the use of fixed-point arithmetic to implement GD, and that $\rn(\wt{d}_i^{(k)})=0$ for the $i$th component in $\wt{{\bf d}}^{(k)}$, i.e., $|\,t\ \nabla f(\bx^{(k)})_i\, |<\tfrac12 \,u$. If $\mathrm{fi}(\nabla f(\bx^{(k)})\,)=\nabla f(\bx^{(k)})$, then for $\csr$ we have 
 \begin{align}\label{eq:dficsr}
 \expt\,[\,\wt{d}_i^{(k)}\ \big|\ t\ \nabla f(\wt\bx^{(k)})_i]= t\ \nabla f(\wt\bx^{(k)})_i,
\end{align}
and for $\srh$ we have 
\begin{align*}
 \expt\,[\,\wt{d}_i^{(k)}\ \big| \ t\ \nabla f(\bx^{(k)})_i]=\begin{cases}
 u\,\sign(\nabla f(\bx^{(k)})_i), &\text{if $\eps\ge1-\frac{|\,t\ \nabla f(\bx^{(k)})_i\, |}{u}$,}
 \\t\ \nabla f(\bx^{(k)})_i+\eps\,u\,\sign(\nabla f(\bx^{(k)})_i), &\text{otherwise}
 .\end{cases}
\end{align*}
\end{proposition}
The proof is available in Appendix~\ref{appendixB}. From \cref{prop:dfixed}, we observe that $\csr$ and $\srh$ result in an average updating magnitude of GD that is similar to and slightly larger than that in exact computation, respectively. In fixed-point arithmetic, the updating length of GD is independent of its iterate $\bx^{(k)}$  for both $\csr$ and $\srh$. Additionally, the rounding bias introduced by $\srh$ only depends on $u$ and has the same sign vector as its gradient.

In floating-point arithmetic, signed-$\srh$ and $\srh$ generate the same magnitude of rounding bias but the rounding bias may have a different sign. When GD stagnates with $\rn$, we observe the following properties of the updating magnitudes of GD when using $\csr$ and signed-$\srh$. 
\begin{proposition}\label{prop:dfloat}
Assume the use of floating-point arithmetic to implement GD, and that $\rn(x_i^{(k)}-\wt{d}_i^{(k)})=x_i^{(k)}$, i.e., $\big|\frac{t\ \nabla f(\bx^{(k)})_i}{x_i^{(k)}}\big|<u$. If $\mathrm{fl}(t\,\mathrm{fl}(\nabla f(\bx^{(k)})\,)\,)=t\ \nabla f(\bx^{(k)})$, then for $\csr$ we have
 \begin{align}\label{eq:dflcsr}
 \expt\,[\,\wt{d}_i^{(k)}\ \big|\ x_i^{(k)}-t\ \nabla f(\bx^{(k)})_i\,]=t\ \nabla f(\wt\bx^{(k)})_i,
\end{align}
and for signed-$\srh$, when $0<p_{\eps s}<1$, we have
\begin{small}
 \begin{align}\label{eq:dflsrh}
 &\expt\,[\,\wt{d}_i^{(k)}\ \big|\ x_i^{(k)}-t\ \nabla f(\bx^{(k)})_i\,]\\
 &=\!\begin{cases}(1+\eps+\eps\,\delta_i^{(k)})\,t\ \nabla f(\bx^{(k)})_i+\sign(\nabla f(\bx^{(k)})_i)\, | x_i^{(k)}|\,\eps\,\delta_i^{(k)},&\!\!\!\text{if $\sign(x_i^{(k)}\,\nabla f(\bx^{(k)})_i)=-1,$}\\
 (1+\eps-\eps\,\delta_i^{(k)})\,t\ \nabla f(\bx^{(k)})_i+\sign(\nabla f(\bx^{(k)})_i)\, | x_i^{(k)}|\,\eps\,\delta_i^{(k)},&\!\!\!\text{if $\sign(x_i^{(k)}\,\nabla f(\bx^{(k)})_i)=\ph{-}1,$}\nonumber
 \end{cases}
\end{align}
\end{small}\noindent
 where $0<\delta_i^{(k)}<2 \,u$.
\end{proposition}
The proof is available in Appendix~\ref{appendixB}. Since $t\ \nabla f(\bx^{(k)})_i$ is relatively small compared to $x_i^{(k)}$ in \cref{prop:dfloat},
 the magnitude of $\wt{d}_i^{(k)}$ is significantly affected by $|x_i^{(k)}|$. Consequently, $\expt\,[\,\wt{d}_i^{(k)}\ \big| \ x_i^{(k)}-t\ \nabla f(\bx^{(k)})_i\,]$ automatically adapts to the scale of $x^{(k)}_i$ for each entry of $\bx^{(k)}$ by using signed-$\srh$ in floating-point arithmetic, whereas this adaptivity does not hold for $\csr$. Comparing \cref{prop:dfixed,prop:dfloat}, the magnitude of the updating vector of GD is independent of $u$ or the current iterate when using $\csr$ in both fixed-point and floating-point arithmetic. In particular, when GD stagnates with $\rn$, $\csr$ may lead to a similar convergence behavior to exact computation on average. This means using a uniform stepsize $t$ for each entry of $\bx^{(k)}$. However, when using signed-$\srh$, the convergence speed mainly depends on the magnitude of the iterate of GD in floating-point arithmetic and is largely improved by the adaptive rounding bias along each coordinate of $\bx^{(k)}$, while the updating magnitude of GD is almost consistent along each coordinate of $\bx^{(k)}$ in fixed-point number representation.

In general, $\csr$ performs similarly in both floating-point and fixed-point arithmetic. $\srh$ and signed-$\srh$ have a better effect than $\csr$ on both floating-point and fixed-point arithmetic. Thanks to the non-uniform distributed number representation in floating-point arithmetic, the rounding bias caused by $\srh$ and signed-$\srh$ is also non-uniform, which results in a faster convergence than that of GD in exact arithmetic. For fixed-point arithmetic, the rounding bias caused by $\srh$ is also uniform as its number distribution, which may have a less obvious advantage in accelerating GD than in floating-point arithmetic. Despite the fact that floating-point number representation has a wider number representation owing to the non-uniform distributed number representation, many practical low-cost embedded microprocessors and microcontrollers (e.g., FPGA designs) are limited to finite-precision signal processing using fixed-point arithmetic due to the cost and complexity of floating-point hardware. In both number representations, the implementation of GD using $\srh$ shows better performance compared to $\csr$. 

\section{Simulation studies}\label{sec:simulation}
In this section, we validate our theoretical results by applying GD with limited-precision number formats on several case studies: a quadratic problem, the two-dimensional Rosenbrock's function, Himmelblau's function, a binary logistic regression model (BLR), and a four-layer fully connected NN. We numerically show that Rosenbrock's function satisfies the PL condition in a certain domain and achieves linear convergence of GD. Himmelblau's function is used to demonstrate that by using stochastic rounding, GD can converge exactly to the optimal point instead of a nearby neighborhood. BLR is proven to satisfy the PL condition by \cite{karimi2016linear}, and we validate our theoretical results by training a BLR using different rounding methods and learning stepsizes. Finally, we test the convergence of GD on training a four-layer NN and show that, despite this example not satisfying the PL condition, the numerical results are similar to those observed in the other experiments. 

As discussed in \cref{sec:magnit}, the stagnation of GD is normally caused by $\bsigma_2$ in \cref{eq:gd_fp}. Therefore, for each simulation study, we employ two different number formats: one for evaluating $\bsigma_2$ and the other for the working precision, where all the other operations are implemented using this working precision.
Note that in all case studies, apart from the four-layer NN, we use low-precision fixed-point number formats, with their implementation relying on the Matlab \textsf{fi} toolbox.
Due to the slow computation speed of \textsf{fi} toolbox, in the four-layer NN case, we only evaluate $\bsigma_2$ with limited-precision fixed-point numbers and use single-precision computations for the other operations.
Note that selecting the best number format for a procedure that relies on fixed-point arithmetic operations is problem-dependent and, to the authors' knowledge, there is no systemic way to select a priori the best number format for a given problem. In the prepossessing procedure, we test various values for QI and QF in different case studies. 
For convenience, the values of the QI are chosen to prevent overflow problems for the case studies like Quadratic, Rosenbrock, and Himmelblau functions. For more challenging examples like BLR and Four-layer NN, we choose the number of integers in such a way that it provides a similar result to that of single-precision computation.
Nevertheless, Overflow is generally less critical than underflow when considering the vanishing gradients problem. Here, we focus on showing that the vanishing gradient problems can be eliminated by the use of stochastic rounding methods. The selected values of QI are reported in \cref{tab:qformat}. 

The considered values for QF are taken as small as possible in such a way that it allows us to easily observe the influence of rounding errors and validate our theoretical results. We start by setting $u$ to $2^{-12}$, and then we increase or decrease $u$ until we achieve the largest value of $u$ that guarantees the convergence. The considered number formats are summarized in \cref{tab:qformat}. Additionally, the starting point $\bx^{(0)}$ for all the numerical studies in this section is chosen as either: (1) integers that can be exactly represented in the designed fixed-point arithmetic, e.g., Quadratic, Rosenbrock, and Himmelblau functions, or (2) a rounded version of the original starting point obtained using the rounding to the nearest method, e.g., BLR and the four-layer fully connected NN where the MNIST data are normalized. In the latter case, the rounded $\bx^{(0)}$ serves as the starting point for GD, and different rounding methods are then applied to study their impact on the convergence of the algorithm.
\begin{table}[h!]
\caption{{\footnotesize Number formats of fixed-point representation for different case studies. The third column indicates how the multiplication $t \, \nabla f(\bx_k)$ is carried out.}}\label{tab:qformat}
\begin{tabular}{lll}
\toprule 
 Case study & Working precision & Multiplication \\ \midrule \rule{0pt}{2.3ex}
Quadratic & Q$26.6 \ (u=2^{-6})$ & Q$26.6 \ (u=2^{-6})$ \\
Rosenbrock & Q$8.10\  (u=2^{-10})$ & Q$12.6 \ (u=2^{-6})$ \\
Himmelblau & Q$8.8\  (u=2^{-8})$& Q$8.8\  (u=2^{-8})$\\
BLR & Q$15.8 \ (u=2^{-8})$ & Q$15.8 \ (u=2^{-8})$ and Q$15.6\  (u=2^{-6})$ \\
Four-layer NN & Binary$32 \ (u=2^{-23})$& Q$8.8 \ (u=2^{-8})$ 
\\ \botrule
\end{tabular}
\end{table}

For the cases of Rosenbrock's function and NN, we also compare the performances of low-precision fixed-point computation with those of low-precision floating-point computation. For all the simulation studies, we compare the results obtained by $\csr$ and $\srh$ to those obtained by $\rn$. The baselines are obtained by single-precision (\textsf{Binary32}) floating-point computation with $\rn$. We remark that there is a huge difference between the magnitude of the machine precision of single precision ($2^{-23}$) and those of other fixed-point number formats, which stay above $2^{-10}$. Therefore, we look at the comparison with the baselines as a surrogate of a comparison with exact arithmetic. 

\subsection{Quadratic problems}
In \cref{sec:compar}, we have demonstrated that $\csr$ performs very similarly in fixed-point and floating-point arithmetic, yet signed-$\srh$ and $\srh$ have a positive impact on floating-point arithmetic and a modestly positive impact on fixed-point arithmetic. To better understand this phenomenon, let us compare the convergence behavior of GD using different rounding methods with fixed-point and floating-point arithmetic. For instance, when optimizing a least-squares problem of the form $F_1(\bx)=\tfrac{1}{2} \, (\bx-\bx^*)^TA\,(\bx-\bx^*)$, where $A=\mathrm{diag}(100, 10, 10^{-3}, 10^{-3}, 10^{-3})$ and $\bx^*=[10^{-1}, 1, 10, 100, 1000]^T$, so that the entries of $\bx^*$ are of different scales. When using fixed-point arithmetic, the rounding bias introduced by $\srh$ can not take care of the scale of each entry of $\bx^*$, but it can be done when using $\srh$ in floating-point arithmetic. To achieve a descent updating direction, we employ signed-$\srh$ \cite[Def.~2.3]{xia2022float} in floating-point arithmetic to customize the sign of the rounding bias. 

\begin{figure}[htp]
\centering
\subfloat[Fixed-point arithmetic]{\label{fig:differentscalea}\includegraphics[width=0.37\textwidth]{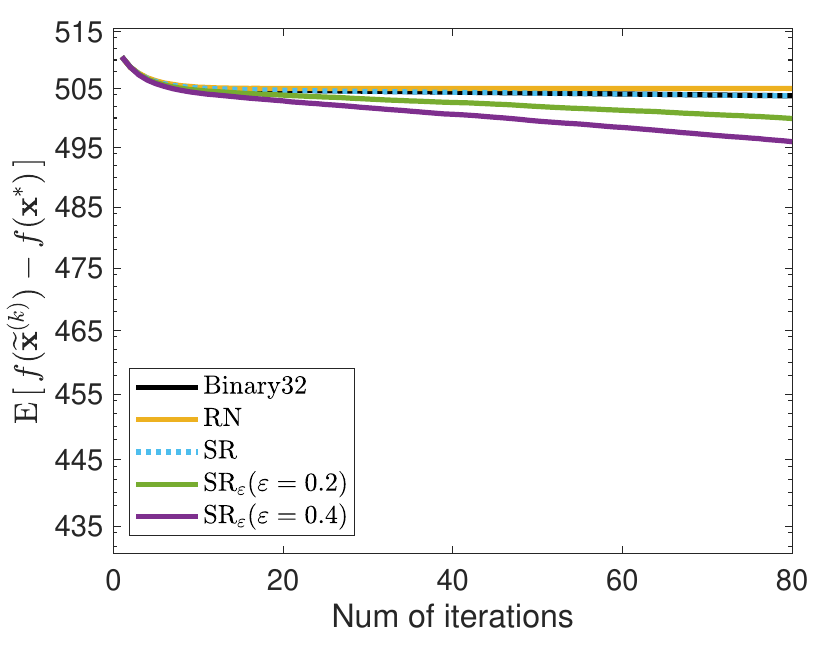}}\qquad
\subfloat[Floating-point arithmetic]{\label{fig:differentscaleb}\includegraphics[width=0.37\textwidth]{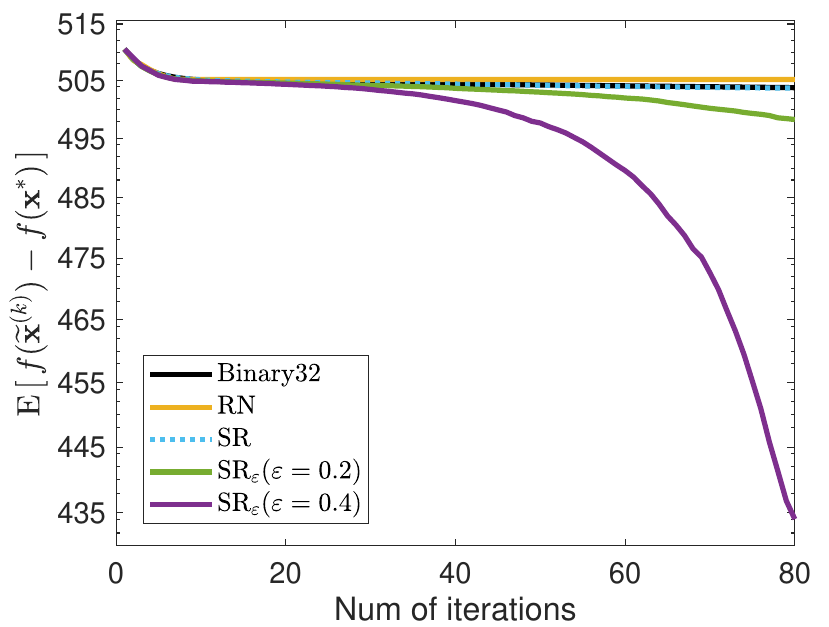}}
\caption{A quadratic problem: comparison of the objective values using different rounding scheme with fixed-point arithmetic (Q$26.6$) (a) and floating-point arithmetic (8 bits with 3 significant bits) (b); settings: stepsize $t=2^{-6}$ and $\bx^*=[10^{-1}, 1, 10, 100, 1000]^T \in \R^5$. }\label{fig:differentscale}
\end{figure}
\cref{fig:differentscale} shows the comparison of objective values when optimizing $F_1$ using GD with different rounding schemes using fixed-point arithmetic (\cref{fig:differentscalea}) and floating-point arithmetic (\cref{fig:differentscaleb}). Due to the overlapping between the results of $\csr$ and \textsf{Binary32}, we utilize a dashed line for $\csr$ (only) in \cref{fig:differentscale} for readability. For both figures, the baseline is obtained by optimizing $F_1$ using single-precision (32-bit) computation and $\rn$. In \cref{fig:differentscalea}, both $\bsigma_1$ and $\bsigma_2$ are obtained using the same rounding scheme. For instance, the yellow line in \cref{fig:differentscalea} shows that both $\bsigma_1$ and $\bsigma_2$ are obtained using $\rn$ with fixed-point number representation Q$26.6$. In \cref{fig:differentscaleb}, we apply 8-bit floating-point computation with 3 significant bits.
Comparing \cref{fig:differentscalea,fig:differentscaleb}, it can be observed that the convergence rates of GD using $\rn$ and $\csr$ are very similar in fixed-point and floating-point number formats. However, $\srh$ has extraordinarily different performance in these two number formats. \cref{fig:differentscalea} suggests that in fixed-point arithmetic, GD with the accumulated rounding bias in the descent direction maintains a linear convergence rate but with a smaller base compared to exact computation. On the other hand, using $\srh$ in floating-point arithmetic appears to yield an almost superlinear convergence rate. 

To better understand how $\theta_k$ \cref{eq:case2_theta} influences the convergence of GD with the employment of $\csr$ (\cref{prop:monoto_csr}) and $\srh$ (\cref{prop:monoto_srh}), let us compare the convergence behavior of GD using different rounding precisions. As an example, consider the optimization of a least-squares problem represented by  $F_4(\bx)=\tfrac{1}{2} \, (\bx-\bx^*)^TA\,(\bx-\bx^*)$, where $A=\mathrm{diag}(1, 1, 10^{-3}, 10^{-3}, 10^{-3})$ and $\bx^*=[1000, 100, 10, 100, 1000]^T$. This particular problem formulation is significant as it captures scenarios where the elements of $\bx^*$ exhibit different scales and the choice of $A$ prevents $L$ from being excessively large, which prevents $\theta_k$ from becoming consistently negative. 
To obtain an approximation of $\theta_k$, we replace the sets $\calw_i^{(k)}$ in \cref{eq:case2_theta} with the gradient entries at the $k$th step, encountered across the various simulations, and we set $L=1$. 

\cref{fig:u4} shows the approximation of $\theta_k$ (\cref{fig:Ou4}) and the value of $\expt\,[\,f(\mathbf{x}^{(k+1)})-f(\mathbf{x}^{(k)})\,]$ (\cref{fig:fu4}) over 40 simulations using fixed-point arithmetic with $u=2^{-4}$. 
Note that \cref{fig:fu4} is plotted on a logarithmic scale, with an additional offset of $10^{-16}$ added to the data. It can be seen that $\theta_k$ is always negative due to the large rounding precision $u$, which in turn causes $\expt\,[\,f(\mathbf{x}^{(k+1)})\,]=\expt\,[\,f(\mathbf{x}^{(k)})\,]$ for some iteration with the utilization of $\csr$, i.e., stagnation of GD. This observation is consistent with the discussion after \cref{eq:case2_theta} and \cref{prop:monoto_csr}, i.e., a negative $\theta_k$ may lead to stagnation or oscillation of GD. Furthermore, the employment of $\srh$ appears to guarantee the strict monotonicity of $\expt\,[\,f(\mathbf{x}^{(k)})\,]$ even if $\theta_k$ is negative. This phenomenon occurs because the definition of $\theta_k$ considers the smallest possible value of the gradient entries. As discussed after \cref{eq:case2_theta}, these small gradients can be compensated by the expected rounding errors intentionally designed in the descent direction when applying $\srh$. By increasing the rounding precision to $u=2^{-6}$, the results presented in \cref{fig:u6} show that $\theta_k$ remains consistently positive throughout the optimization process, with only a few exceptions in the last 10 iteration steps. Additionally, it shows that $\expt\,[\,f(\mathbf{x}^{(k+1)})-f(\mathbf{x}^{(k)})\,]$ is always positive for both $\csr$ and $\srh$, indicating a strict monotonicity of the objective function. Overall, these observations validate Propositions \ref{prop:monoto_csr} and \ref{prop:monoto_srh}, which highlight the strict monotonicity properties associated with the sign of $\theta_k$.

\begin{figure}[th!]
\centering
\subfloat[]{\label{fig:Ou4}\includegraphics[width=0.38\textwidth]{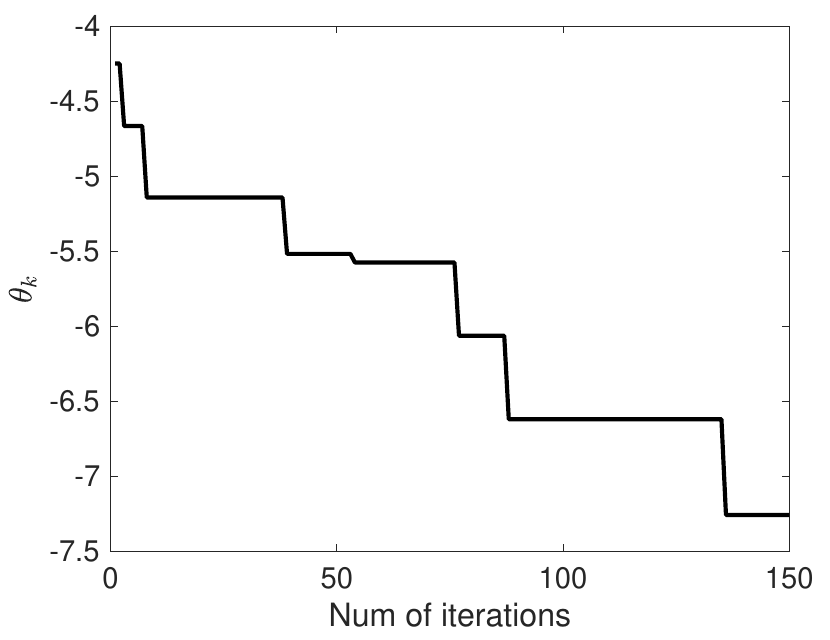}}\quad
\subfloat[ ]{\label{fig:fu4}\includegraphics[width=0.38\textwidth]{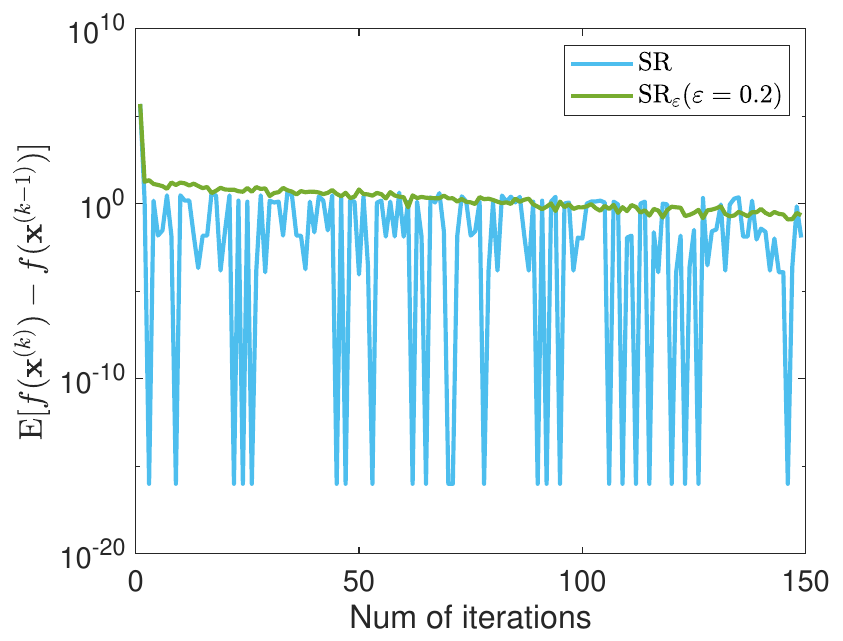}}
\caption{Approximation of $\theta_k$ (a) and $\expt\,[\,f(\mathbf{x}^{(k+1)})-f(\mathbf{x}^{(k)})\,]$ (b) over 40 simulations using different rounding schemes with fixed-point arithmetic (Q$28.4$); settings: stepsize $t=1$ and $\bx^{(0)}=\mathbf{0} \in \R^5$. }
\label{fig:u4}
\end{figure}
 \begin{figure}[th!]
\centering
\subfloat[]{\label{fig:Ou6}\includegraphics[width=0.38\textwidth]{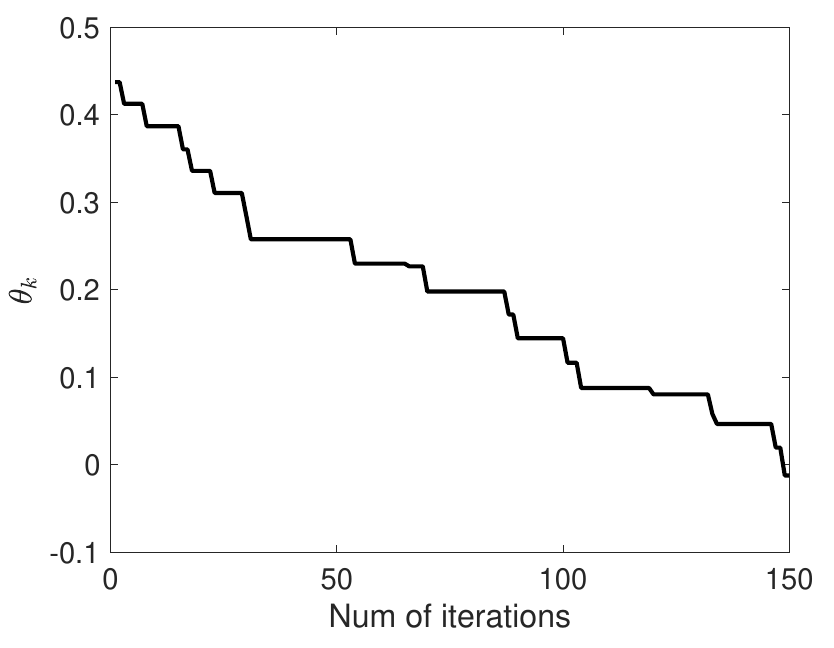}}\quad
\subfloat[ ]{\label{fig:fu6}\includegraphics[width=0.38\textwidth]{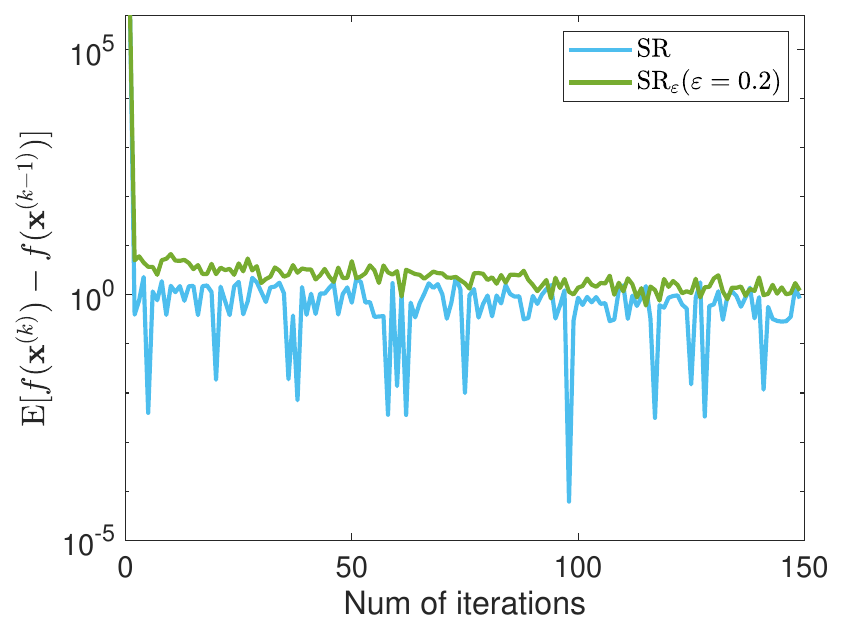}}
\caption{Approximation of $\theta_k$ (a) and $\expt\,[\,f(\mathbf{x}^{(k+1)})-f(\mathbf{x}^{(k)})\,]$ (b) over 40 simulations using different rounding schemes with fixed-point arithmetic (Q$26.6$); settings: stepsize $t=1$ and $\bx^{(0)}=\mathbf{0}\in \R^5$. }
\label{fig:u6}
\end{figure}

\subsection{Rosenbrock's function}\label{sec:rosenb}
Rosenbrock's function is a non-convex function, characterized by a wide almost flat valley around its optimum; it is defined as $F_2(\bx)=(1-x_1)^2+100\,(x_2-x_1^2)^2$. Although it does not satisfy \cref{eq:PLineq} for all $\bx$, it can be numerically checked that Rosenbrock's function satisfies \cref{eq:lpineq} and \cref{eq:PLineq} with $L=2610$ and $\mu=0.2$ for $\bx\in [0,2]^2$.
\begin{figure}[th!]
\centering
\subfloat[Gradient descent trajectories]{\label{fig:f2fix_a}\includegraphics[width=0.35\textwidth, trim=14mm 75mm 15mm 60mm, clip]{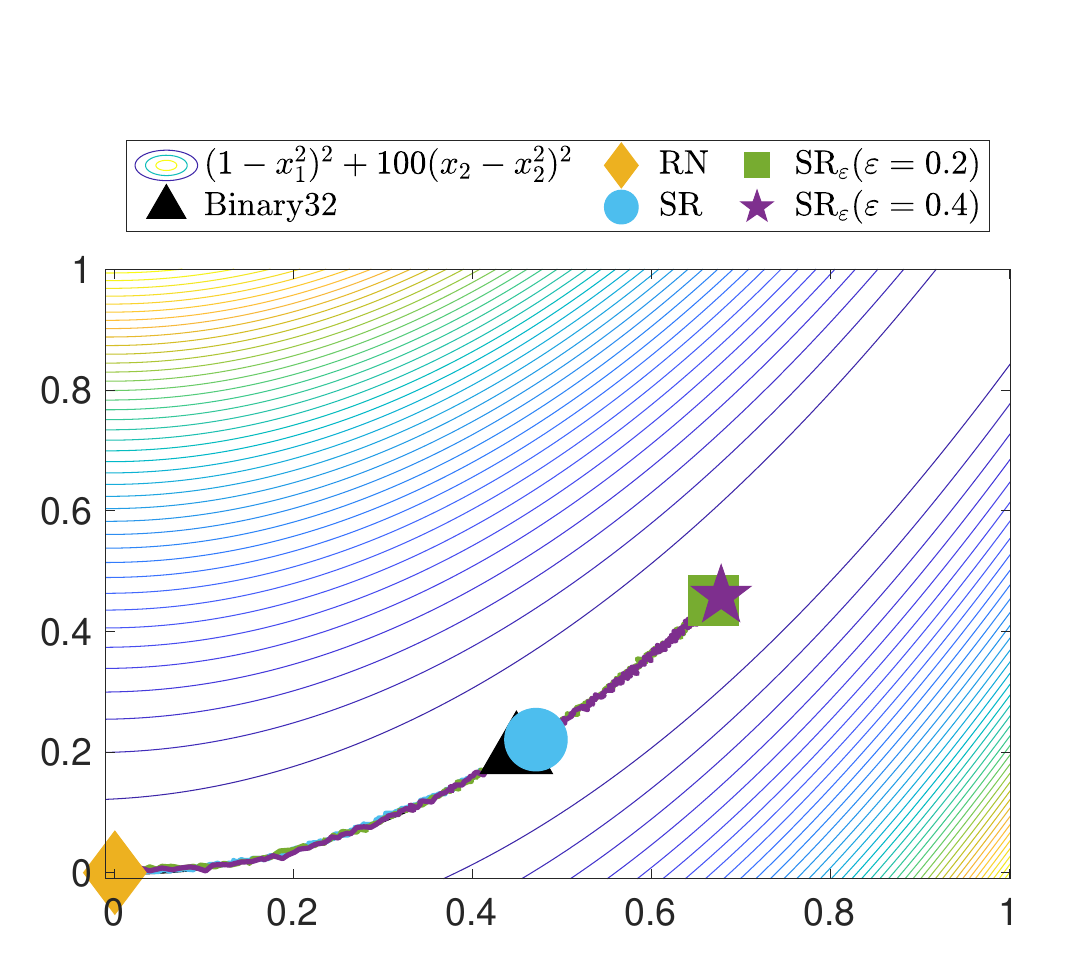}}\quad
\subfloat[Values of objective function]{\label{fig:f2fix_b}\includegraphics[width=0.35\textwidth]{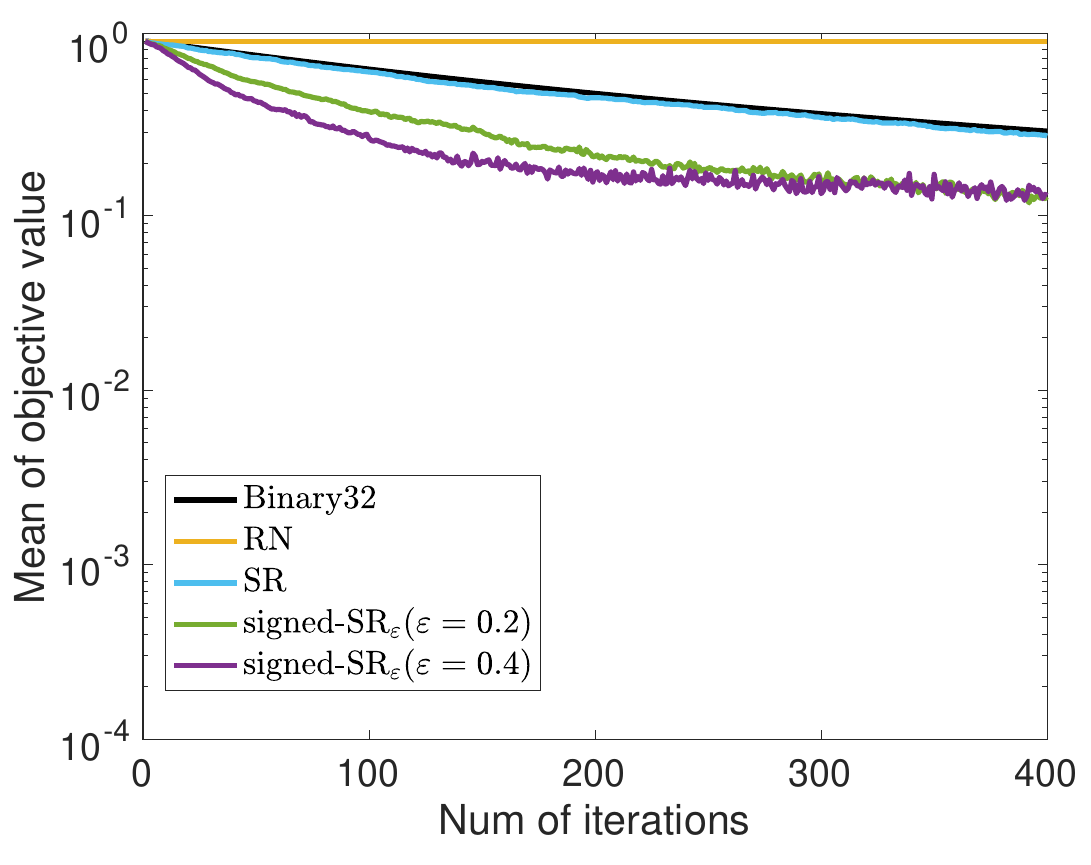}}
\caption{Rosenbrock's function in two dimensions: comparison of the gradient descent trajectories (a) and of the correspondent objective values (b); settings: stepsize $t=2^{-10}$ and Q$12.6$ for evaluating the multiplication of $t$ and gradient and Q$8.10$ for the remaining operations. }\label{fig:rosenbrockfix}
\end{figure}

\cref{fig:rosenbrockfix} shows the trajectories (\cref{fig:f2fix_a}) and the corresponding means of objective function evaluations over 30 simulations (\cref{fig:f2fix_b}) when implementing GD with fixed-point numbers formats and different rounding schemes. The endpoints of GD obtained by different rounding methods are represented by distinct shapes, e.g., a triangle for single-precision computation, a diamond for RN, a circle for SR, a square for $\srh(\varepsilon=0.2)$, and a pentagram for $\srh(\varepsilon=0.4)$. To better observe the influence of rounding errors on the convergence of GD, we apply low precision to compute the multiplication of $t$ and the gradients (Q$12.6$) and higher precision for the remaining operations to maintain the accuracy (Q$8.10$). The starting point is $\bx^{(0)}=[0,0]^T$. It can be seen from \cref{fig:f2fix_a} that GD relying on $\rn$ stagnates quite early because of the loss of gradient information. $\csr$ follows a very similar convergence to the baseline that validates \cref{coro:convergencerate_c1_csr}. When using $\srh$, increasing the value of $\eps$ leads to a faster convergence to the optimum, which matches the conclusion in \cref{them:convergencerate_c1_srh,prop:convergencerate_c2_srh,prop:convergencerate_c3_srh}. Furthermore, from \cref{fig:f2fix_b}, with $\eps=0.4$, the mean of the objective value at the $64$th iteration is similar to that obtained by single-precision at the $400$th iteration, i.e., 0.31. Due to the wide valley in Rosenbrock's function, a small deviation can cause a large oscillation in the objective function. Therefore, the small rounding bias introduced by $\srh$ will result in oscillations in the objective function's values. 

\begin{figure}[thp]
\centering
\subfloat[Gradient descent trajectories]{\label{fig:f2a}\includegraphics[width=0.35\textwidth, trim=14mm 75mm 15mm 60mm, clip]{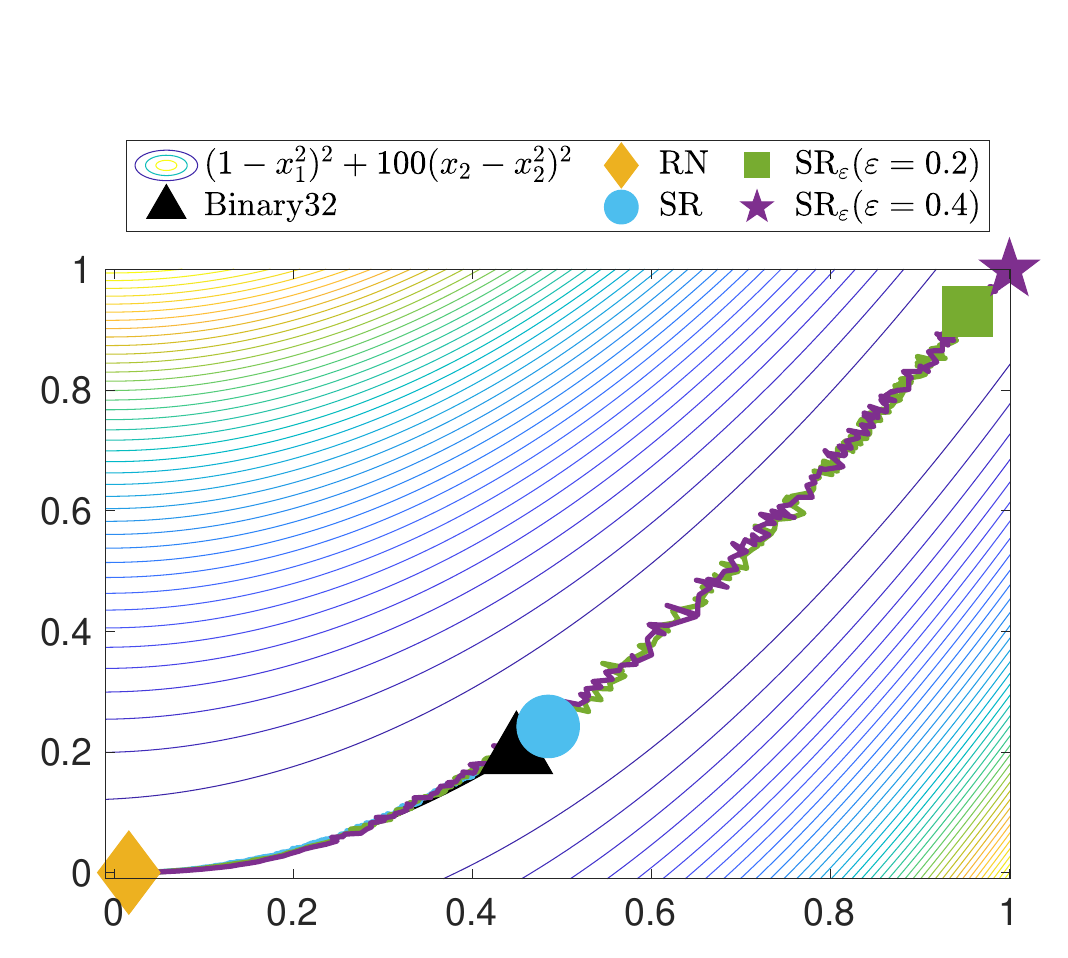}}\quad
\subfloat[Values of objective function]{\label{fig:f2b}\includegraphics[width=0.35\textwidth]{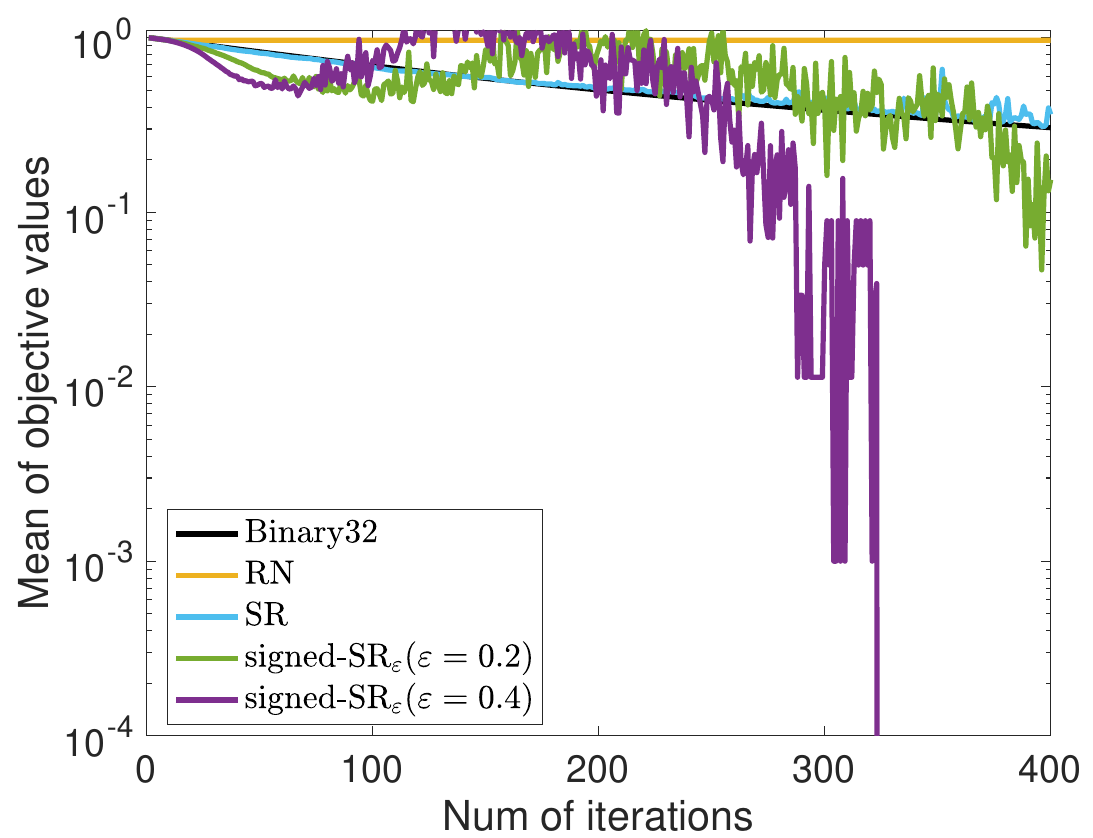}}
\caption{Rosenbrock's function in two dimensions: comparison of the gradient descent trajectories (a) and of the correspondent objective values (b); settings: stepsize $t=2^{-10}$ and \textsf{Binary8} \cite[Sec.~2.1]{xia2022float} with 3 significant bits.}\label{fig:rosenbrockfloat}
\end{figure}
Further, we repeat the same simulation study with an 8-bit floating-point number format with 3 significant digits. \cref{fig:rosenbrockfloat} shows the trajectories and the corresponding mean of objective function evaluations when different rounding methods are employed in floating-point arithmetic. Note that signed-$\srh$ introduces rounding biases with the same magnitudes as those obtained by $\srh$. From \cref{fig:rosenbrockfloat}, it can be observed that $\rn$ causes stagnation of GD, and $\csr$ follows a similar trajectory to that obtained by single-precision computation. Signed-$\srh$ significantly accelerates the convergence of GD due to the expected rounding errors in the descent direction. When $\eps=0.4$, despite the presence of oscillations in the mean of objective function values, GD converges to the optimum within a maximum of 324 iterations for all 30 simulations. 

\subsection{Himmelblau's function}\label{sec:Himmelblau's function}
Now, in line with the discussion in \cref{sec:signd}, we show with an illustrative example that when the optimal point $\bx^{*}$ can be represented exactly in the available number format, we have that $\|\wt\bx^{(k)} -\bx^{*}\|\to 0$. On the other hand, when $\bx^{*}$ cannot be represented exactly, GD converges to a neighborhood around $\bx^{*}$. The size of this neighborhood is determined by the value of $u$. Let us consider the minimization of Himmelblau's function $F_3(\bx)=(x_1^2+x_2-11)^2+(x_1+x_2^2-7)^2$, using single-precision computation and fixed-point number representation with Q$8.8$. We remark that Himmelblau's function does not satisfy the PL condition for all $x$, but this is inconsequential for this test. The function $f$ has four global minimizers $\bx^*_i$, for $i=1,\dots,4$, known in closed form. In particular, we focus on scenarios where GD converges towards two distinct points: $\bx^*_1=[3,\,2]^T$, which is exactly representable, and $\bx^*_4=[3.584428,\,-1.848126]^T$, which is not exactly representable. In \cref{fig:himmelblau_func} we report the averaged objective function values (over 40 runs) for different rounding schemes, i.e., $\rn$, $\csr$, and $\srh$. As anticipated, when the starting point is close to $\bx^*_4=[3.584428,\,-1.848126]^T$, all the rounding methods can only guide GD to a neighborhood around $\bx^*_4$. When GD converges to $\bx^*_1=[3,\,2]^T$, all the stochastic rounding methods retrieve the exact value $\bx^*_1$. However, $\rn$ causes GD to stagnate due to the problem of vanishing gradient. Additionally, in the context of utilizing $\csr$ to achieve convergence to $\bx^*_1$, \cref{fig:h4} shows the number of gradient components satisfying \cref{eq:condu_nostagnation}. It can be seen that with the progression of iteration steps, the number of gradient components satisfying \cref{eq:condu_nostagnation} gradually decreases.
This trend suggests a transition in the updating process, first from Case I to Case III, and finally to Case II, aligning with the discussion after \cref{eq:gd_fp}.
\begin{figure}[htp]
\centering
\subfloat[Contour plot]{\label{fig:h1}\includegraphics[width=0.4\textwidth]{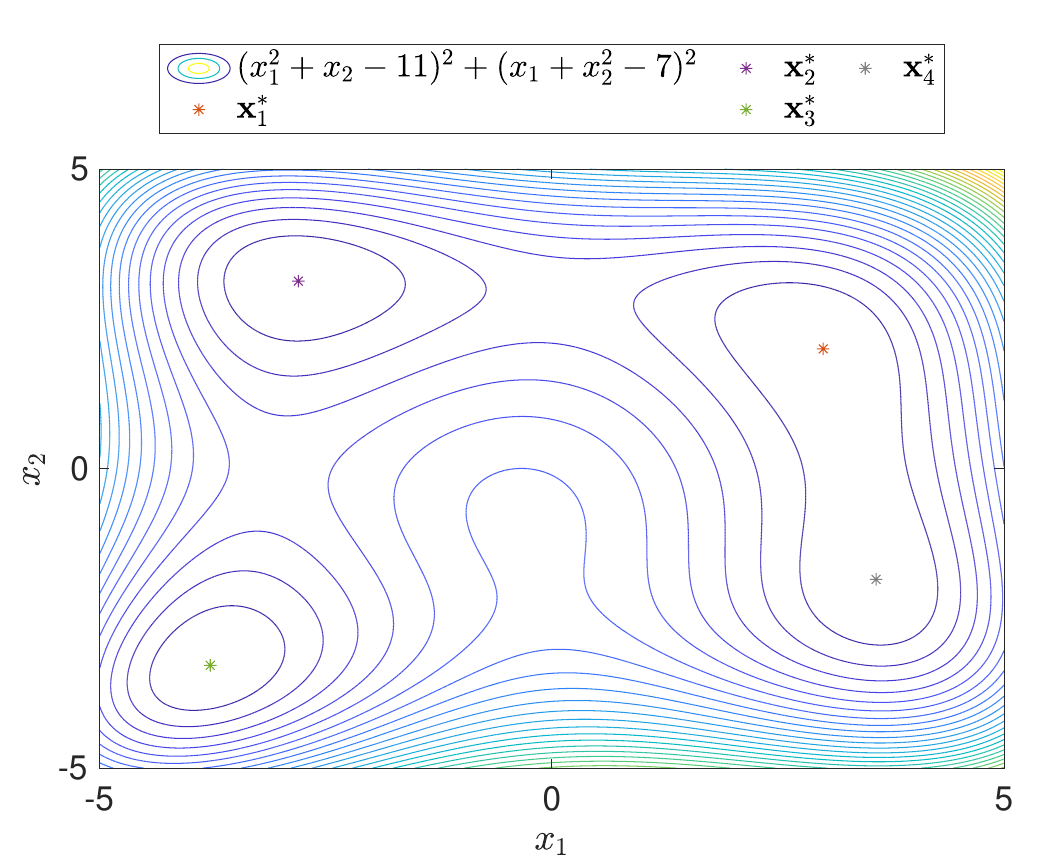}}\
\subfloat[GD converges to $\bx^*_4$]{\label{fig:h3}\includegraphics[width=0.42\textwidth]{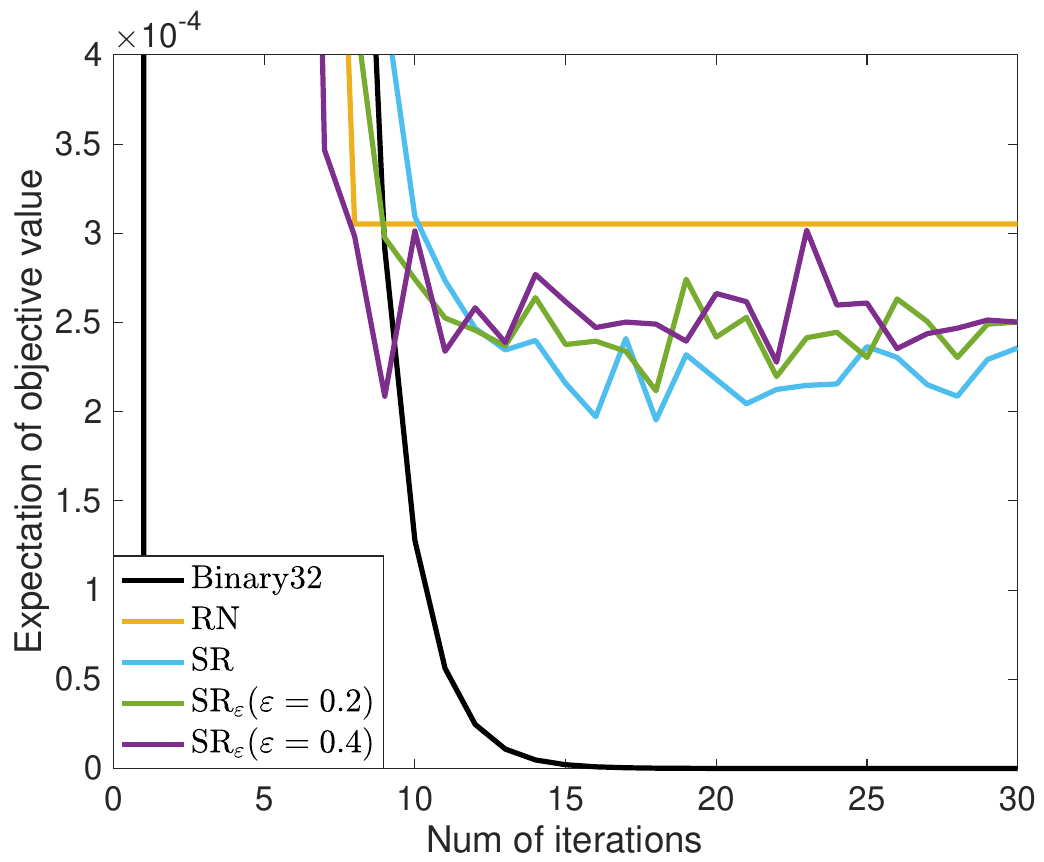}}\\
\subfloat[GD converges to $\bx^*_1$]{\label{fig:h2}\includegraphics[width=0.42\textwidth]{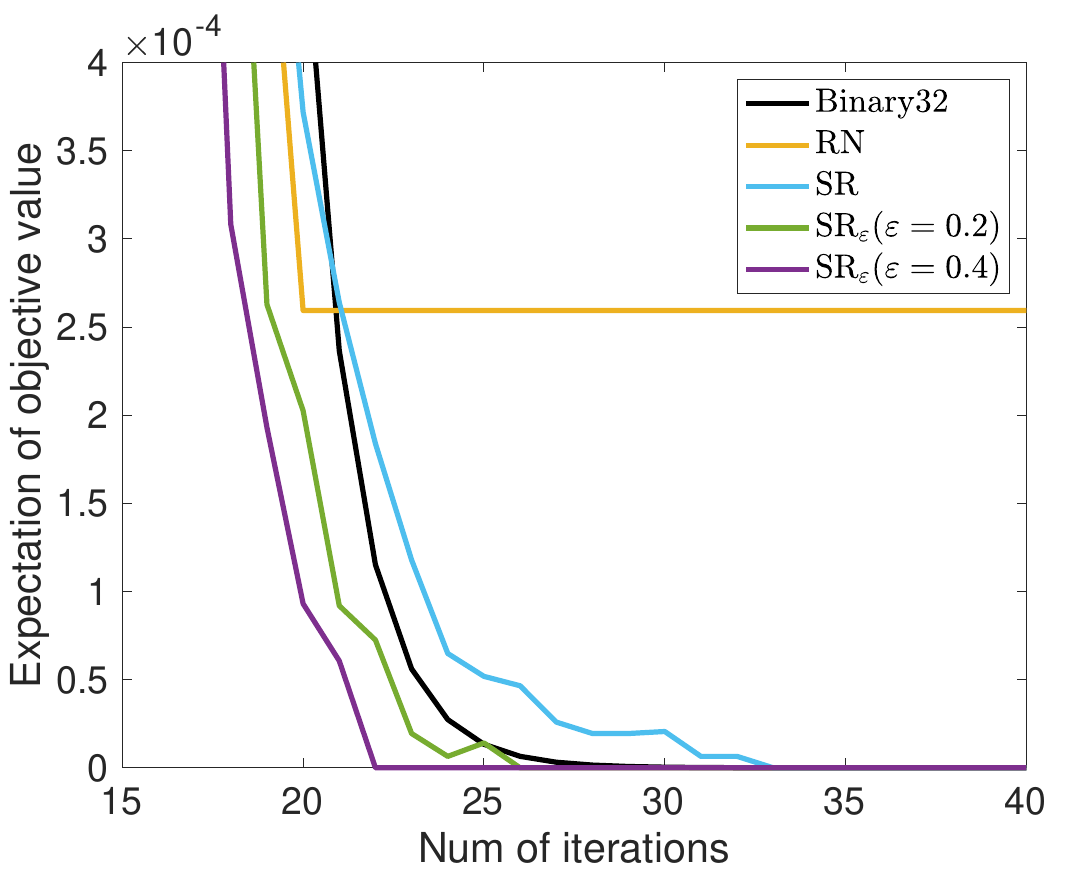}}\
\subfloat[GD converges to $\bx^*_1$]{\label{fig:h4}\includegraphics[width=0.42\textwidth]{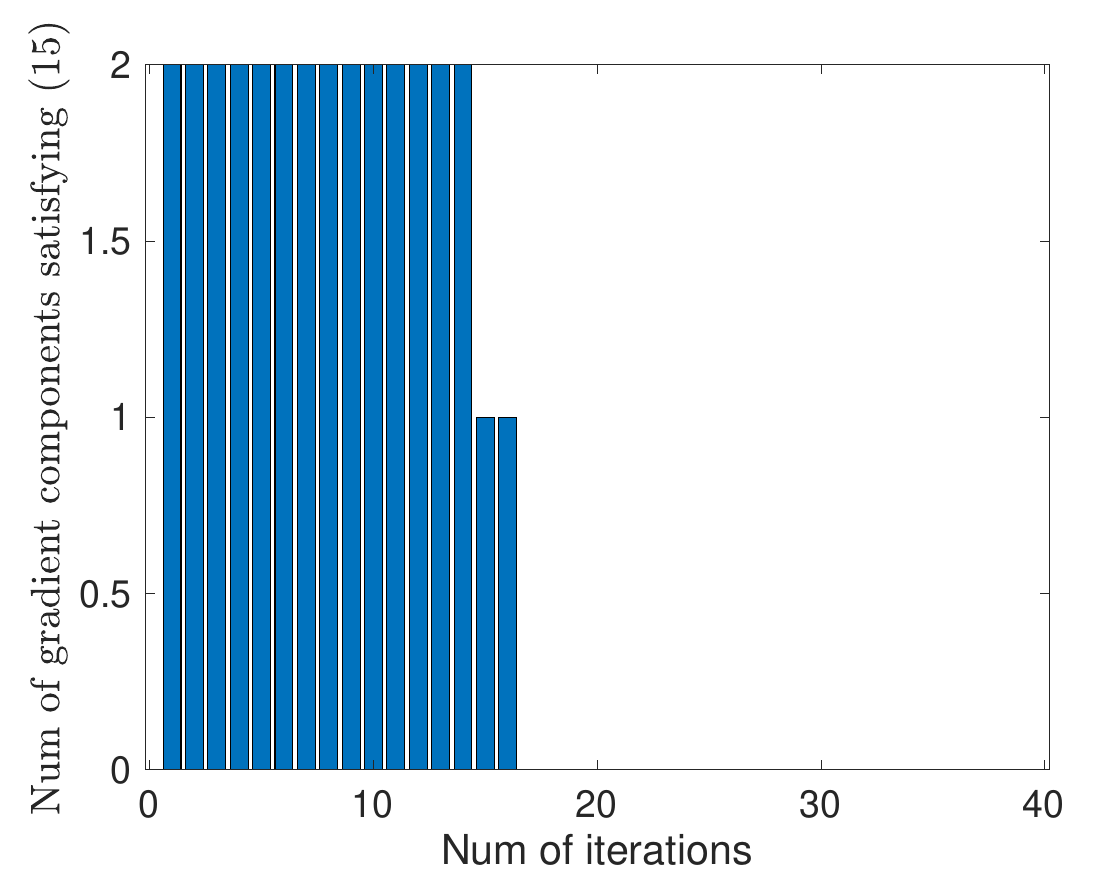}}
\caption{Himmelblau's Function implemented using with stepsize $t=0.012$ and fixed-point numbers Q$8.8$: contour plot with the 4 global minimizers $\bx^*_i$, with $i=1,\dots,4$ (a), comparison of the objective values when GD converges to $\bx^*_4$ (b) and comparison of the objective values when GD converges to $\bx^*_1$ (c); Number of gradient components satisfying \cref{eq:condu_nostagnation} when GD converges to $\bx^*_1$ using $\csr$.}\label{fig:himmelblau_func}
\end{figure}

\subsection{Binary logistic regression}
Let us study the impact of rounding errors on solving logistic regression problems. Logistic regression is commonly used to solve binary classification problems and is proven to satisfy the PL condition by \cite{karimi2016linear} over any compact set. We use BLR to classify the handwritten digits 3 and 8 in the MNIST database. As in \citep{gupta2015deep}, the pixel values are normalized to $[0,1]$. The default decision threshold is set at $0.5$ for interpreting probabilities to class labels since the sample class sizes are almost equal \citep{chen2006decision}. Specifically, class 1 is defined for those predicted scores larger than or equal to $0.5$. 

\begin{figure}[t!]
\centering
\subfloat[Q15.8]{\label{fig:regf2}\includegraphics[width=0.35\textwidth]{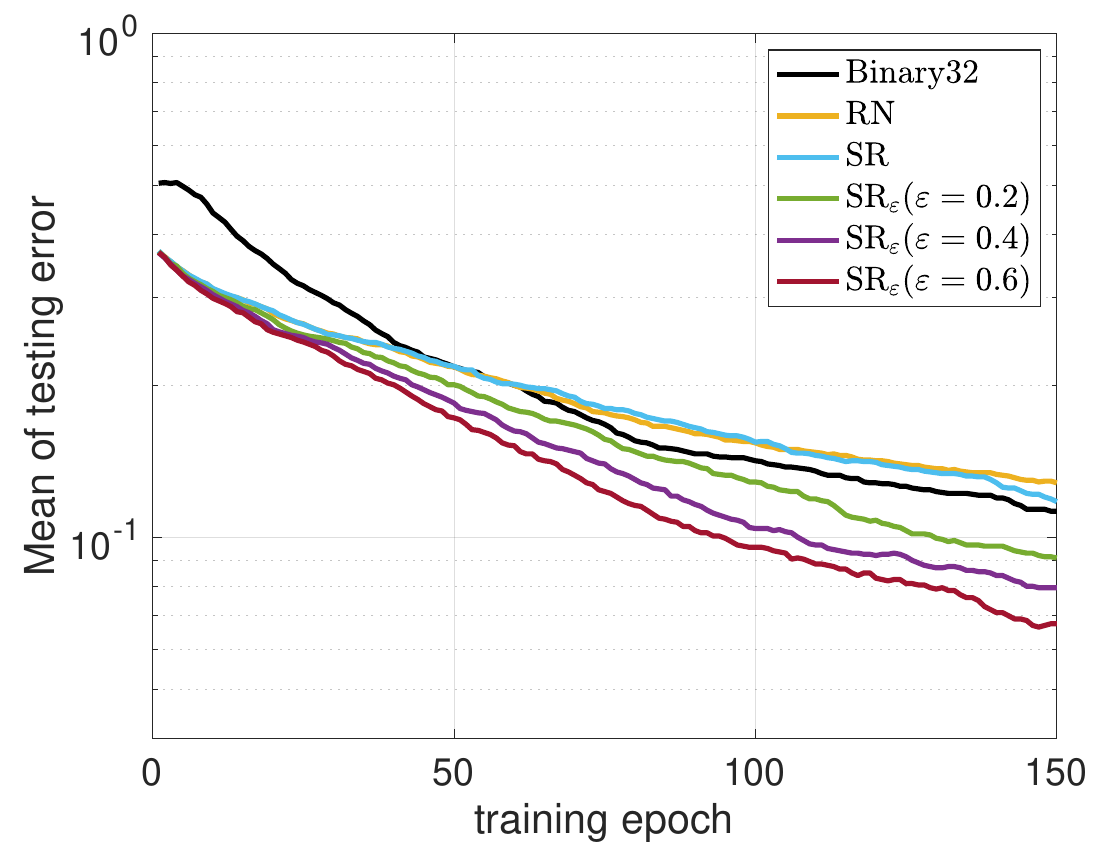}}\quad
\subfloat[Q15.6]{\label{fig:regf1}\includegraphics[width=0.35\textwidth]{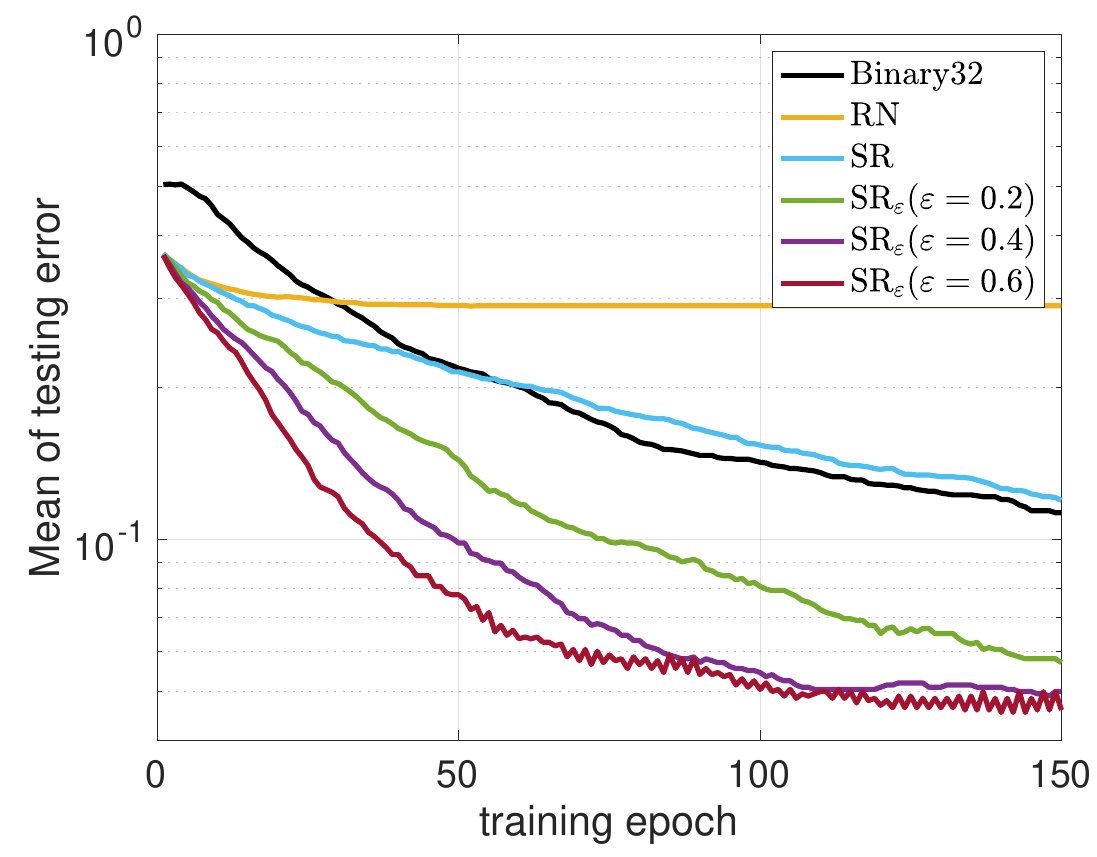}}
\caption{Mean of objective values over 10 simulations of BLR with stepsize $t=0.1$ ($0.1015625$ in Q$15.8$) between $\rn$, $\csr$ and $\srh$ with $\eps=0.2$, $\eps=0.4$ and $\eps=0.6$ with working precision Q$15.8$ and the precision for evaluating $\bsigma_2$, i.e., Q$15.8$ (a) and Q$15.6$ (b).}\label{fig:logisre}
\end{figure}

We study the convergence of GD with different rounding precisions. For each simulation study, we demonstrate two rounding precisions, i.e., the working precision and the precision for evaluating $\bsigma_2$. \cref{fig:logisre} shows the comparison of the BLR's testing errors, obtained by the various rounding methods. In \cref{fig:regf2}, we employ enough digits (Q$15.8$) so that GD has no stagnation utilizing $\rn$. It can be observed that $\csr$ and $\rn$ yield very similar results, while $\srh$ leads to faster convergence of GD compared to $\rn$ and $\csr$ with the same rounding precision. Additionally, increasing the value of $\eps$ results in a faster convergence of GD. Keeping the same working precision and lowering the precision in evaluating $\bsigma_2$, we observe the stagnation of GD with $\rn$ and that $\csr$ provides a very similar result to the one shown in \cref{fig:regf2}; for more details see \cref{fig:regf1}. In the same figure, $\srh$ utilizes the large rounding errors and leads to a significantly faster convergence rate of GD than the one shown in \cref{fig:regf2}. Again, a faster convergence of GD can be obtained by increasing the value of $\eps$. However, comparing \cref{fig:regf2,fig:regf1}, it can be seen that when the rounding precision and $\eps$ are both large, oscillations may occur as GD approaches the global optimum. For an example, see the result of $\srh$ with $\eps=0.6$ in \cref{fig:regf1}.

\begin{figure}[htp!]
\centering
\subfloat[$\rn$]{\label{fig:t1}\includegraphics[width=0.33\textwidth]{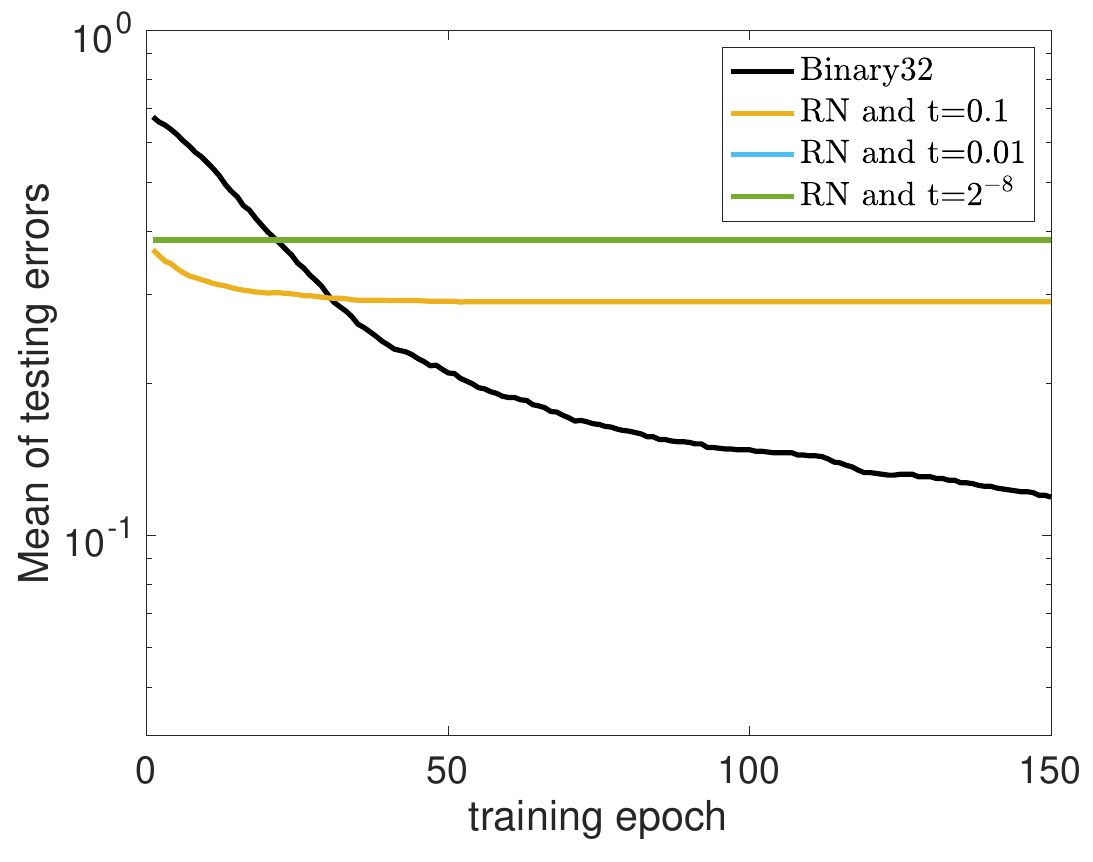}}
\subfloat[$\csr$]{\label{fig:t2}\includegraphics[width=0.33\textwidth]{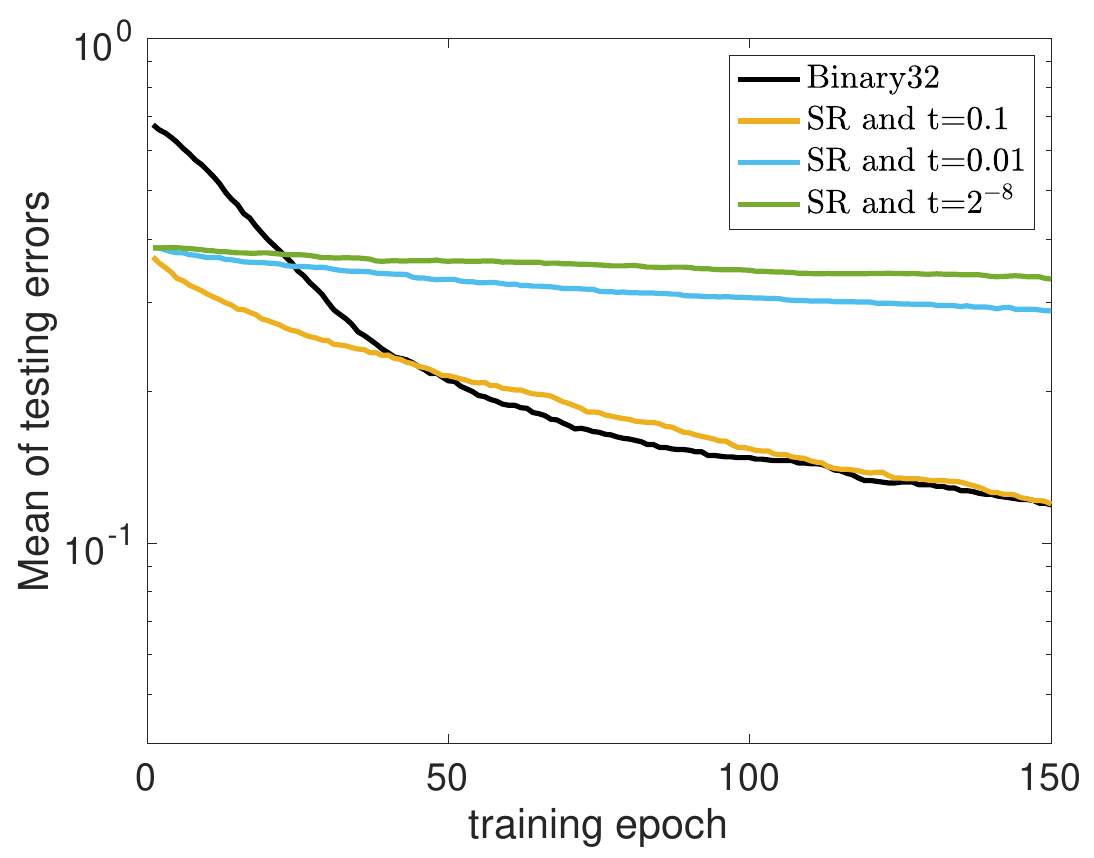}}
\subfloat[$\srh$]{\label{fig:t3}\includegraphics[width=0.33\textwidth]{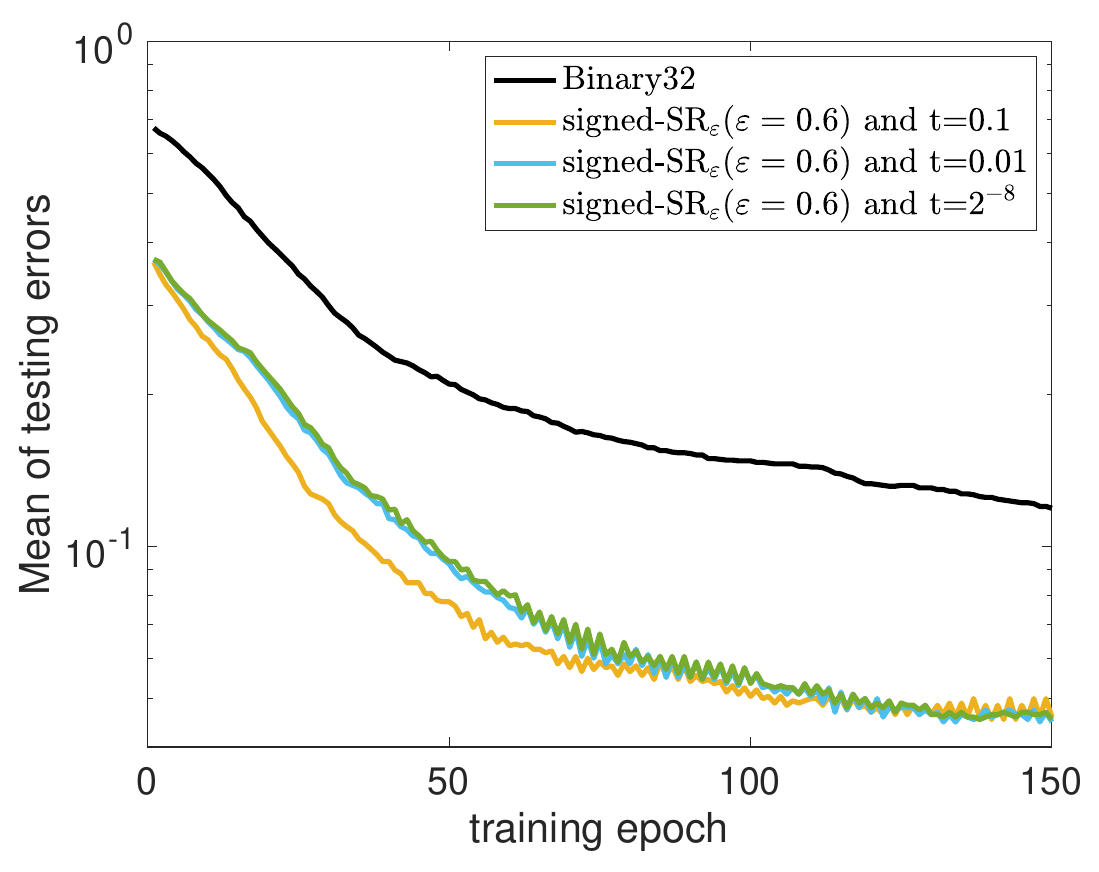}}
\caption{Mean of objective values over 10 simulations with working precision Q$15.8$, precision for evaluating $\bsigma_2$ Q$15.6$, and different stepsize $0.1$ ($0.1015625$ in Q$15.8$), $0.01$ ($0.01171875$ in Q$15.8$) and the smallest number that can be represented in Q$15.8$, i.e., $2^{-8}\approx0.004$, for $\rn$ (a), $\csr$ (b), and $\srh$ (c).}
\label{fig:sensitivity}
\end{figure}

Next, we study the influence of the stepsize $t$ in \cref{eq:gd_fp} on the convergence of GD with low-precision computation for different rounding methods. We employ the same number format as in \cref{fig:regf1} and we vary the stepsize of GD. In particular, we consider $t=0.1$, $0.01$ and the smallest number that can be represented in Q$15.8$, i.e., $2^{-8}\approx0.004$ for each rounding method such as $\rn$, $\csr$, and $\srh$ with $\eps=0.6$. \cref{fig:sensitivity} shows the results obtained using the different values of $t$ and rounding schemes. From \cref{fig:t1}, it can be seen that decreasing the value of $t$ makes GD stagnate earlier with $\rn$. When $\csr$ is employed (cf.~\cref{fig:t2}), the convergence rate is significantly more sensitive to the value of $t$. When employing $\csr$, a small value of $t$ leads to a slower convergence rate. When $t$ is relatively small, e.g., $t=2^{-8}$, GD almost stagnates with $\csr$; see \cref{fig:t2}. However, it can be observed that the convergence rate obtained using $\srh$ is hardly affected by the stepsize when $t\le 0.1$. In general, both $\rn$ and $\csr$ are sensitive to $t$, while $\srh$ is less sensitive to $t$ due to the parameter $\eps$ in \cref{eq:srh_p}. When $t$ is small, the probability of rounding numbers towards 0 in $\csr$ (cf.~\cref{eq:csr}) is close to 1, while it is dominated by $1-\eps$ in $\srh$ (cf.~\cref{eq:srh_p}). Therefore, an almost constant rounding probability is achieved in $\srh$, which makes $\srh$ less sensitive to the value of $t$ when implementing GD in a low-precision number format.

\subsection{Four-layer fully connected NN}
Now we show that our theoretical conclusion for problems satisfying the PL condition is also applicable for the training of a four-layer fully connected NN. In particular, the convergence rate of GD with $\srh$ is faster than that with $\csr$, although the PL condition does not hold here. The latter property is showcased for both fixed-point and floating-point representation systems. The NN is trained to classify the 10 handwritten digits (from 0 to 9) in the MNIST database. Again, the pixel values are normalized to $[0,1]$. The NN is built with the ReLu activation function in the hidden layer and the softmax activation function in the output layer. The hidden layers contain 512, 256, and 128 units. In the backward propagation, the cross-entropy loss function is optimized using GD. The weights matrix is initialized using Xavier initialization \cite[]{glorot2010understanding} and the bias is initialized as a zero vector. 

\begin{figure}[th!]
\centering
\subfloat[Fixed-point arithmetic]{\label{fig:nna}\includegraphics[width=0.35\textwidth]{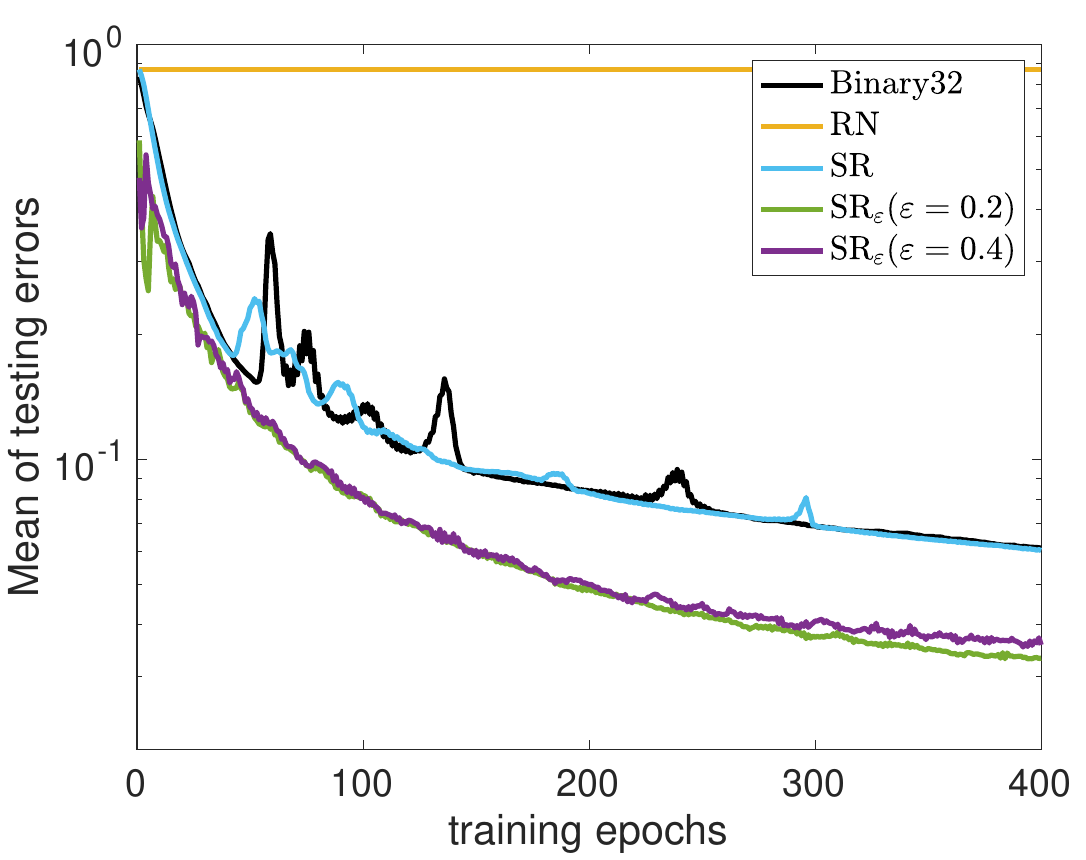}}\quad
\subfloat[Floating-point arithmetic]{\label{fig:nnb}\includegraphics[width=0.35\textwidth]{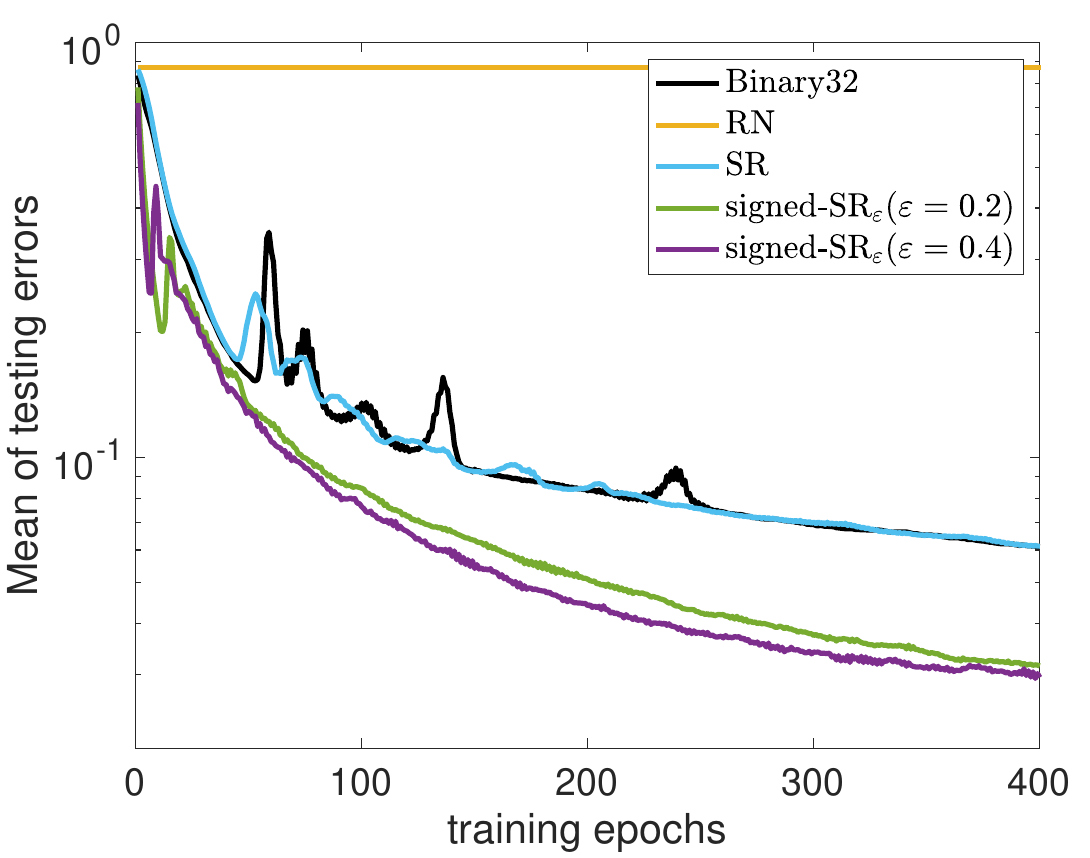}}
\caption{Mean of testing errors of the NN trained over 10 simulations using fixed-point number format Q$8.8$ (a) and using \textsf{Binary16} with 5 significant digits (b).}
\label{fig:nns}
\end{figure}
\cref{fig:nna} shows the means of testing errors of the NNs trained using \textsf{Binary32} for working precision and Q$8.8$ for evaluating $\bsigma_2$ with different rounding methods, such as $\rn$, $\csr$, and $\srh$. Again, $\rn$ causes the stagnation of GD due to the problem of vanishing gradient.
$\csr$ results in a very similar testing error to single-precision computation. The testing error is around $0.06$ with $400$ iterations of GD using $\csr$ and single-precision computation, while a similar accuracy can be achieved by $\srh$ with around 150 iterations. Comparing the results of $\srh$ with $\eps=0.2$ and $\eps=0.4$, slightly larger oscillations are gained with a larger value of $\eps$.
Overall, $\srh$ with $\eps=0.2$ leads to a slightly lower testing error than $\srh$ with $\eps=0.4$. Next, we repeat the simulations with a 16-bit floating-point number format with 5 significant digits using different rounding methods. From \cref{fig:nnb}, it can be seen that the results are very similar to that of fixed-point number formats.
Comparing \cref{fig:nna} and \cref{fig:nnb}, the outcomes obtained using $\srh$ with a floating-point number format demonstrate a slight improvement over those achieved with a fixed-point number format, although the significance of this improvement is negligible. Additionally, for a large number of layers, it may be beneficial to use $\csr$ for evaluating computations such as inner products as it typically reduces the magnitude of rounding errors, and it is unbiased.  The use of biased stochastic rounding such as $\srh$ is recommended  for evaluating the multiplication of the learning rate with the rounded gradient as it prevents the vanishing gradient problem. 
In the situation where the rounding errors are significant, for instance because of a large dimension $n$, one may need to tune $u$ or $t$ to restore the monotonicity of GD. 
In such a case, if the sign of the rounded gradient corresponds to that of the true gradient, then a smaller stepsize $t$ should be sufficient.
Otherwise, we recommend to use higher precision (i.e., smaller $u$). 

In general, when the convergence of GD is guaranteed, low-precision computation can result in a faster convergence rate compared to higher precision when using $\srh$; see \cref{fig:rosenbrockfix,fig:logisre,fig:nns}. With low-precision computation, the utilization of $\csr$ yields results of GD that are similar to those obtained by the exact computation, and the use of $\rn$ may result in stagnation of GD. Additionally, the performance of $\csr$ is sensitive to the stepsize of GD. When the gradients and the stepsize are too small or when the rounding precision is too large, this sensitivity may cause issues such as loss of gradients similar to $\rn$ or very slow convergence; see, e.g., \cref{fig:t2}. On the other hand, $\srh$ is less sensitive to the stepsize of GD such that changing $t$ can hardly affect the convergence rate of GD; see \cref{fig:t3}. However, when computing GD using $\srh$ with a large value of $\eps$, oscillation may happen. When comparing the effects of rounding bias in fixed-point and floating-point arithmetic, the implementation of GD using $\srh$ with low-precision floating-point computation behaves similar to a gradient descent method with adaptive stepsizes in each coordinate of the current iterate, whereas the implementation using $\srh$ with low-precision fixed-point computation performs similarly to a strategy that combines vanilla gradient descent and stochastic sign gradient descent methods. For most of the numerical studies, $\srh$ results in faster convergence of GD in floating-point arithmetic than in fixed-point arithmetic;  nevertheless, when training a fully connected NN the convergence of GD with $\srh$ in fixed-point arithmetic is fairly similar to that in floating-point arithmetic. 

\section{Conclusion}\label{sec:conclu}
We have studied the convergence of the gradient descent method (GD) in limited-precision fixed-point arithmetic using unbiased stochastic rounding ($\csr$) and $\eps$-biased stochastic rounding ($\srh$), for problems satisfying the Polyak--{\L}ojasiewicz condition. In the analysis, we have proven that a linear convergence rate of GD can be obtained when utilizing either $\csr$ or $\srh$, and the convergence bound for $\srh$ is shown to be stricter than the one corresponding to $\csr$. By delving into the factors that impact the convergence of GD in low-precision computation, we have provided valuable insights and knowledge to guide informed decision-making when selecting an appropriate rounding method for training neural networks (NNs).

In our numerical studies, we have demonstrated that $\srh$ provides faster convergence than $\csr$ and $\rn$, on average, when the same number format is used. In particular, $\srh$ may considerably accelerate the convergence of GD with low-precision computation for both the training of logistic regression models and NNs, and this is potentially valuable for machine learning. 
The comparison between floating-point and fixed-point computations illustrates that the implementation of GD using $\srh$ with floating-point arithmetic performs similarly to a gradient descent method with adaptive stepsizes in each coordinate of the current iterate, while the implementation using $\srh$ with fixed-point arithmetic behaves like a combination of vanilla gradient descent and stochastic sign gradient descent methods.
For the training of a fully connected NN, the convergence of GD with $\srh$ in fixed-point arithmetic is very similar to that in floating-point arithmetic, and for both number formats, the performance of GD using $\srh$ is superior to that of $\csr$. 

\section*{Acknowledgement}
We thank the reviewers for their constructive comments and the editors for the handling of this paper. We thank Mark Peletier for discussions on the upper bound on the PL constant.

\section*{Declarations}
\textbf{Funding} This research was funded by the EU ECSEL Joint Undertaking under grant agreement no.~826452 (project Arrowhead Tools).

\vspace{2mm}
\noindent\textbf{Data availability} The MNIST dataset used in this study is available in the open repository \url{https://git-disl.github.io/GTDLBench/datasets/mnist_datasets/}. 

\vspace{2mm}
\noindent\textbf{Conflict of interest}
The authors declare that they have no conflict of interest.

\noindent

\begin{appendices}

\section{Proofs of lemmas}\label{ap:appendixA}
In this appendix, we prove some helpful lemmas used in our convergence analysis.

\noindent
\begin{proof}[\bf Proof of \cref{lem:bound_h}] 
Under the condition of Case I \cref{eq:condu_nostagnation}, we consider the upper bound and lower bound of $r_i^{(k)}$ in two scenarios, i.e., $t\, |\nabla f(\wt\bx^{(k)})_i+\sigma_{1,i}^{(k)}|= u$ (S1) and $t\, |\nabla f(\wt\bx^{(k)})_i+\sigma_{1,i}^{(k)}|> u$ (S2).

S1: The condition $t\, |\nabla f(\wt\bx^{(k)})_i+\sigma_{1,i}^{(k)}|= u$ implies that $\sigma_{2,i}^{(k)}=0$, which results in $r_i^{(k)}=\frac{\sigma^{(k)}_{1,i}}{\nabla f(\wt\bx^{(k)})_i}$. Condition \cref{eq:condu_nostagnation} implies that $\nabla f(\wt\bx^{(k)})_i\neq-\sigma_{1,i}^{(k)}$, i.e., $r_i^{(k)}\neq-1$, and \cref{eq:bound_gradient} gives $-1\le r_i^{(k)}\le1$. Therefore,  we have that $-1< r_i^{(k)}\le1$ under conditions \cref{eq:bound_gradient} and \cref{eq:condu_nostagnation}.

S2: When $t\, |\nabla f(\wt\bx^{(k)})_i+\sigma_{1,i}^{(k)}|>u$, observing that $\sigma_{2,i}^{(k)}< u$, we have $r_i^{(k)}=\frac{t\,\sigma^{(k)}_{1,i}+\sigma^{(k)}_{2,i}}{t\ \nabla f(\wt\bx^{(k)})_i}<\frac{t\,\sigma^{(k)}_{1,i}+t\, |\,\nabla f(\wt\bx^{(k)})_i+\,\sigma^{(k)}_{1,i}|}{t\ \nabla f(\wt\bx^{(k)})_i}\le3$. Concerning the lower bound we consider separately the cases where $\sigma_{1,i}^{(k)}$ and $ \,\nabla f(\wt\bx^{(k)})_i$ have the same or opposite signs. When $\sign(\sigma_{1, i})=\sign(\nabla f(\wt\bx^{(k)})_i)$ we have that 
\[
 r_i^{(k)}\ge \frac{| t\,\sigma_{1, i}|-|\sigma_{2,i}|}{| t\ \nabla f(\wt\bx^{(k)})_i|}>\frac{| t\,\sigma_{1, i}|-| t\ \nabla f(\wt\bx^{(k)})_i+t\,\sigma_{1, i}|}{| t\ \nabla f(\wt {\bx}^{(k)})_i|}= -1.
\]
When $\sign(\sigma_{1, i})=-\sign(\nabla f(\wt\bx^{(k)})_i)$, \cref{eq:bound_gradient} and the property $t\, |\nabla f(\wt\bx^{(k)})_i+\sigma_{1,i}^{(k)}|>u$ yield $|t\ \nabla f(\wt\bx^{(k)})_i|-| t\,\sigma^{(k)}_{1,i} |> u$, which implies $\frac{-| t\,\sigma^{(k)}_{1,i}|}{| t\ \nabla f(\wt\bx^{(k)})_i|}>-1+\frac{u}{| t\ \nabla f(\wt\bx^{(k)})_i|}$. Therefore, we have \[
r_i^{(k)}\ge\frac{-| t\,\sigma^{(k)}_{1,i}|-| \sigma^{(k)}_{2,i}|}{| t\ \nabla f(\wt\bx^{(k)})_i|}>-1+\frac{u-|\sigma^{(k)}_{2,i} |}{| t\ \nabla f(\wt\bx^{(k)})_i|}\ge-1.
\]
Combining these two scenarios, we have $-1<r_i^{(k)}<3$ concluding the proof.
\end{proof}

\begin{proof}[\bf Proof of \cref{lem:csr_sigma1}]
 Clearly, we have 
 \begin{align*}
 \expt\,[\,\nabla f(\wt\bx^{(k)})^T( \nabla f(\wt\bx^{(k)})+\bsigma_{1}^{(k)})\,]
 &=\expt\,[\, \|\nabla f(\wt\bx^{(k)})\|^2\,]+\sum_{i=1}^{n}\expt\,[\,\sigma_{1,i}^{(k)}\,\nabla f(\wt\bx^{(k)})_i \,].
 \end{align*}
 The second condition in \cref{assum:universal assumption} implies  $\expt\,[\,\sigma_{1,i}^{(k)}\,\big| \ \nabla f(\wt\bx^{(k)})_i]\lesssim c\,u^2$ and based on the law of total expectation, we obtain
 \begin{align}
 \big| \expt\,[\,\sigma_{1,i}^{(k)}\,\nabla f(\wt\bx^{(k)})_i \,]\big|&=\Bigg|\!\!\sum_{\nabla f(\wt\bx^{(k)})_i=q} \!\!\!\!\expt\,[\,\sigma_{1,i}^{(k)}\,\nabla f(\wt\bx^{(k)})_i \ \big| \ \nabla f(\wt\bx^{(k)})_i=q \,]\,P(\nabla f(\wt\bx^{(k)})_i=q)\Bigg|\nonumber\\
  &\lesssim \sum_{\nabla f(\wt\bx^{(k)})_i=q} c\,u^2\, | q| \,P(\nabla f(\wt\bx^{(k)})_i=q)\nonumber\\ &=c\,u^2\, \expt\,[\, |\nabla f(\wt\bx^{(k)})_i |\,],\nonumber
 \end{align}
 which in turn yields that $\sum_{i=1}^{n}\expt\,[\,\sigma_{1,i}^{(k)}\,\nabla f(\wt\bx^{(k)})_i \,] $ is bounded from above by 
\[c\,u^2\, \expt\,[\, \|\nabla f(\wt\bx^{(k)}) \|_1\,]\le \sqrt{n}\,c\,u^2\, \expt\,[\, \|\nabla f(\wt\bx^{(k)}) \|\,] \le L\chi \sqrt{n} \,c\,u^2.\]
\end{proof}

\begin{proof}[\bf Proof of \cref{lem:srh_dk}]
When $\srh$ is applied, on the basis of \cref{eq:strounding} and \cref{eq:srh}, we have 
\begin{align}\label{eq:inner_sigma2grad}
 \expt\,[\,\sigma_{2,i}^{(k)}\ \big|& \ t\,(\nabla f(\wt\bx^{(k)})_i+\sigma_{1,i}^{(k)}) \,] \nonumber\\&=\,\begin{cases}\eps\,u\,\sign(\nabla f(\wt\bx^{(k)})_i+\sigma_{1,i}^{(k)}),\quad &0<p_{\eps}<1,\\[1mm]
 \lfloor t\,(\nabla f(\wt\bx^{(k)})_i+\sigma_{1,i}^{(k)})\rfloor-t\,(\nabla f(\wt\bx^{(k)})_i+\sigma_{1,i}^{(k)})+u,& p_{\eps}=0,\\[1mm]
 \lfloor t\,(\nabla f(\wt\bx^{(k)})_i+\sigma_{1,i}^{(k)})\rfloor-t\,(\nabla f(\wt\bx^{(k)})_i+\sigma_{1,i}^{(k)}),& p_{\eps}=1.  
 \end{cases}
\end{align}
Note that we omit the dependency on $t\,(\nabla f(\wt\bx^{(k)})_i+\sigma_{1,i}^{(k)})$ in $p_\eps$ for brevity.
Furthermore, $p_\eps(t\,(\nabla f(\wt\bx^{(k)})_i+\sigma_{1,i}^{(k)}))=0$ implies $\nabla f(\wt\bx^{(k)})_i+\sigma_{1,i}^{(k)}>0$ and $p_\eps(t\,(\nabla f(\wt\bx^{(k)})_i+\sigma_{1,i}^{(k)}))=1$ implies $\nabla f(\wt\bx^{(k)})_i+\sigma_{1,i}^{(k)}<0$; on the basis of \cref{eq:bound_gradient}, we have $\sign(\nabla f(\wt\bx^{(k)})_i+\sigma_{1,i}^{(k)})=\sign(\nabla f(\wt\bx^{(k)})_i)$. Therefore, we have
\begin{align}\label{eq:sigma2nablaf_i}
  \expt\,[&\,\sigma_{2,i}^{(k)}\nabla f(\wt\bx^{(k)})_i\ \big|\ t\,(\nabla f(\wt\bx^{(k)})_i+\sigma_{1,i}^{(k)}) \,] \\&=\,\begin{cases}\eps\,u\, |\nabla f(\wt\bx^{(k)})_i|,\quad &0<p_{\eps}<1,\\[1mm]
 |\lfloor t\,(\nabla f(\wt\bx^{(k)})_i+\sigma_{1,i}^{(k)})\rfloor-t\,(\nabla f(\wt\bx^{(k)})_i+\sigma_{1,i}^{(k)})+u|\,\, |\nabla f(\wt\bx^{(k)})_i|,& p_{\eps}=0,\\[1mm]
 |\lfloor t\,(\nabla f(\wt\bx^{(k)})_i+\sigma_{1,i}^{(k)})\rfloor-t\,(\nabla f(\wt\bx^{(k)})_i+\sigma_{1,i}^{(k)})|\,\, |\nabla f(\wt\bx^{(k)})_i|,& p_{\eps}=1
 ,\nonumber
 \end{cases}
\end{align}
which are all positive random variables for all the three cases, concluding the claim.
\end{proof}

\begin{proof}[\bf Proof of \cref{lem:lambda_bound}.]
We denote by $\cals$ the finite set of values that the $i$th component of $\nabla f(\wt\bx^{(k)})+\bsigma_1^{(k)}$ can assume. The set $\cals_1$ is the subset of $\cals$ such that for all $\nabla f(\wt\bx^{(k)})_i+\sigma_{1,i}^{(k)}\in \cals_1$ satisfying $0<p_\eps(t\,(\nabla f(\wt\bx^{(k)})_i+\sigma_i^{(k)}))<1$. Analogously we define $\cals_2$ and $\cals_3$ associated with the conditions $p_\eps=0$ and $p_\eps=1$, respectively. According to the law of total expectation, we have
\begin{small}
\begin{align*}
\expt\,[\,&\sigma_{2,i}^{(k)}\,\nabla f(\wt\bx^{(k)})_i\,]=\!\!\sum_{j=1,2,3}\!\!\expt\,[\,\sigma_{2,i}^{(k)}\,\nabla f(\wt\bx^{(k)})_i\ \big|\ f(\wt\bx^{(k)})_i+\sigma_i^{(k)}\!\in \mathcal{S}_j\,]\,P(f(\wt\bx^{(k)})_i+\sigma_i^{(k)} \in \mathcal{S}_j).
\end{align*}
\end{small}\noindent
Note that in view of \cref{eq:epsilon}, we have
\[
p_\varepsilon =1\quad \Rightarrow\quad|\lfloor t\,(\nabla f(\wt\bx^{(k)})_i+\sigma_{1,i}^{(k)})\rfloor-t\,(\nabla f(\wt\bx^{(k)})_i+\sigma_{1,i}^{(k)})|<\eps\,u
\]
and
\[
p_\varepsilon=0\quad \Rightarrow\quad
|\lfloor t\,(\nabla f(\wt\bx^{(k)})_i+\sigma_{1,i}^{(k)})\rfloor-t\,(\nabla f(\wt\bx^{(k)})_i+\sigma_{1,i}^{(k)})+u|<\eps\,u.
\] 
Together with \cref{eq:inner_sigma2grad}, we obtain $ \min_{i=1,\dots,n}\expt\,[\,\sigma_{2,i}^{(k)} \nabla f(\wt\bx^{(k)})_i\,]\!\le \eps\,u\min_{i=1,\dots,n}\expt\,[\, |\,\nabla f(\wt\bx^{(k)})_i\, |\,].$ Jensen's inequality gives that $\sqrt{n}\,\min_{i=1,\dots,n}\expt\,[\, |\,\nabla f(\wt\bx^{(k)})_i\, |\,]\!\le \expt\,[\, \|\nabla f(\wt\bx^{(k)})\|\,]$; together with condition~\cref{eq:uc1}, we have 
\begin{align*}
 \rho_k=\,&\min_{i=1,\dots,n}\frac{n\,\expt\,[\,\sigma_{2,i}^{(k)} \nabla f(\wt\bx^{(k)})_i\,]}{\expt\,[\, \|\nabla f(\wt\bx^{(k)})\|^2 \,]}\le \frac{n\,\eps\,u\,\min_{i=1,\dots,n}\expt\,[\, |\,\nabla f(\wt\bx^{(k)})_i\, |\,]}{\expt\,[\, \|\nabla f(\wt\bx^{(k)})\|^2 \,]}\nonumber\\
 \le \,&\frac{\sqrt{n}\,u\,\eps\, \expt\,[\, \|\nabla f(\wt\bx^{(k)})\|\,]}{\expt\,[\, \|\nabla f(\wt\bx^{(k)})\|^2 \,]}\underset{\cref{eq:uc1}}{\le}\frac{2\,t\,\eps\,\expt\,[\, \|\nabla f(\wt\bx^{(k)})\|\,]^2}{\expt\,[\, \|\nabla f(\wt\bx^{(k)})\|^2 \,]},
\end{align*} 
which leads to $\rho_k\le 2\,t\,\eps$, according to Jensen's inequality. Conditions \cref{eq:bound_gradient} and \cref{eq:condu_nostagnation} imply that $|\nabla f(\wt\bx^{(k)})_i|>0$. Together with \cref{eq:sigma2nablaf_i}, we have $0<\rho_k\le 2\,t\,\eps$.
\end{proof}

\section{Proofs of propositions}\label{appendixB}

\begin{proof}[\bf Proof of \cref{prop:convergencerate_c3_srh}]
When $\sigma^{(k)}_{1,i}$ and $\sigma^{(k)}_{2,i}$ are obtained by $\csr$ and $\srh$, respectively, substituting \cref{eq:expectedfxk1} and \cref{eq:c2srh_ith} into \cref{eq:f_c3} yields 
\begin{align}
 \expt\,[\,f(\wt\bx^{(k+1)})\,]&\lesssim\expt\,[\,f(\wt\bx^{(k)})\,]-\tfrac12 \sum_{i\in\calc_1}(\,t\,\expt\,[\,\nabla f(\wt\bx^{(k)})_i^2 \,] +\expt\,[\,\sigma_{2,i}^{k}\nabla f(\wt\bx^{(k)})_i\,]\,)\nonumber\\&\qquad-\tfrac12 \,\theta_k\sum_{i\in\calc_2}(\,t\,\expt\,[\,\nabla f(\wt\bx^{(k)})_i^2 \,]+\beta_k\,u\ \expt\,[\, |\nabla f(\wt\bx^{(k)})_i|\,]\,)\nonumber\\
 &=\expt\,[\,f(\wt\bx^{(k)})\,]-\tfrac12 \,t\,\expt\,[\, \|\,\nabla f(\wt\bx^{(k)})\, \|^2 \,] \nonumber\\
 &\qquad -\sum_{i\in\calc_1}(\expt\,[\,\sigma_{2,i}^{k}\nabla f(\wt\bx^{(k)})_i\,]-\tfrac12 \,\theta_k\,\beta_k\,u\ \expt\,[\, |\nabla f(\wt\bx^{(k)})_i|\,])\nonumber\\
 &\qquad-\sum_{i\in\calc_2}\,\tfrac12 \,(\theta_k-1)\,t\,\expt\,[\,\nabla f(\wt\bx^{(k)})_i^2 \,]-\tfrac12 \,\theta_k\,\beta_k\,u\ \expt\,[\, \|\nabla f(\wt\bx^{(k)})\|\,]\,.\nonumber
\end{align}
Furthermore, the properties $\beta_k\,u\le | \expt\,[\,\sigma_{2,i}^{(k)}\,]|$ and $\theta_k\le 2$ imply that $\tfrac12 \,\beta_k\,u\ \expt\,[\, |\nabla f(\wt\bx^{(k)})_i|\,]\le \expt\,[\,\sigma_{2,i}^{k}\nabla f(\wt\bx^{(k)})_i\,]$. On the basis of \cref{eq:tau_2} and \cref{eq:alpha_k}, we have
\begin{align}
 \expt\,[\,f(\wt\bx^{(k+1)})\,]\lesssim\,
 \,&\expt\,[\,f(\wt\bx^{(k)})\,]-\tfrac12 \,(t+\theta_k\,\tau_2+\alpha_k)\,\expt\,[\, \|\,\nabla f(\wt\bx^{(k)})\, \|^2 \,].\nonumber
\end{align}
Expanding the recursion $k$ times, we obtain \cref{eq:conv-expsrhplstageIII}, with $\tau_2< 2\,t\,\eps$ and $|\alpha_j|<t\, |\theta_j-1|$ for all $j$. Since $\mu\,(t+\alpha_j+\theta_j\,\tau_2)\le \mu\, t\,(1+| \theta_j-1|+2\,\theta_j\,\varepsilon)\le \mu\, t\,(1+1+4\varepsilon)\le 2\,\mu \,t\,(1+2\varepsilon)\le L\,t\,(1+2\varepsilon)\le \frac{1+2\varepsilon}{4}<1.$ Further the property $1-2 \,\mu\,t>0$ and $-t<\alpha_j<t$ (cf.~\cref{eq:alpha_k}) indicate that $1-\mu\,(t+\alpha_j)>1-2 \,\mu\,t>0$, combining with the property $\theta_j\,\tau_2>0$ we have $0\le \mu\,(t+\alpha_j+\theta_j\,\tau_2)<1$.
\end{proof}

\begin{proof}[\bf Proof of \cref{prop:dfixed}]
The unbiased property of $\csr$ yields \cref{eq:dficsr}. According to \cref{eq:srh_p} and \cref{eq:dficsr}, for $\nabla f(\bx^{(k)})_i<0$, we have $-u\,p_0(t\ \nabla f(\bx^{(k)})_i)=t\ \nabla f(\bx^{(k)})_i$, which indicates 
\begin{align}
 -u\,p_{\eps}(t\ \nabla f(\bx^{(k)})_i)&=-u\,(\,p_0(t\ \nabla f(\bx^{(k)})_i)+\sign(\nabla f(\bx^{(k)})_i)\,\eps\,)\nonumber\\
 &=t\ \nabla f(\bx^{(k)})_i+\eps\,u\,\sign(\nabla f(\bx^{(k)})_i).\nonumber
\end{align}
The same result can be obtained for $\nabla f(\bx^{(k)})_i>0$ by using the property $t\ \nabla f(\bx^{(k)})_i= u\,(1-p_0(t\ \nabla f(\wt\bx^{(k)})_i)\,)$. 
\end{proof}

\begin{proof}[\bf Proof of \cref{prop:dfloat}]
Derived from the model of floating-point number representation \cite[Sec.~2.3]{xia2022float}, we have \begin{align}
 \mathrm{fl}(x_i^{(k)}-t\ \nabla f(\bx^{(k)})_i)=\begin{cases}
 (x_i^{(k)}-t\ \nabla f(\bx^{(k)})_i)\,(1+\delta_i^{(k)}),&\text{if $\sign(x_i^{(k)}\,\nabla f(\bx^{(k)})_i)=-1,$}\\
 (x_i^{(k)}-t\ \nabla f(\bx^{(k)})_i)\,(1-\delta_i^{(k)}),&\text{if $\sign(x_i^{(k)}\,\nabla f(\bx^{(k)})_i)=\ph{-}1,$}\nonumber
 \end{cases}
\end{align}
where $0<\delta_i^{(k)}<2 \,u$. In light of \cref{eq:dk}, we obtain
\begin{align}
\wt{d}_i^{(k)}=\,\,&x_i^{(k)}-\mathrm{fl}(x_i^{(k)}-t\ \nabla f(\bx^{(k)})_i)\nonumber\\=\,&
 \begin{cases}
 -\delta_i^{(k)}\,x_i^{(k)}+t\ \nabla f(\bx^{(k)})_i\,(1+\delta_i^{(k)}),&\text{if $\sign(x_i^{(k)}\,\nabla f(\bx^{(k)})_i)=-1,$}\\
 \delta_i^{(k)}\,x_i^{(k)}+t\ \nabla f(\bx^{(k)})_i\,(1-\delta_i^{(k)}),&\text{if $\sign(x_i^{(k)}\,\nabla f(\bx^{(k)})_i)=\ph{-}1.$}\nonumber
 \end{cases}
\end{align}
On the basis of the unbiased property of $\csr$ and taking the expectation of $\wt{d}_i^{(k)}$, we obtain $\expt\,[\,\wt{d}_i^{(k)}\ \big |\ x_i^{(k)}-t\ \nabla f(\bx^{(k)})_i\,]=\,t\ \nabla f(\bx^{(k)})_i$,
which gives \cref{eq:dflcsr}. On the basis of \cref{eq:dflcsr} and \cref{eq:ssrh}, by making $\sign(v)=\sign(\nabla f(\bx^{(k)})_i)$ in signed-$\srh$ (cf.~\cref{eq:ssrh}), when $x_i^{(k)}>0$ and $\nabla f(\wt\bx^{(k)})_i<0$, we have 
\begin{align}
 \expt&\,[\,\wt{d}_i^{(k)}\ \big|\ x_i^{(k)}-t\ \nabla f(\bx^{(k)})_i\,]\nonumber\\
 &=(-\delta_i^{(k)}\,x_i^{(k)}+t\ \nabla f(\bx^{(k)})_i\,(1+\delta_i^{(k)}))\,(1\!-p_0(x_i^{(k)}-t\ \nabla f(\bx^{(k)})_i)\!-\sign(\nabla f(\bx^{(k)})_i)\,\eps)\nonumber\\
 &=t\ \nabla f(\bx^{(k)})_i+\eps\,(-\delta_i^{(k)}\,x_i^{(k)}+t\ \nabla f(\bx^{(k)})_i\,(1+\delta_i^{(k)}))\nonumber\\
 &=(1+\eps+\eps\,\delta_i^{(k)})\,t\ \nabla f(\bx^{(k)})_i+\sign(\nabla f(\bx^{(k)})_i)\, | x_i^{(k)}|\,\eps\,\delta_i^{(k)}.\nonumber
\end{align}
Analogously, when $x_i^{(k)}<0$ and $\nabla f(\bx^{(k)})_i>0$, we have \begin{align}
 \expt&\,[\,\wt{d}_i^{(k)}\ \big|\ x_i^{(k)}-t\ \nabla f(\bx^{(k)})_i\,]\nonumber\\
 &=(-\delta_i^{(k)}\,x_i^{(k)}+t\ \nabla f(\bx^{(k)})_i\,(1+\delta_i^{(k)}))\,(\,p_0(x_i^{(k)}-t\ \nabla f(\bx^{(k)})_i)+\sign(\nabla f(\bx^{(k)})_i)\,\eps\,)\nonumber\\
 &=(1+\eps+\eps\,\delta_i^{(k)})\,t\ \nabla f(\bx^{(k)})_i+\sign(\nabla f(\bx^{(k)})_i)\, | x_i^{(k)}|\,\eps\,\delta_i^{(k)}.\nonumber
\end{align} 
Therefore, \cref{eq:dflsrh} holds for the case when $\sign(x_i^{(k)}\,\nabla f(\bx^{(k)})_i)=-1$. The proof of \cref{eq:dflsrh} for the case when $\sign(x_i^{(k)}\,\nabla f(\bx^{(k)})_i)=1$ can be obtained by an analogous argument as the one for $\sign(x_i^{(k)}\,\nabla f(\bx^{(k)})_i)=-1$. 
\end{proof}

\end{appendices}

\footnotesize\bibliography{references}

\end{document}